\documentclass{article}

\usepackage{arxiv}

\usepackage[utf8]{inputenc} 
\usepackage[T1]{fontenc}    
\usepackage{hyperref}       
\usepackage{url}            
\usepackage{booktabs}       
\usepackage{amsfonts}       
\usepackage{nicefrac}       
\usepackage{microtype}      
\usepackage{lipsum}		
\usepackage{graphicx}
\usepackage{natbib}
\usepackage{doi}

\usepackage{caption}
\usepackage{subcaption}
\usepackage{graphicx}
\usepackage{amssymb}
\usepackage{algorithm}
\usepackage{algpseudocode}

\usepackage{caption}

 \usepackage{amsmath} 
 \usepackage{amsthm}

\usepackage{algpseudocode}
\usepackage{statmath}

\newtheorem{theorem}{Theorem}
\newtheorem{corollary}{Corollary}
\newtheorem{lemma}{Lemma}

\newtheorem{assumption}{Assumption}[section]
\newtheorem{remark}{Remark}[section]
\newtheorem{example}{Example}[section]

\title{Reliable and scalable variable importance estimation via warm-start and early stopping}

\author{ Zexuan Sun \\
	 Department of Statistics\\
  University of Wisconsin, Madison\\
  Madison, WI 53706 \\
  \texttt{zexuan.sun@wisc.edu} \\
	\And
	 Garvesh Raskutti \\
  Department of Statistics\\
  University of Wisconsin, Madison\\
  Madison, WI 53706 \\
  \texttt{raskutti@stat.wisc.edu} \\
}

\renewcommand{\shorttitle}{Early Stopping VI}

\hypersetup{
pdftitle={A template for the arxiv style},
pdfsubject={q-bio.NC, q-bio.QM},
pdfauthor={David S.~Hippocampus, Elias D.~Striatum},
pdfkeywords={First keyword, Second keyword, More},
}

\begin{document}
\maketitle
\begin{abstract}
	As opaque black-box predictive models become more prevalent, the need to develop interpretations for these models is of great interest. The concept of \textit{variable importance} and \emph{Shapley values} are interpretability measures that applies to any predictive model and assesses how much a variable or set of variables improves prediction performance. When the number of variables is large, estimating variable importance presents a significant computational challenge because re-training neural networks or other black-box algorithms requires significant additional computation. In this paper, we address this challenge for algorithms using gradient descent and gradient boosting (e.g.  neural networks, gradient-boosted decision trees). By using the ideas of early stopping of gradient-based methods in combination with warm-start using the \emph{dropout} method, we develop a scalable method to estimate variable importance for any algorithm that can be expressed as an \textit{iterative kernel update equation}. Importantly, we provide theoretical guarantees by using the  theory for early stopping of kernel-based methods for  neural networks with sufficiently large (but not necessarily infinite) width and gradient-boosting decision trees that use symmetric trees as a weaker learner. We also demonstrate the efficacy of our methods through simulations and a real data example which illustrates the computational benefit of early stopping rather than fully re-training the model as well as the increased accuracy of our approach.
\end{abstract}


\section{Introduction}




The ubiquitous application of predictive modeling in critical sectors such as judiciary, healthcare, and education necessitates models that are not only accurate but also interpretable. The emphasis on interpretability arises from the need to ensure transparency and fairness in decisions that significantly impact human lives. This push towards interpretable models is supported by the recent scholarly debate highlighting the inadequacies of black-box methods in sensitive applications
(see e.g.~\citep{Rudin2019Why,bb2, AllenInterpretable, DuInterpretable, MolnarInterpretable, YuInterpretable}). 

Traditional parametric models such as linear or logistic regression often fail to achieve good prediction performance with complex modern data, leading to a shift towards non-parametric methods, particularly neural networks. Interpretability for neural networks and other black-box approaches remains challenging. There are a broad range of approaches for interpretability measures (see e.g.~\citep{varnn2,  varnn3}) some of which are model-specific (see e.g.~\citep{intro3, intro4}) and some of which are model-agnostic. Model-agnostic variable importance (VI) or feature attribution techniques are increasingly important as they evaluate variable impact independent for the prediction algorithm leaving the choice of prediction algorithm up to the user without interpretability considerations. Retraining models without variables of interest can benchmark VI but poses challenges in high-dimensional settings due to computational demands \citep{intro5}. Alternative methods like \emph{dropout}, which plugs
in input data with dropped features to the full model
trained with all features directly, offer computational feasibility but typically underfit the data \citep{intro6}.

In this paper, we address this challenge by developing an algorithm that attempts to provide the computational advantages of the dropout approach whilst offering similar accuracy to re-training approaches. In particular, we use a \emph{warm-start} with the dropout model and apply \emph{early stopping} to improve fit while avoiding the high computational cost of full re-training. Early stopping has a long history and is widely employed in gradient-based algorithms like neural networks \citep{esnn}, non-parametric regression \cite{esnonpara}, and boosting \citep{bl2,esbtyu,esbs}. As an efficient regularization technique, early stopping has lower computational complexity and can address dropout underfitting, making it suitable for estimating VI. However theoretical guarantees for early stopping in the context of estimating VI remains an open challenge.

In this paper, we propose a general scalable strategy for any gradient-based iterative algorithm to estimate variable importance. Drawing ideas from \citep{esnonpara,lazyvi}, we combine warm-start initialization using the dropout method and early-stopping to estimate the VI efficiently. The key idea is to start the training of a new model without variables of interest from the original (full model) training data and then run out approach with the removed variables for a small number of iterations to reduce computational costs.

\subsection{Our Contributions} 
Our main contributions are summarized as follows:

\begin{itemize}
    \item We propose a general scalable framework with supporting theoretical guarantees to estimate VI efficiently for any  iterative algorithm
    that can be expressed as an \textit{iterative kernel update equation}. Importantly we provide theoretical guarantees for this method by leveraging theory for the early stopping
of kernel-based methods (Theorem~\ref{genthmfix} and \ref{genthmpop}).  

\item Utilizing the \textit{neural tangent kernel} for  neural networks with sufficiently large (but not infinite) width we apply our general theoretical bound to feed-forward neural networks (Section~\ref{SecNN}). Further, we use a well-defined kernel to also adapt our bounds to gradient boosted decision trees (Section~\ref{SecGBDT}). Each of these theoretical results is of independent interest. Moreover, if the VI estimator is constructed using neural network,  the asymptotically normality holds and we can build Wald-type confidence interval accordingly (Section \ref{waldnn}).
 \item As an interesting side product of our theoretical results, we find that under warm-start initialization, the global convergence of Neural tangent kernel still hold and the network behave approximately as linear model under mean square error loss (Section \ref{nnlinour}).
 
\item The theoretical bounds are supported in a simulation. 
We then demonstrate the computational advantages over re-training and the the accuracy advantages over dropout for estimating both variable importance and Shapley values for a both simulated data and an application to understanding the variable importance of predicting flue gas emissions. We also test the coverage probability of the Wald-type CI built with neural network for a simulated dataset.  
\end{itemize}

\subsection{Related work} 

The concept of early stopping has been known for some time \citep{esold1, esold2}. In recent years, theoretical insights have emerged, exploring its values in various contexts, such as classification error in boosting algorithms \citep{esboost1, esbtyu, esboost2}, L2-boosting for regression \citep{bl2, l2boost2}, and gradient algorithms in reproducing kernel Hilbert spaces (RKHS) \citep{grdrkhs1, esboost2, esnonpara, esbs}. Notably, \cite{esbtyu} were the first to prove the optimality of early stopping for L2-boosting on spline classes, though their rule was not data-driven. Later, \cite{esnonpara} proposed a data-dependent stopping rule for nonparametric regression under mean-squared error loss. Additionally, \cite{esbs} extended the scope to a wide range of loss functions and boosting algorithms, deriving stopping rules for various kernel classes. In this paper, we use and adapt these early stopping concepts for estimating variable importance for general gradient-based approaches. To the best of our knowledge this is the first time theoretical guarantees have been provided for estimation variable importance for neural networks and gradient boosting.

More recently, advancements in understanding optimization dynamics in overparameterized models have provided new perspectives relevant to early stopping using machinery of RKHS. Deep learning, for instance, typically involves optimizing overparameterized networks using gradient-based methods, a practice that has achieved remarkable empirical success but presents significant theoretical challenges. A major advance came with \cite{jacot}, who introduced the Neural Tangent Kernel (NTK) and showed that, in the infinite-width limit, the kernel remains constant throughout training, governing the optimization dynamics. This framework has been used to establish quantitative guarantees for the global convergence of gradient descent in overparameterized neural networks \citep{ode, ntkre1, ntkre2, ntkre3}. Further, \cite{linearnn} demonstrated that, in this regime, network parameterization becomes approximately linear, while \citep{ode, ntkre2, cntk2} derived NTKs for convolutional networks, albeit focusing on simpler architectures.
In contrast, analogous kernels for gradient-boosting decision trees (GBDTs) remain less explored. \cite{gdbt} defined a kernel for a specific GBDT algorithm using symmetric trees as weak learners. These findings are exploited in this paper to study overparameterized neural networks and GBDTs through a reproducing kernel Hilbert space (RKHS) lens, potentially linking them to theoretical results on early stopping \citep{esnonpara}.

Variable importance (VI) has been a fundamental topic in statistical modeling, with early work focusing on linear models \citep{varlin1,varlin2} where VI is straightforward to compute and interpret. Classical approaches, such as variance-based methods \citep{varbase}, assess the relative importance of variables by quantifying their contributions to model output variance. Recent advances have extended VI to nonparametric settings \citep{remvass}, enabling confidence interval construction using machine learning techniques \citep{willi}. Algorithm-specific measures have also been developed, with extensive studies on random forests \citep{vargbdt1,vargbdt2,vargbdt3} and neural networks \citep{varnn1,varnn2,varnn3,lazyvi}.
With the rise of complex machine learning models, Shapley values \citep{Shapley} have become a prominent tool for feature importance due to their strong theoretical foundation \citep{shapthm,shapthm2}. However, their high computational cost has driven the development of stochastic estimators \citep{shapstoc1,sample} and efficient algorithms tailored to specific model types, such as tree models \citep{shapspe1} and neural networks \citep{shapspe2,shapspe3,shapspe4}. Comprehensive reviews of Shapley-based methods can be found in \cite{shapreview}. Beyond Shapley values, alternative model-agnostic techniques include feature occlusion \citep{willi,flood}, permutation-based measures \citep{permut1,permut2}, and methods like leave-one-covariate-out (LOCO) method \citep{loco1,loco2} and floodgate \citep{flood}, which quantify predictive performance impact. Other approaches focus on conditional independence tests, including generalized covariance measures (GCM) \citep{other1}, rank-based statistics \citep{other2}, and conditional predictive impact (CPI) \citep{other3} that tests whether the feature of interest provides greater predictive power than a corresponding knockoff feature \citep{other4}. In this paper, we provide a scalable approach for estimating VI and Shapley values when the number of features is large.

The work perhaps most closely related to the work in this paper is prior work in ~\cite{lazyvi} which estimates VI through a regularized local linear approximation in a scalable way. This work applied specifically to feed-forward neural networks with sufficiently large width. In this paper, our early stopping approaches to any gradient-based method (e.g. finite but large-width neural networks and gradient-boosted decision trees) and the more precise analysis provides a broader collection of theoretical results for different types of neural tangent kernels.

\section{Notation and Preliminaries}

We adopt similar notation to what is used in \cite{lazyvi} but extend to subsets of features. Suppose we have data $(\mathbf{X}_i, Y_i), i = 1, \dots, N$ for $ (X, Y) \sim P_0$, where $\mathbf{X}_i \in \mathbb{R}^p$ is the $i$-th $p$-dimensional  co-variate, and $Y_i$ is the $i$-th observed predictor. Let $I \subseteq \{ 1, \dots, p \}$, and $k = |I|$, let $X_{- I } \in \mathbb{R}^{p-k }$
denote the features with $j$-th co-variate, $j\in I$, dropped. For simplicity, we drop the $j$-th co-variate by replacing with its marginal mean $\mu_j = \mathbb{E}(X_j)$. Let $P_0, P_{0,-I}$ be the population distributions for $X$ and $X_{-I}$ and let $P_N, P_{N,-I}$ be the empirical distributions of $X$ and $X_{-I}$. 
We denote $\mathbb{E}_0, \mathbb{E}_{0,-I}$ as the expectation taken with respect to $P_0$ and $P_{0,-I}$ respectively.
For notation simplicity, we use
$\boldsymbol{X}$,
$\boldsymbol{X}^{(I)}$ and $\boldsymbol{Y}$  to denote   
$(\mathbf{X}^{ T}_1, \dots, \mathbf{X}^{ T}_N)^T$,
$(\mathbf{X}^{{(I)} T}_1, \dots, \mathbf{X}^{{(I)} T}_N)^T$    and $(Y_1, \dots, Y_N)^T$, respectively.


Let $f_0$ denote the true function mapping $X$ to the expected value of $Y$ conditional on $X$, and let $f_{0,-I}$ denote the function mapping $X_{-I}$ to the expected value of $Y$ conditional on $X_{-I}$ :
\begin{equation}
    \begin{aligned}
    \label{tf}
f_0(X) & :=\mathbb{E}_0[Y \mid X] ; \\
f_{0,-I}\left(X^{(I)}\right) & :=\mathbb{E}_{0,-I}\left[Y \mid X_{-I}\right] .
\end{aligned}
\end{equation}

We assume that the samples take the following form with respect to the $X$ and $X^{(I)}$:
\begin{equation}
    \begin{aligned}
       Y_i &= f_0( \mathbf{X}_i) + w_i \\
       Y_i &= f_{0,-I}( \mathbf{X}^{(I)}_i) + w^{(I)}_i 
    \end{aligned}
    \label{wdef}
\end{equation}
where $w_i$ and $w^{(I)}_i$ are noise random variables independent of $\mathbf{X}_i$ and $\mathbf{X}^{(I)}_i$ respectively. 



To measure the accuracy of an approximation $\hat{f}$ to its target function $f^*$, we use two different ways: the $L^2(P_{N})$-norm
\begin{equation}
\| \hat{f} - f^*\|_N^2 =  \frac{1}{N} \sum_{i=1}^N\left(\hat{f}(\mathbf{X}^{}_i)-f^* \left(\mathbf{X}^{}_i\right)\right)^2
\end{equation}
which evaluates the functions solely at the observed design points, and $L^2(P)$-norm
\begin{equation}
\| \hat{f} - f^*\|_2^2 = \mathbb{E}\left[
\left(
\hat{f}(\mathbf{X})-f^* \left(\mathbf{X}\right)\right)^2\right]
\end{equation}
which is the standard mean-squared error.

Similar to \citep{willi}, the variable importance (VI) measure is defined in terms of a predictive skill metric. The primary predictive skill measure we consider is the negative mean squared error (MSE) as the loss we use is the least-squares loss:
\begin{equation}
    V(f, P)=-\mathbb{E}_{(X, Y) \sim P}[Y-f(X)]^2.
    \label{nmse}
\end{equation}
And the VI measure is defined as 
\begin{equation}
    \mathrm{VI}_I:=V\left(f_0, P_0\right)-V\left(f_{0,-I}, P_{0,-I}\right).
    \label{videf}
\end{equation}

\subsection{Reproducing Kernel Hilbert Spaces}
 
Leveraging theory for the early stopping
of kernel-based methods requires the use
 use of the machinery of reproducing kernel Hilbert spaces (Aronszajn, 1950; Wahba, 1990; Gu and Zhu, 2001).
We consider a Hilbert space $\mathcal{H} \subset L^2(P_{0,-I})$, meaning a family of functions $g: \mathcal{X} \rightarrow \mathbb{R}$, with $\|g\|_{L^2(P_{0,-I})}<\infty$, and an associated inner product $\langle\cdot, \cdot\rangle_{\mathcal{H}}$ under which $\mathcal{H}$ is complete. The space $\mathcal{H}$ is a reproducing kernel Hilbert space (RKHS) if there exists a symmetric kernel function $\mathbb{K}: \mathcal{X} \times \mathcal{X} \rightarrow \mathbb{R}_{+}$such that: (a) for each $x \in \mathcal{X}$, the function $\mathbb{K}(\cdot, x)$ belongs to the Hilbert space $\mathcal{H}$, and (b) we have the reproducing relation $f(x)=\langle f, \mathbb{K}(\cdot, x)\rangle_{\mathcal{H}}$ for all $f \in \mathcal{H}$. Any such kernel function must be positive semidefinite. 
Furthermore, under appropriate regularity conditions, Mercer's theorem \citep{mercer} assures that the kernel can be expressed as an eigen-expansion given by
\begin{equation}
    \mathbb{K}\left(x, x^{\prime}\right)=\sum_{k=1}^{\infty} \lambda_k \phi_k(x) \phi_k\left(x^{\prime}\right),
    \label{polamb}
\end{equation}
where the sequence of eigenvalues
$\lambda_1 \geq \lambda_2 \geq \lambda_3 \geq \ldots \geq 0$ is non-negative, and $\left\{\phi_k\right\}_{k=1}^{\infty}$ represents the corresponding orthonormal eigenfunctions in $L^2(P)$. The rate at which the eigenvalues decay will be critical to our analysis.

We define the associated \textit{empirical kernel matrix} $K^{(I)} \in \mathbb{R}^{N \times N}$ for kernel 
$\mathbb{K}^{(I)}$ on $X_{-I}$ with entries\footnote{The superscript is to denote the dependence of the kernel on dropped data.}
\begin{equation}
\label{eqn:EmpKer}
     K^{(I)}(i,j) = \frac{1}{N} \mathbb{K}^{(I)}(\mathbf{X}^{(I)}_i, \mathbf{X}^{(I)}_j).
\end{equation} 
The eigenvalues of $K^{(I)}$ converge to the population eigenvalues $(\lambda_i)_{i \in \mathbb{N}}$ in (\ref{polamb}) \citep{eigenconver}.
The \textit{empirical kernel matrix} plays a crucial role in our analysis, serving as the foundation for reparameterizing the gradient update equation and establishing the optimal stopping rule for the early stopping procedure. 

We define our algorithm in terms of general kernel matrices and then later discuss how both neural networks and gradient boosted decision trees can be expressed as iterative kernel updates.

\section{Algorithm: Warm-start plus early stopping}
Our primary focus is on estimating $\text{VI}_I$ where $I \subset \{1,2,...,p\}$ for general 
algorithms that employ gradient updates, such as neural networks and gradient-boosting decision trees.

\paragraph{Warm-start.} 
The first step for estimating 
$\text{VI}_I$ is to obtain an estimation of $f_0$. Consider a function class $\mathcal{H}$,
we  train the full  model $f_N^c$ from scratch using  all the features $\boldsymbol{X}$ with respect to the least-square loss:
\begin{equation}
 f^c_N \in \underset{f \in \mathcal{H}}{\arg \min  }    \frac{1}{2 N} \|\boldsymbol{Y} - f(\boldsymbol{X})\|_2^2.
\end{equation}

Given the subset of features to be dropped $I$, the next step is to estimate  the \textit{reduced} model $f_{0,-I}$. Instead of training a model from an arbitrary initialization, we start from the parameters of $f_N^c$ and use gradient descent to optimize the least-square loss for the dropped data over the same function class $\mathcal{H}$ with features in $I$ zeroed out or set to a constant value. This loss given the dropped features becomes:
\begin{equation}
 \mathcal{L}(f) := \frac{1}{2 N} \|\boldsymbol{Y} - f(\boldsymbol{X}^{(I)})\|_2^2.
\end{equation}

Assume a fixed step size $\epsilon$ for gradient descent algorithms, let the model parameters at iteration $\tau$ be denoted by $\theta_{\tau}$.  Let $\nabla_{\theta} f(\theta_\tau) \in \mathbb{R}^{N \times |\theta|}$ represent the gradient of $f_{\tau}$ with respect to $\theta$, evaluated at $\boldsymbol{X}^{(I)}$.
We initialize $\theta_0$ with the model 
parameter of  $f_N^c$, and 
the parameter update via gradient descent is given by:
\begin{equation}
\theta_{\tau+1} = \theta_{\tau} - \frac{\epsilon }{N}
\nabla_{\theta} {f(\theta_\tau)}^{T} 
(  f_{\tau}(\boldsymbol{X}^{(I)})  - \boldsymbol{Y} )
\label{gdup}
\end{equation} 
For gradient boosting algorithms, we consider the gradient descent in functional spaces, we update the function directly rather than the model parameters. We start from $f_N^c$ and the update $f_{\tau}$ via:
\begin{equation}
    f_{\tau+1}=f_\tau-\epsilon \nabla_{\mathcal{H}} \mathcal{L}\left(f_\tau\right)
    \label{fungra}
\end{equation}
where $\nabla_{\mathcal{H}} \mathcal{L}\left(f_\tau\right)$ is the functional gradient of  functional  $\mathcal{L}$ in the space $\mathcal{H}$. We can regard $\nabla_{\mathcal{H}} \mathcal{L}\left(f_\tau\right)$ as a weak learner. 

\paragraph{Early stopping.}
The next important component of our algorithm is \textit{early stopping}. It is well-known that running gradient descent for too many iterations can lead to overfitting, resulting in a poor approximation of the underlying reduced model, $f_{0,-I}$. 
Since we start from $f_N^c$, which is expected to be closer to $f_{0,-I}$ 
compared to a random initialization, we halt the gradient descent early at a specific iteration, denoted by 
$\widehat{T}$, based on certain criteria to prevent overfitting and reduce computational costs.
Then combine these two steps, the estimated VI under this warm-start early stopping approach is:
\begin{equation} \widehat{\mathrm{VI}}_I = \frac{1}{N} \left\{
\| \boldsymbol{Y} - f_{\widehat{T}}\left(\boldsymbol{X}^{(I)}\right)\|_2^2
-\| \boldsymbol{Y} - f_N^c\left(\boldsymbol{X}\right)\|_2^2
\right\} .
\end{equation}

\subsection{Shapley values}
\label{shapintro}
When variables are correlated, the defined VI in (\ref{videf}) approaches zero \citep{lazyvi}. Recent studies suggest using Shapley values for variable importance due to their effective management of correlated variables, assigning similar weights to correlated significant variables
\citep{shapint, sample}. These studies also highlight the high computational cost of Shapley values, requiring a model fit for each subset of variables. However, our algorithm may speed up computing Shapley values by efficiently calculating the quantity defined in (\ref{videf}). 
The Shapley value is defined via a value function \textit{val} of features in $I$.
The Shapley value of a feature value is its contribution to the payout, weighted and summed over all possible feature value combinations:
\begin{equation*}
\sum_{S \subseteq\{1, \ldots, p\} \backslash\{j\}} w_j \left(\operatorname{val}(S \cup\{j\})-\operatorname{val}(S)  \right)
\label{shapdef}
\end{equation*}
where $w_j = \frac{|S|!(p-|S|-1)!}{p!}$ \citep{iml}. In the context of our definition of variable importance~\eqref{videf}
$$
\operatorname{val}(S \cup\{j\})-\operatorname{val}(S) = VI_j^S = V\left(f_{0, -S}, P_{0, -S}\right)-V\left(f_{0,-(S \cup j)}, P_{0,--(S \cup j)}\right).
$$
Combined with a subset sampling scheme proposed by \citep{sample} to reduce the total number of subsets when calculating Shapley values, we are able to estimate the Shapley values using our warm-start early stopping framework with potential lower computational costs. 

\section{Theoretical guarantees} 
\label{theoryassum}
To accurately estimate $VI_{I}$, a reliable estimate for the reduced model $f_{0,-I}$ is crucial. We provide theoretical guarantees for our warm-start early stopping approach in terms of the estimation error between the early-stopped model $f_{\widehat{T}}$ and the target $f_{0,-I}$ in both a fixed design where $\boldsymbol{X}^{(I)}$ remain unchanged and a random design where $\mathbf{X}_i^{(I)}$ are i.i.d. samples from $P_{-I}$ case. We first rewrite the gradient update using the empirical kernel matrix and define the stopping rule. Then, we present a general convergence bound for kernel-based models trained via gradient descent and gradient boosting. We apply this approach to two examples: neural networks with large widths and gradient boosting decision trees using symmetric trees as weak learners with added noise. 
Finally, be leveraging theoretical results from  \cite{willi}, 
we give theoretical guarantees to 
construct a efficient VI estimator with valid asymptotic normality 
using neural network.

\subsection{Kernel gradient update and error decomposition}
Our theoretical analysis largely depends on re-writing the gradient update equations in (\ref{gdup}) and (\ref{fungra}) using the empirical kernel. For both gradient descent and gradient boosting algorithms, let the model at each iteration $\tau$, denoted by $f_{\tau}$, induce a kernel $\mathbb{K}^{(I)}_{\tau}$. When updating the model using (\ref{gdup}) and (\ref{fungra}), suppose that for the model evaluated at $\boldsymbol{X}^{(I)}$, the updating rule can be expressed as follows:
\begin{equation}f_{\tau+1}\left(\boldsymbol{X}^{(I)}\right)=(I-\epsilon K^{(I)}_{\tau}) f_\tau\left(\boldsymbol{X}^{(I)}\right)+\epsilon K^{(I)}_{\tau} \boldsymbol{Y}
\label{evlup}
\end{equation}
where $K^{(I)}_{\tau}$ is the associated empirical kernel matrix of kernel $\mathbb{K}_{\tau}^{(I)}$ as defined in ~\eqref{eqn:EmpKer}. We refer this equation as \textit{iterative kernel update equation}. The recursion (\ref{evlup}) is central to our analysis. In Sections~\ref{SecNN} and ~\ref{SecGBDT} we derive the update equation for neural networks and gradient boosted decision trees.

Further suppose that for the kernel $\mathbb{K}^{(I)}_{\tau}$ converges to a stationary kernel $\mathbb{K}^{(I)}$ under some model specific conditions. Denote empirical kernel matrix of the stationary kernel $\mathbb{K}^{(I)}$ by $K^{(I)}$, by adding and subtracting  $K^{(I)}$ in (\ref{evlup}), we have
\begin{equation}
    f_{\tau+1}\left(\boldsymbol{X}^{(I)}\right)=(I-\epsilon K^{(I)}) f_\tau\left(\boldsymbol{X}^{(I)}\right)+\epsilon K^{(I)} \boldsymbol{Y}+\epsilon \delta_\tau
    \label{addsubeq}
\end{equation}
where $\delta_\tau=\left(K^{(I)}_\tau-K^{(I)}\right)(    \boldsymbol{Y} 
 - f_{\tau}(\boldsymbol{X}^{(I)})
)$. This update equation is identical  to Eq. (19) in \cite{esnonpara} with an additional $\delta_{\tau}$ induced by the evolving kernel $\mathbb{K}^{(I)}_{\tau}$, which adds complexity to the overall analysis.


Let $r=\operatorname{rank}(K^{(I)})$, decompose $K^{(I)}$ using eigendecomposition $K^{(I)}=U \Lambda U^T$, where $U \in \mathbb{R}^{N \times N}$ is an orthonormal matrix and
\begin{equation}  \Lambda:=\operatorname{diag}\left(\hat{\lambda}_1, \hat{\lambda}_2, \ldots, \hat{\lambda}_r, 0,0, \ldots, 0\right)
\end{equation}
is the diagonal matrix of eigenvalues. We then define a sequence of diagonal shrinkage matrices as follows:
\begin{equation}
    S^\tau:=(I-\epsilon \Lambda)^\tau
\end{equation}
Then start from (\ref{addsubeq}), by some direct calculations, we have the following lemma shows that the  prediction error can be bounded in terms of the eigendecomposition and these shrinkage matrices:

\begin{lemma}[Error decomposition]
    At each iteration $\tau=0,1,2, \ldots$,

    \begin{equation}
    \begin{aligned}
        \left\|f_\tau-f_{0,-I}\right\|_N^2 \leq & \underbrace{2 \sum_{j=1}^r\left[S^\tau\right]_{j j}^2\left(\zeta_{j j}^*\right)^2+2 \sum_{j=r+1}^n\left(\zeta_{j j}^*\right)^2}_{\text {Bias } B_\tau^2} 
        +
        \underbrace{\frac{4}{N} \sum_{j=1}^r\left(1-S_{j j}^\tau\right)^2\left[U^T w^{(I)}\right]_j^2}_{\text {Variance } V_\tau} \\
        &  + \underbrace{\frac{4 \epsilon^2}{N}\|\sum_{i=0}^{\tau-1} S^{\tau-1-i} \tilde{\delta}_i\|_2^2}_{\text {Difference } D_\tau^2} 
    \label{lemgenb}
    \end{aligned}
    \end{equation} 
  where $\zeta_{j j}^*=\frac{1}{\sqrt{N}}\left[U^T \left(f_{0,-I}
\left(\boldsymbol{X}^{(I)}
\right)-
f_N^c\left(\boldsymbol{X}^{(I)}
\right)
\right)
\right]_j$, and $\tilde{\delta}_\tau=U^T \delta_\tau$.
\label{applem1}
\end{lemma}
 The difference term arises due to the evolving kernel during training. If the kernel were constant, the difference term $D^2_{\tau}$ in the bound would vanish, and our algorithm would reduce to the early stopping framework proposed in \cite{esnonpara}. The proof of Lemma
 \ref{applem1} is built on lemma 6 in \cite{esnonpara}, which relies on zero initialization.  To adopt our warm-start initialization,  we should analyse the dropout error $\boldsymbol{e}^{(I)} := f_{0,-I}(\boldsymbol{X}^{(I)}) - f_N^c(\boldsymbol{X}^{(I)})$ instead. This is straightforward for gradient boosting because of the additive structure of updating $f_\tau$. For gradient descent algorithms, we need to linearize the model which will induce additional errors which we address later. We will see this happening in the case of neural networks. When analyzing $\boldsymbol{e}^{(I)}$, the underlying true function becomes $f_{0,-I} - f_N^c$. Thus, the training process can be interpreted approximately as  starting from $0$ and using functions from $\text{span}\left\{\mathbb{K}^{(I)}(\cdot, X_{-I})\right\}$ to approximate  the shifted target $f_{0,-I} - f_N^c$.


\subsection{Stopping rule and general bound}
The stopping threshold is closely related to a model complexity measure, known as the local empirical Rademacher complexity \citep{Mendelson02}. In this paper, it takes the form
\begin{equation}
    \widehat{\mathcal{R}}_K(\varrho):=\left[\frac{1}{N} \sum_{i=1}^N \min \left\{\widehat{\lambda}_i, \varrho^2\right\}\right]^{1 / 2}
    \label{localran}
\end{equation}
where $\widehat{\lambda}_i$ is the eigenvalue of $K^{(I)}$.

Let $\mathcal{H}$ denote the RKHS induced by the stationary kernel $\mathbb{K}^{(I)}$, and 
denote the Hilbert norm in $\mathcal{H}$ by $\| \cdot \|_{\mathcal{H}}$.
We make the following assumptions in this paper.
\begin{assumption}
    $f_N^c$ and $f_{0,-I}$ belongs to $\mathcal{H}$, i.e., $f_N^c, f_{0,-I} \in \text{span} \{ \mathbb{K}^{(I)}(\cdot, X_{-I}) \}$.
    \label{inass}
\end{assumption}
This assumption is made for purely theoretical convenience, avoiding  the need to  add additional  mis-specification error terms.

 \begin{assumption}
    The data $\left\{(\mathbf{X}_i, Y_i)\right\}_{i=1}^N$ are contained in  closed and bounded set in $\mathbb{R}^{p+1}$.
    \label{xb}
\end{assumption}
\begin{assumption}
     $w_i^{(I)}$ are independent zero-mean random variables satisfying the following condition:
     \label{rmdassum}
\begin{equation}
   \mathbb{E}\left[e^{t w_i^{(I)}}\right] \leq e^{t^2 \sigma^2 / 2}, \text { for all } t \in \mathbb{R} . 
\end{equation}
\end{assumption}
\begin{assumption}
    The dropout error does not explode satisfying $\|\boldsymbol{Y} - f_N^c(\boldsymbol{X}^{(I)}) \|_2 =  O(\sqrt{N})$.
    \label{dropb}
\end{assumption}
\begin{assumption}
   The trace of $K^{(I)}$ satisfies $tr(K^{(I)})  = O(1).$
   \label{trass}
\end{assumption}
 For certain kernel classes, such as kernels with polynomial eigen-decay  and finite rank kernels, Assumption~\ref{trass} is satisfied.

For a given noise variance $\sigma>0$, we define the critical empirical radius $\widehat{\varrho}_N>0$, to be the smallest positive solution for the inequality
\begin{equation}
    \widehat{\mathcal{R}}_K(\varrho) \leq \frac{\varrho^2 C_{\mathcal{H}}^2
    }{2 e \sigma}
    \label{epdef}
\end{equation}
where $C_{\mathcal{H}}$ is $\left\|f_N^c -f_{0,-I}\right\|_{\mathcal{H}}$, the quantity to measure the difference between our start and the target. 

Define $\eta_{\tau}=\tau \epsilon$ to be the sum of step sizes over $\tau$ iterations.  The stopping threshold $\widehat{T}_{max}$ is defined as 
\begin{equation*}
    \widehat{T}_{\max }:=\arg \min \left\{\tau \in \mathbb{N} \left\lvert\, \widehat{\mathcal{R}}_K\left(1 / \sqrt{\eta_\tau}\right)>\frac{C^2_{\mathcal{H}}}{2 e \sigma \eta_\tau}\right.\right\}-1
\end{equation*}
the integer $\widehat{T}_{\max }$ belongs to the interval $[0, \infty)$ and is unique in our setup. And $\widehat{T}_{\max }$ optimized the sum of $B_\tau^2$ and  $V_\tau$ in (\ref{lemgenb}).  According to \cite{esnonpara}, for iteration $\tau \leq \widehat{T}_{\max }$,  with certain probability, we can bound $B_\tau^2$ and  $V_\tau$ as
\begin{equation*}
    B_\tau^2 + V_\tau \leq \frac{C}{\eta_{\tau}}. 
\end{equation*}
The  difference term $D^2_{\tau}$ is non-decreasing with  $\tau$.
Suppose  $D^2_{\tau}$ is upper bounded by a function $g(\tau)$, which is also non-decreasing with  $\tau$.
Finally, our proposed optimal stopping time $\widehat{T}_{\text{op} }$ is given by:
\begin{equation*}
    \widehat{T}_{\text{op} }:=\arg \min \left\{\tau \leq \widehat{T}_{\max }  \left\lvert\, 
  \frac{C}{\eta_{\tau}} + 
        g(\tau)  \right.\right\}.
\end{equation*}
The reason we should stop the gradient updates early is reflected in two aspects of the theory. 
First, the quantity $C_{\mathcal{H}}$ limits $\widehat{T}_{\max}$ from being too large, 
as we expect $f_N^c$ to be reasonably close to $f_{0,-I}$. 
Second, the difference term $D_{\tau}^2$ in the bound (\ref{lemgenb}) 
favors a smaller number of iterations, since running for too many iterations 
would result in larger difference errors in the bound.

\begin{theorem}[General convergence bound under fixed design]
 Under assumptions \ref{inass}-\ref{trass}, 
 consider a model  that  use gradient descent or gradient boosting with step size $\epsilon \leq \min\{1, 1/\widehat{\lambda}_1  \}$, 
 and under some additional model-specific conditions,
 start from the  
 full model $f_N^c$, and stop the  update early at iteration $\widehat{T}_{\text{op}}$,
with  probability at least $1-c_1\exp( -c_2 N \widehat{\varrho}_N )$,  the following bound holds
\begin{equation*}
      \|f_{\widehat{T}_{op}} - f_{0,-I}\|_N^2      \leq \mathcal{O}\left(
      \frac{1}{N^{\frac{1}{2}}}
      \right).
\end{equation*}
where $c_1$ and $c_2$ are some universal positive constants.
\label{genthmfix}
\end{theorem}

Note that Theorem \ref{genthmfix} applies to the case of fixed design points  $\boldsymbol{X}^{(I)}$, so that the probability is considered only over the sub-Gaussian noise variables $w^{(I)}$. 
For random design points, $\mathbf{X}^{(I)}_i \sim^{i.i.d.} P_{-I}$, we can also establish bounds on the generalization error in terms of the $L^2(P_{-I})$-norm. 
In this setting, it is beneficial to express certain results using the population version of the local empirical Rademacher complexity from (\ref{localran})
\begin{equation}
    \mathcal{R}_{\mathbb{K}}(\varrho):=\left[\frac{1}{N} \sum_{j=1}^{\infty} \min \left\{\lambda_j, \varrho^2\right\}\right]^{1 / 2}
\end{equation}
where $\lambda_j$ represents the eigenvalues of the population kernel $\mathbb{K}^{(I)}$, as defined in (\ref{polamb}). 
Based on this complexity measure, we define the critical population rate $\varrho_N$ as the smallest positive solution to the inequality\footnote{The prefactor 40 is chosen for theoretical convenience same as in \cite{esnonpara}.}
\begin{equation}
    \mathcal{R}_{\mathbb{K}}(\varrho) \leq \frac{\varrho^2 C_{\mathcal{H}}^2}{40\sigma}.
\end{equation} 
Unlike the critical empirical rate $\widehat{\varrho}_n$, this value is not data-dependent, as it is determined by the population eigenvalues of the kernel operator underlying the RKHS. We make the following additional assumption on the sum of eigenvalues to limit the decay rate of $\lambda_j$
and 
be consistent with assumption \ref{trass} in fixed desgin case.
\begin{assumption}
    The sum of eigenvalues of the population kernel $\mathbb{K}^{(I)}$, $\sum_i \lambda_i$ does not diverge. 
    \label{trasspop}
\end{assumption}

\begin{theorem}[General convergence bound under random design]
 Under assumptions \ref{inass}-\ref{dropb} and assumption \ref{trasspop},  and suppose that $\boldsymbol{X}_i^{(I)}$ are sampled i.i.d. from $P_{-I}$, 
 consider a model  that  use gradient descent or gradient boosting with step size $\epsilon \leq \min\{1, 1/\widehat{\lambda}_1  \}$, 
 and under some additional model-specific conditions,
 start from the  
 full model $f_N^c$, and stop the  update early at iteration $\widehat{T}_{\text{op}}$,
with  probability at least $1-c_1\exp( -c_2 N \varrho_N )$,  the following bound holds
\begin{equation*}
      \|f_{\widehat{T}_{op}} - f_{0,-I}\|_2^2      \leq \mathcal{O}\left(
      N^{-\frac{1}{2}}\right).
\end{equation*}
where $c_1$ and $c_2$ are some universal positive constants.
\label{genthmpop}
\end{theorem}
Theorem \ref{genthmfix} and \ref{genthmpop} applies  for any models using 
gradient boosting or gradient descent meeting the required assumptions.  The proofs are built on the techniques in \cite{esnonpara}
, which relies on techniques from empirical process theory and concentration of measure. The major challenge for the proofs are controlling error induced by the evolving kernels during training in both empirical norm and population norm. This is highly model specific and we need  to 
use different proof strategies for different models. See discussions in  Section \ref{SecNN} and  Section \ref{SecGBDT} for neural networks and GBDT respectively. 
\begin{remark}
    The same theoretical results should also hold for the kernel ridge regression estimator counterpart. And we can apply the same proof techniques in this paper with potentially some minor modifications, similar to the claim in \cite{esnonpara}.
\end{remark}

\subsection{Neural networks}
\label{SecNN}
Neural networks, as powerful black-box models, are widely used for accurate predictions across various scenarios. They typically employ gradient-based methods to update parameters, making them ideal candidates for our general algorithm. With sufficiently large widths and the theoretical insights from the \textit{neural tangent kernel}, we can analyze neural networks in the context of kernel-based methods and derive corresponding theoretical guarantees.

First introduced in \cite{jacot}, the Neural Tangent Kernel (NTK) provides a theoretical tool to study the neural network in the RKHS regime. Denote a neural network by $f(\theta, x)$, the corresponding \textit{Neural Tangent Kernel} is defined as
\begin{equation}
    \left\langle
     \nabla_{\theta} f\left(\theta, x\right),
     \nabla_{\theta} f\left(\theta, x^{\prime}\right)\right\rangle.
\end{equation}
Consider a fully connected neural network with width $m$ using gradient descent to update parameters. 
At each  iteration, the network $f_{\tau}$ induces a corresponding NTK, denoted by $\mathbb{K}_{\tau}$. As shown in~\citep{jacot}, under random initializations and assuming a continuous gradient flow, for a network of infinite width, $\mathbb{K}_{\tau}$ remains constant during training, and moreover $\mathbb{K}_{0}
$ converges  to certain deterministic kernel $\mathbb{K}$ when widths goes to infinity
\citep{jacot}. And when we  use least-squares loss, infinite-width networks behave as linearized networks \citep{linearnn}.

In this paper we consider both NTK parameterization and standard parameterization for  
fully-connected feed-forward neural networks.
\paragraph{NTK parameterization.}
Consider a network
with $L$ hidden layers with widths $n_l = m$, for $l=1, \ldots, L$ and a readout layer with $n_{L+1}=1$.\footnote{We set the output dimension to 1 for illustration purposes. The theoretical results can be easily extended to cases where $n_{L+1} = k$.} For each $x \in \mathbb{R}^{n_0}$, we use $h^l(x), x^l(x) \in \mathbb{R}^{n_l}$ to represent the pre- and post-activation functions at layer $l$ with input $x$. The recurrence relation for a feed-forward network is defined as
\begin{equation}
    \left\{\begin{array} { l l } 
{ h ^ { l + 1 } = x ^ { l } W ^ { l + 1 } + b ^ { l + 1 } } \\
{ x ^ { l + 1 } = \phi ( h ^ { l + 1 } ) }
\end{array} \text { and } \left\{\begin{array}{ll}
W_{i, j}^l & =\frac{\sigma_\omega}{\sqrt{n_l}} \omega_{i j}^l \\
b_j^l & =\sigma_b \beta_j^l
\end{array}\right.\right.
\label{ntknorm}
\end{equation}
where $\phi$ is a point-wise activation function, $W^{l+1} \in \mathbb{R}^{n_l \times n_{l+1}}$ and $b^{l+1} \in \mathbb{R}^{n_{l+1}}$ are the weights and biases, $\omega_{i j}^l$ and $b_j^l$ are the trainable variables, drawn i.i.d. from a standard Gaussian $\omega_{i j}^l, \beta_j^l \sim \mathcal{N}(0,1)$ at initialization, and $\sigma_\omega^2$ and $\sigma_b^2$ are weight and bias variances. We refer to it as the NTK parameterization. This is a widely used parameterization in NTK literature \citep{jacot,ntkpara2,ntkpara3,ntkpara4} for theoretical convenience. 
\paragraph{Standard parameterization.}
 A network with standard parameterization is  defined as: 
\begin{equation}
    \left\{\begin{array} { l l l } 
{ h ^ { l + 1 } } & { = x ^ { l } W ^ { l + 1 } + b ^ { l + 1 } } \\
{ x ^ { l + 1 } } & { = \phi ( h ^ { l + 1 } ) }
\end{array} \text { and } \left\{\begin{array}{ll}
W_{i, j}^l & =\omega_{i j}^l \sim \mathcal{N}\left(0, \frac{\sigma_\omega^2}{n_l}\right) \\
b_j^l & =\beta_j^l \sim \mathcal{N}\left(0, \sigma_b^2\right)
\end{array}\right.\right. . 
\label{stdnorm}
\end{equation}
This construction is more commonly used for training neural networks. The network represents the same function under both NTK and standard parameterization, but their training dynamics under gradient descent usually differ. Nonetheless, using a carefully chosen layer-dependent learning rate can make the training dynamics equivalent \citep{linearnn}.

In order to fit neural network into our general framework. We first need 
to verify assumption \ref{trass}. As shown by \citet{jacot}, the NTK can be computed recursively using specific formulas. Under the assumption \ref{xb}, the input data is bounded, we can ensure that $tr(K^{(I)})$ remains bounded as well. Next, we show that the update can be expressed similarly to equation (\ref{evlup}).
Apply first-order Taylor's expansion and plug in the gradient update equation (\ref{gdup}), we have 
\begin{equation}
    f_{\tau+1}(\boldsymbol{X}^{(I)}) =    f_{\tau}(\boldsymbol{X}^{(I)}) 
     - \frac{\epsilon }{N}
     \nabla_{\theta} {f(\theta_\tau)}
\nabla_{\theta} {f(\theta_\tau)}^{T} g_\tau(\boldsymbol{X}^{(I)})
\label{genrep}
\end{equation}
where $g_\tau(\boldsymbol{X}^{(I)}) := f(\boldsymbol{X}^{(I)}) - \boldsymbol{Y}$. 
By the definition of the NTK, we derive the same update equation as in (\ref{evlup}). However, since neural networks are not linear models, there will be an additional linearization error. By extending Theorem 2.1 from \cite{linearnn} to the warm-start initialization setting, the $L_2$ norm of this error is bounded by $\mathcal{O}(m^{-1/2})$. Moreover, the NTK stability result remains valid in our setup, allowing us to bound the difference between $K_{\tau}^{(I)}$ and $K^{(I)}$.
These results are crucial to our analysis, as we need to write the update equation for neural networks precisely. 
See Appendix \ref{nnlinour} for the details derivation of these theoretical results.  

We add the following model-specific assumptions for neural networks. 
\begin{assumption}
    The empirical kernel matrix induced by the stationary NTK 
    , as width  $ m \rightarrow \infty$, $K^{(I)}$ is full rank, i.e.  $0<\lambda_{\min }:=\lambda_{\min }(K^{(I)}) \leq$ $\lambda_{\max }:=\lambda_{\max }(K^{(I)})<\infty$. Let $\eta_{\text {critical }}=2\left(\lambda_{\min }+\lambda_{\max }\right)^{-1}$.
  \label{linea1}  
\end{assumption}
  Note that   the analytic NTK matrix    in \cite{linearnn} is different than     
    the empirical kernel matrix we consider.
    As ours is normalized by the sample size $N$, and since the loss we use is also normalized by $N$, the effects cancel out. This is why we can make assumption about $K^{(I)}$ directly here. 
 \begin{assumption}
 The activation function $\phi$ satisfies
\begin{equation}
    |\phi(0)|, \quad\left\|\phi^{\prime}\right\|_{\infty}, \quad \sup _{x \neq \tilde{x}}\left|\phi^{\prime}(x)-\phi^{\prime}(\tilde{x})\right| /|x-\tilde{x}|<\infty.
\end{equation}
 \end{assumption}
 \begin{assumption}
  \label{linea4}
     The full model $f_N^c$ has same structure as the reduced model, i.e. width, layer, dimension, 
     and 
     is trained under normal random initialization as in (\ref{ntknorm})
     or 
     (\ref{stdnorm}) using complete features $\boldsymbol{X}$ using gradient descent with learning rate $\epsilon=\frac{\eta_0}{m}$ or $\frac{\eta_0}{m}$ for NTK and standard parameterization, respectively.  \end{assumption}
The following lemma enables us to write the \textit{iterative kernel update equation} for neural network in a rigorous sense. 
\begin{lemma}
\label{nnlemma}
Under assumptions \ref{linea1}-\ref{linea4}, for a neural network applying warm-start initialization and use 
gradient descent with learning rate $\epsilon = \frac{\eta_0}{m}$ for standard parameterization 
and $\epsilon = \eta_0$ for NTK parameterization,  
and $\eta_0 < \eta_{\text {critical }}$, 
the iterative kernel update equation for neural network can be written as
\begin{equation}
  f_{\tau+1}({\boldsymbol{X}^{(I)}}) = (I-\eta_0 K^{(I)}) f_{\tau}({\boldsymbol{X}^{(I)}}) + \eta_0 K^{(I)} \boldsymbol{Y} + \eta_0  \delta_{\tau}
  \label{nnupdaeq}
\end{equation}
where $\delta_{\tau}$ is a term containing errors caused by the changing kernel and linearization of neural networks. 
And
for any $\gamma > 0$, there exists $M \in \mathbb{N}$, for  a network described above
with width $m \geq M$, 
the following holds with probability at least $1-\gamma$ 
\begin{equation}
    \|\delta_{\tau}\|_2 \leq \mathcal{O}\left(m^{-\frac{1}{2}}\right).
\end{equation}

\end{lemma}
Making use of this result and applying the same proof technique for the general algorithm, we are able to derive the following convergence bound for feed-forward neural networks under fixed design.
\begin{corollary}[Error bound for neural network under fixed design] 
For any $\gamma > 0$, there exists an $M \in \mathbb{N}$, such that for a fully-connected neural network with width $m \geq M$, if we use our warm-start initialization when applying gradient descent with learning rate $\epsilon = \frac{\eta_0}{m}$ for standard parameterization and $\epsilon = \eta_0$ for NTK parameterization, then
\label{cor1}
\begin{equation}
     \|f_{\widehat{T}_{op}} - f_{0,-I}\|_N^2       \leq \mathcal{O}\left( N^{-\frac{1}{2}}\right)
\end{equation}
with probability at least $1- \gamma - c_1 \exp \left(-c_2 N \widehat{\varrho}_N^2\right)$.
\end{corollary}

To derive the population norm convergence bound we need to know the specific eigenvalue decay rate  for the RKHS induced by NTK, which is a complicated problem. Existing results consider input data normalized as $x = \frac{x}{\|x\|_2} \in \mathbb{S}^{p-1}$. For a two-layer ReLu network without bias, 
\citep{egdecnnused, nndecay3} proved that the eigenvalues $\lambda_i$ decay at a rate of $O\left(i^{-p}\right)$. 
\citep{laplace2} showed  that the NTK and Laplace Kernel have the same RKHS, which means that $\lambda_i \sim i^{-p}$; 
\citep{nnedecay2} proved  that the  eigenvalues decay at a rate no faster than  $O\left(i^{-p}\right)$ and gave empirical results showing that the
eigenvalues
decay exactly as  $\Theta \left(i^{-p}\right)$. However they all assume that the bias term $b$ in the feed-forward neural network parameterization are 
initialized to zero, which is different than our setup. To ensure theoretical preciseness, we use result from \cite{egdecnnused} and \cite{egedecay}, and give the following population norm bound for ReLu network with bias $b=0$.

\begin{corollary} In addition to conditions of corollary \ref{cor1}, normalize input data as $x = \frac{x}{\|x\|_2} \in \mathbb{S}^{p-1}$,  
for any $\gamma > 0$, there exists $M \in \mathbb{N}$, for  ReLu network with bias $b=0$ and width $m \geq M$, the following holds with probability at least $1- \gamma - c_1 \exp \left(-c_2 N {\varrho}_N^2\right)$ 
\begin{equation}
     \|f_{\widehat{T}_{op}} - f_{0,-I}\|_2^2       \leq \mathcal{O}\left( N^{-\frac{p}{p+1}}\right).
\end{equation}
\label{popnn}
\end{corollary}
Even though this population convergence bound is restricted to ReLU networks with $0$ bias, empirical results for networks with non-zero bias \citep{nnedecay2} suggest that the above results should also apply to ReLU networks with bias. For general feed-forward neural networks with features not normalized to lie on the sphere, we can expect a convergence rate of at least $O(N^{-\frac{1}{2}})$. This is because the general constraint, specifically the bounded sum of eigenvalues, should hold in general since the trace of the empirical kernel matrix is always bounded in our setup.  And 
by classical results \citep{eigenconver}, the empirical eigenvalues $\widehat{\lambda}_i$ would converge to the population counterpart $\lambda_i$, so the $\sum_i \lambda_i$ should also converge, which lead to a convergence rate of at least $O(N^{-\frac{1}{2}})$.


The convergence results may seem counterintuitive at first glance since rates get faster as $p$ grows, but the underlying reason is that the eigenvalue decay rate is faster in higher dimensions. This is due to the increasing complexity of function spaces on the hypersphere. Since the RKHS of the NTK matches that of the Laplace kernel \citep{laplace2}, and the Laplace kernel enforces smoothness, becoming more stringent as dimensionality increases, this leads to a rapid drop-off in eigenvalues as $p$ grows. Essentially, higher dimensional spheres ``squeeze'' the RKHS, resulting in a faster decay in the spectrum of the kernel matrix. In other words, the magnitude of the enforcing smoothness because of the constrain of hyper-sphere is stronger than the magnitude of increasing function class in higher dimensions.




We do not specify a particular rate for $m$ in relation to $N$. However, we can always select a sufficiently large width to ensure that the difference term $D^2_{\tau}$ is bounded within the desired rate. Our results are expressed in terms of the existence of $M \in \mathbb{N}$. If the width $m \rightarrow \infty$, the difference term would become zero, reducing to the original framework of \cite{esnonpara}. However, this is impractical in real-world applications, and such a strong condition is unnecessary.

One major advantage of using neural networks is that, due to universal approximation theory for arbitrary width and bounded depth \citep{unapp1, unapp2}, we can approximate any continuous target function. This makes our approach a model-agnostic VI estimation method, as it doesn't rely on strong assumptions about $f_{0,-I}$.

 \subsection{Gradient boosted decision trees}
 \label{SecGBDT}
A classic gradient boosting algorithm \citep{bsalgo} iteratively minimizes the expected loss $\mathcal{L}(f) = \mathbb{E}[L(f(x), y)]$ using weak learners. In each iteration $\tau$, the model is updated as $f_\tau(x) = f_{\tau-1}(x) + \epsilon w_\tau(x)$, where the weak learner $w_\tau$ is selected from a function class $\mathcal{W}$ to approximate the negative gradient of the loss function. When $\mathcal{W}$ consists of decision trees, the algorithm is known as Gradient Boosting Decision Trees (GBDT). A decision tree recursively partitions the feature space into disjoint regions called leaves, with each leaf $R_j$ assigned a value representing the estimated response $y$ for that region \citep{gdbt}.

For theoretical convenience, given the complexity of standard GBDT, we explore a simpler GBDT algorithm that uses symmetric trees as weak learners and incorporates additional noise introduced in \cite{gdbt}. Each weak learner, denoted by $\nu$ is a symmetric, oblivious tree, i.e., all nodes at a given level share the same splitting criterion (feature and threshold). Specifically, each weak learner $\nu$ is chosen by a \textit{SampleTree} algorithm with   noise level  controlled by a parameter $\beta$, which is referred to as random strength.
To limit the number of candidate splits, each feature is quantized into $n+1$ bins. The maximum tree depth is limited by $d$.  

The parameter random strength, $\beta$ is an important parameter both in  the model implementations and theoretical analysis. For 
typical GBDT algorithms, the new weak learner is chosen by maximizing a predefined score function. Suppose we want to maximize score $D$. The random noise is then added to the score as follows (see 
Algorithm \ref{stalgo} in Appendix \ref{sampletreealgo}
for details about this procedure):
\begin{equation}
\label{chosetree}
    D -\beta \log(-\log(u))
\end{equation}
where $u$ is a sample from uniform distribution $U[0,1]$. Because $u$ is a random variable, each time we sample a different $u$, we would choose  a different  week learner, which induces a distribution over all possible trees at each iteration, which we denote as $p(\nu | f_{\tau}, \beta)$. These distributions are associated with the kernels to be defined for the GBDT we are analyzing.  And the parameter $\beta$ enables us to control the difference between between $\mathbb{K}^{(I)}_{\tau}$ and $\mathbb{K}_{\tau}$, which will become evident later.  Generally speaking, the parameter $\beta$ for GBDT acts as the same way as the width $m$ for neural network, which is mainly a hyperparameter to control the difference term $D^2_{\tau}$ in the final bound.

\paragraph{Tree structure.}
Let $\mathcal{V}$ represent all tree structures. Each $\nu$ corresponds to a decision tree. Let 
$\phi_{\nu}: X \rightarrow \{0,1\}^{L_{\nu}}$.  Having $\phi_{\nu}$, define a weak learner associated with it as $x \mapsto \left \langle\theta, \phi_{\nu}(x)  \right \rangle_{\mathbb{R}^{L_{\nu}}}$, for $\theta \in \mathbb{R}^{L_{\nu}}$, which is leaf values. Define a linear space $\mathcal{F} \subset L_2(\rho)$
 of all possible ensembles of trees from $\mathcal{V}$:
\begin{equation}
    \mathcal{F} = \text{span}\left\{ \phi_{\nu}^{(j)}(\cdot): X \rightarrow \{0,1 \} | \nu \in \mathcal{V}, j \in\{1, \dots, L_{\nu}\} \right\}.
\end{equation}
Since $\mathcal{V}$ is finite and therefore $\mathcal{F}$ is finite-dimensional and thus topologically closed.

\paragraph{RKHS structure.}
The RKHS structure for this specific GBDT is based on a kernel defined for each tree structure, i.e. weak learner,  denoted by $k_{\nu}(\cdot, \cdot)$, which is defined as 
\begin{equation}
    k_\nu\left(x, x^{\prime}\right)=\sum_{j=1}^{L_\nu} w_\nu^{(j)} \phi_\nu^{(j)}(x) \phi_\nu^{(j)}\left(x^{\prime}\right), \text { where } w_\nu^{(j)}=\frac{N}{\max \left\{N_\nu^{(j)}, 1\right\}}, N_\nu^{(j)}=\sum_{i=1}^N \phi_\nu^{(j)}\left(x_i\right) .
    \label{wekernel}
\end{equation}
Then $\mathbb{K}^{(I)}_{\tau}$ is given by
\begin{equation}
    \mathbb{K}_{\tau}\left(x, x^{\prime}\right)=\sum_{\nu \in \mathcal{V}} k_\nu\left(x, x^{\prime}\right) p(\nu \mid f_\tau, \beta) .
\end{equation}
where $p(\nu \mid f_\tau, \beta)$ is the induced distribution at iteration $\tau$ by  the added noise.
 The kernel during training  $\mathbb{K}_{\tau}^{(I)}$ converges to a certain stationary kernel  $\mathbb{K}^{(I)}$ as the
model approaches the empirical minimizer $f_{*}$, i.e. when $\tau \rightarrow \infty$. 
The stationary kernel $\mathbb{K}$ is defined as
\begin{equation}
    \mathbb{K}\left(x, x^{\prime}\right)=\sum_{\nu \in \mathcal{V}} k_\nu\left(x, x^{\prime}\right) \pi(\nu)
\end{equation}
where $\pi(\nu)$ is
a uniform distribution over $\mathcal{V}$. Since $\pi(\nu)$ does not involve any $f_\tau$,  $\mathbb{K}^{(I)}$ is independent with  $ \mathbb{K}^{(I)}_{\tau}$. 
It turns out that the kernel class for the GBDT belongs to the finite rank kernels, and we can bound $tr(K^{(I)})$ by $2^d$, which satisfies the assumption \ref{trass}. 

In terms of the iterative kernel update equation,
Lemma 3.7 in \cite{gdbt} allows us to express the model updates similarly to equation (\ref{evlup}), but incorporating the distribution $p(\nu \mid f_{\tau},\beta)$. To align with our general framework, we take the expectation with respect to the algorithm's randomness on both sides of the update equations.
Note that since $f_N^c$ is random, the dropout error $\boldsymbol{e}^{(I)}$ is also random.  The \textit{iterative kernel update equation} for the GBDT is given by the following lemma. 
\begin{lemma}
\label{treekernel}
    The iterative kernel update equation for  the particular GBDT we consider is:
    \begin{equation}
    \begin{aligned}\mathbb{E}_u f_{\tau+1}(\boldsymbol{X}^{(I)}) 
&= (I-\epsilon \mathbb{E}K^{(I)})\mathbb{E}_u f_{\tau}(\boldsymbol{X}^{(I)}) + \epsilon \mathbb{E}_u K^{(I)} \mathbb{E}_u \boldsymbol{e}^{(I)} + \epsilon\delta_\tau
\end{aligned}
\end{equation}
where $\delta_{\tau} = \mathbb{E}_u(K^{(I)}-K^{(I)}_{\tau})[f_{\tau}(\boldsymbol{X}^{(I)}) - \boldsymbol{e}^{(I)}]$ and the expectation is taken w.r.t. the randomness of the \textit{SampleTree} algorithm.
\end{lemma}

When $\beta \rightarrow \infty$, $p(\nu \mid f_{\tau},\beta)$  would be simply  $\pi(\nu)$. This means we can bound
The difference between $p(\nu \mid f_{\tau},\beta)$ and $\pi(\nu)$ by setting the random strength $\beta$ to certain values, so we can control the term $\delta_\tau$. By setting $\beta$  greater than a certain threshold we can bound the difference term $D_\tau^2$ in the case of GBDT. We then have the following corollary stating the convergence result under fixed design. 
 \begin{corollary}
 For the GBDT algorithm mentioned above with depth $d$, and random strength $\beta \geq  N^{5/4}$, the following holds with probability at least $1- c_3 \exp \left(-c_4 N \widehat{\varrho}_N^2\right)$
 over warm-start initialization when applying gradient boosting with learning rate $\epsilon \leq \frac{1}{4 M_\beta N}$
\begin{equation}
  \left\|\mathbb{E}_u  f_{\widehat{T}_{op}}-f_{0,-I}\right\|_N^2 \leq \mathcal{O}\left(
   \sqrt{\frac{2^d}{N}}
   \right).
\end{equation}
where $M_{\beta} =  e^{\frac{ d \cdot \|f_{*}(\boldsymbol{X}^{(I)})\|_2^2}{N \beta}}$ and $f_{*}$ is the empirical minimizer. 
\label{treecorfix}
\end{corollary}
\begin{remark}
 If we do not account for the randomness of the GBDT 
 algorithm, $ \mathbb{E} \left\| f_{\widehat{T}_{op}}-f_{0,-I}\right\|_N^2$ will include a variance component that does not converge to zero, even as $N \rightarrow \infty$, due to the algorithm's inherent randomness. However, this inherent variance is quite small as we display in the experimental results.
\end{remark}
To derive the population norm bound, we need to know the exact eigenvalue decay rate for GBDT.  By the definition of the kernels, it is not hard to see that the GBDT kernel belongs to finite kernel class. So we have the following population norm bound for GBDT.
 \begin{corollary}
 For the GBDT algorithm mentioned above with depth $d$, and random strength $\beta \geq N^{15/4}$, the following holds with probability at least $1- c_3 \exp \left(-c_4 N {\varrho}_N^2\right)$
 over warm-start initialization when applying gradient boosting with learning rate $\epsilon \leq \frac{1}{4 M_\beta N}$
\begin{equation}
  \left\|\mathbb{E}_u  f_{\widehat{T}_{op}}-f_{0,-I}\right\|_2^2 \leq \mathcal{O}\left(
  \sqrt{\frac{1}{N}}
   \right).
\end{equation}
\label{treecorpop}
\end{corollary}

\subsection{Extension to hypothesis testing} \label{waldnn}
According to  Theorem  1
in \citep{willi}, when the empirical estimates converges to the targets $f_0$ and $f_{0,-I}$ in $L_2$ norm at the rate of $O_p(N^{-1/4})$, we
 can obtain an asymptotically normal and efficient estimator
for the VI measure. This means that for certain models that we intend to use, the probability of the convergence bound we proved above should converge to zero as $N \rightarrow \infty$, i.e. the term $N\varrho^2 \rightarrow \infty$. This requirement is satisfied for neural network but not met by GBDT since for GBDT $\varrho^2$ is at rate of  $O(1/N)$, which results a constant for  $N\varrho^2$. Therefore we have the following result ensuring an   efficient VI estimator with valid asymptotically normality if we use neural networks. 

\begin{corollary}
\label{testcor}
Same as corollary \ref{popnn}, 
   the warm-start early stopping training method using 
   a  ReLu neural network without bias
   can accurately predict the reduced model, i.e.,
\begin{equation}
    \left\|f_{\widehat{T}_{op}} - f_{0,-I}\right\|_2=O_p\left(N^{-1 / 4}\right) .
\end{equation}
Then if $\mathrm{VI}_I \neq 0$, 
our variable importance estimator $\widehat{\mathrm{VI}}_I$ is asymptotically normal and has an error rate $O_p\left(N^{-1 / 2}\right)$:
\begin{equation}
    \widehat{\mathrm{VI}}_I^{}-\mathrm{VI}_I=\Delta_{N, I}+O_p\left(N^{-1 / 2}\right) 
\end{equation}
where
\begin{equation}
    \Delta_{N, I}\rightarrow_d \mathcal{N}\left(0, \tau_{N, j}^2\right)
    \label{varpop}
\end{equation}
here the variance is $\tau_{N, j}^2=\operatorname{Var}\left(w^{(I)^2}-w^2\right) / N$, where $w$ and $w^{(I)}$ are the population version of the residuals defined in eq. (\ref{wdef}).
\end{corollary}
For general networks, even though we do not provide rigorous 
theoretical results, we should expect the same asymptotical normality to hold. We are able to construct Wald-type confidence  intervals around the estimated VI using our proposed framework. Specifically, 
the $\alpha$-level confidence intervals are given by
\begin{equation}
    \widehat{\mathrm{VI}}_I\pm z_{\frac{\alpha}{2}} \cdot \hat{\tau}_{N, I}
\end{equation}
where $\hat{\tau}_{N, I}$ is the plug-in estimate of $\tau_{N, I}$ in (\ref{varpop}) and $z_{\frac{\alpha}{2}}$ is the $\alpha / 2$ quantile of the standard normal distribution.



The condition that the true VI, ${\mathrm{VI}}_I \neq 0$, is essential for this result. Constructing a confidence interval that remains valid when ${\mathrm{VI}}I = 0$ is challenging. In such scenarios, standard Wald-type confidence intervals based on $\hat{\tau}{N, I}$ typically exhibit incorrect coverage or Type I error rates, as $\widehat{\mathrm{VI}}_I$ does not generally converge to a non-degenerate distribution \citep{willi}.  To deal with this case, we need to use sample splitting, i.e., use different data to train the full model and the reduced model before constructing the confidence interval, see  section  3.4 in \cite{willi} for detailed discussions.

Besides this,
the predictive skill measure in Corollary \ref{testcor} is not strictly limited to the negative MSE emphasized in this paper. According to the theoretical results in \cite{willi}, the same conclusion applies to various other predictive skill measures, including $R^2$, classification accuracy, and the area under the ROC curve (AUC). While accuracy and AUC are typically used for binary outcomes, a model can still be trained using the MSE loss in classification settings. The same convergence rates hold, with $\| f_{\widehat{T}{\text{op}}} - f_{0,I}\|_2$ achieving at least $O_p\left(N^{-1/4}\right)$,because we  do not assume any special structure for $Y$.


\section{Experiments}
\label{bigexp}
The theoretically optimal stopping time $\widehat{T}_{\text{op}}$ is typically impractical to compute due to the lack of knowledge of $f_{0,-I}$ and the noise level $\sigma$. In practice, we recommend using the hold-out method, where the training data is split into a validation set, and updates are halted when the validation loss shows no improvement over several iterations.
Specifically, we use data splitting for training and estimating VI (see Algorithm \ref{esalgo} for details). Further discussion is provided in the Appendix \ref{dicprac}.

We compare our method with two baseline approaches \textit{dropout} and \textit{retrain}. The dropout method estimates $f_{N,-I}$ by applying the full model $f_N^c$ to the dropped data $\boldsymbol{X}^{(I)}$, calculating variable importance as
\begin{equation*}
\widehat{\mathrm{VI}}_I^{\mathrm{DR}}=V\left(f_N^c, P_N\right)-V\left(f_N^c, P_{N,-I}\right).
\end{equation*}
\textit{Retrain}, on the other hand, estimates variable importance by training separate models for each subset $I$, with VI given by
\begin{equation*}
    \widehat{\mathrm{VI}}_I^{\mathrm{RT}}=V\left(f_N, P_N\right)-V\left(f_{N,-I}, P_{N,-I}\right)
\end{equation*}
where $f_{N,-I}$ is derived by training separate models using $\boldsymbol{X}^{(I)}$ from scratch.

In our experiments, we empirically verify the theoretical bounds and evaluate our algorithm's performance on simulated data and a real-world example to demonstrate its practical effectiveness in estimating variable importance.  The more intricate details of the experiments setup are provided in Appendix 
\ref{exp}.

For these experiments, unless otherwise specified we use the following model structures: for neural networks, we train a three-layer fully connected neural network with ReLU activation, where the hidden layer has a width of 2048; for GBDT, the random strength $\beta$ we use is 10,000 and tree depth is $2$. And instead of  computing $\mathbb{E}_u  f_{\widehat{T}}$, we simply use $f_{\widehat{T}}$ in our experiments. We use $3/4$ of all the data to train the model, 
while the remaining samples are employed to estimate the VI and construct Wald-type CI. 
Specifically, we set $q= 0.75$ in Algorithm \ref{esalgo}. The implementation of the proposed algorithm and all the experiments are available in \url{https://github.com/ZexuanSun/Early-stopping-VI}.
\begin{algorithm}
\caption{Early stopping training for $VI_{I}$}
\label{esalgo}
\begin{algorithmic}[1]
\State Input Data $\left\{(\boldsymbol{X}_i, Y_i)\right\}_{i=1}^N$, training size $ N_1 = q N$; 
$N_2  = (1-q) N$; Kernel based model: $f_{\theta}(\cdot)$;  Drop features set $I \subseteq \{ 1, \dots, p \}$; Patience $P$;

\State Train full model $f_{N_1}^c$ using training data $\left\{(\boldsymbol{X}_i, Y_i)\right\}_{i=1}^{N_1}$;
\State Replace feature $j \in I$ with its empirical mean to get $\boldsymbol{X}^{(I)}_i$;
\State Split $\left\{(\boldsymbol{X}^{(I)}_i, Y_i)\right\}_{i=1}^{N_1}$ into a training set $\mathcal{D}_1$ of sample size $q N_1$
and a validation set $\mathcal{D}_2$ of sample size $(1-q) N_1$;
\State Initialize  model with $ f_{N_1}^c$, for each epoch $\tau$, train on $\mathcal{D}_1$,  evaluate on 
$\mathcal{D}_2$;
\State If the loss evaluated on $\mathcal{D}_2$ 
at epoch $\widehat{T}$ has no improvement  after $P$ epochs, stop and return $f_{\widehat{T}}$ as the estimator for $f_{0,-I}$;
\State Use remaining $N_2$ instances to construct $\widehat{VI}_{I}$ and plug in estimate of $\tau_{N,I}$
\begin{equation*}
 \begin{aligned}
 \widehat{\text{VI}}_I&=\frac{1}{N_2} \sum_{i=1}^{N_2}\left[Y_i-
f_{\widehat{T}}(\mathbf{X}_i^{(I)})
\right]^2-\left[Y_i-f^c_{N_1}(\mathbf{X}_i)\right]^2\\
 t_{i, I} &= \left(Y_i-
 f_{\widehat{T}}
 \left(\mathbf{X}_i^{(I)}\right)\right)^2-\left(Y_i-f^c_{N_1}\left(\mathbf{X}_i\right)\right)^2 \\
\hat{\tau}_{N, I} &= \frac{1}{N_2} \sum_{i=1}^{N_2}\left(t_{i, I}-\bar{t}_I\right)^2  / N_2
\end{aligned}
\end{equation*}
\State Construct $\alpha$-level Wald-type CI as 
$ \widehat{\mathrm{VI}}_I^{} \pm z_{\frac{\alpha}{2}} \cdot \hat{\tau}_{N, I}$.
\end{algorithmic}
\end{algorithm}



\subsection{Verifying theoretical bounds}
To verify the theoretical bounds, we consider independent features, dropping only the first feature. Construct $f_0$ as $f(\boldsymbol{X}^{(1)}) + \beta_1 X_1$. For neural networks, $f$ is a two-layer fully connected neural network with ReLU activation, where the hidden layer has a width of 2048. For GBDT, we use a model with a random strength 
$\beta$
of 10,000 and a depth of 2. This design is to satisfy assumption \ref{inass} for both models. We use shallow neural network and GDBT here to reduce the cost for computing the empirical kernel matrix.
Note that our theory does not  assume feature independence, we use independent data for implementation convenience. Since computing the optimal stopping time $\widehat{T}_{\text{op}}$ is difficult, we use $\widehat{T}_{\text{max}}$, which also achieves convergence rate $\mathcal{O}(N^{-1/2})$. For the population bound experiment, we do not normalize data to be on the hyper-sphere and use networks including bias term. 
We repeat the experiment 10 times, and compared the empirical and population norm with the upper bound $\mathcal{O}(N^{-1/2})$. Results are shown in Figure \ref{loglog}. We can see that the loss curve exhibits some stochastic patterns however we can see that 
the overall convergence rate of 
the population bound for both neural networks and GBDT are closer to or even faster than $O(N^{-1})$.  
\begin{figure}[htp!]
     \centering
     \begin{subfigure}[b]{0.4\textwidth}
         \centering
\includegraphics[width=0.9\textwidth]{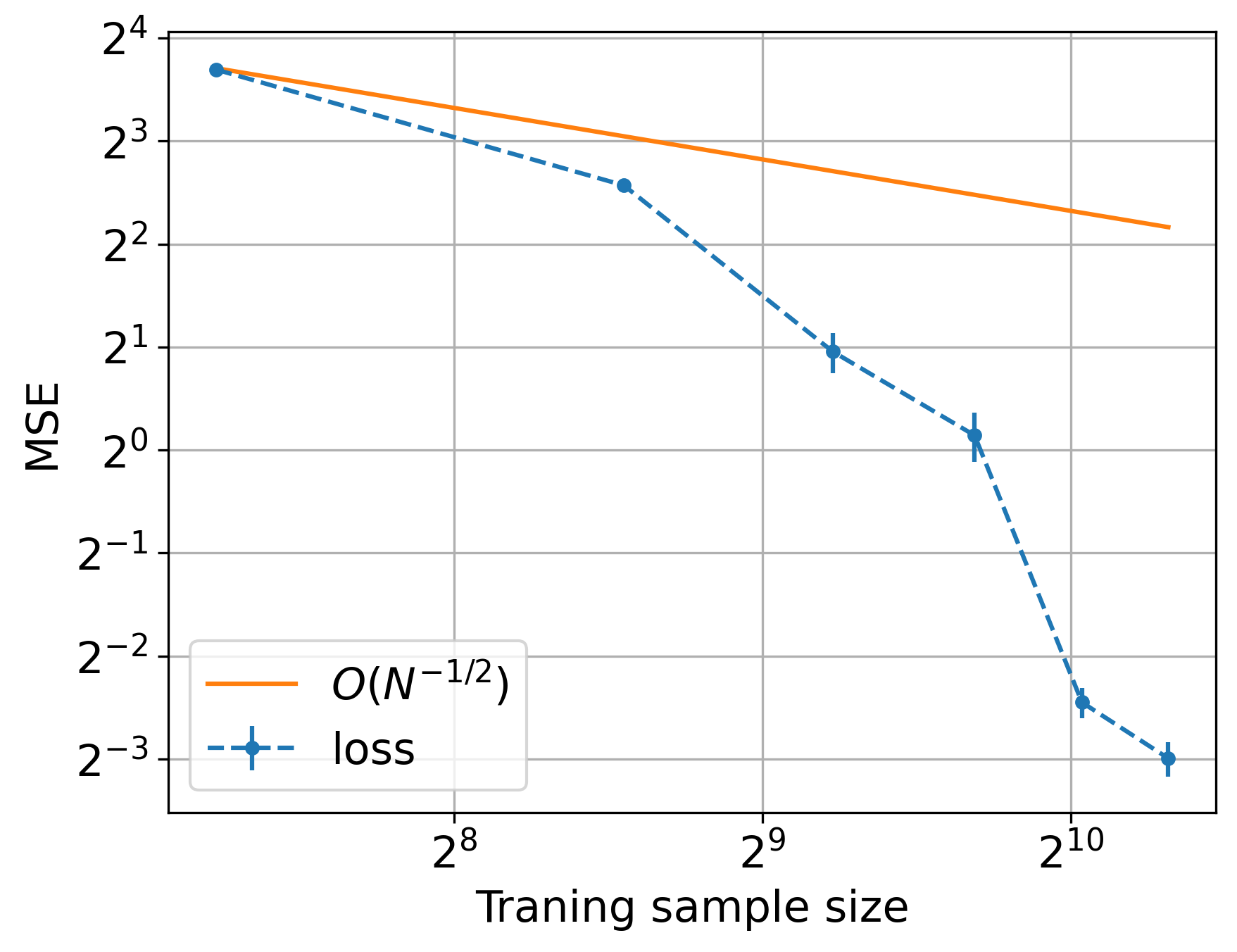}
         \caption{Neural Networks empirical bound}  
     \end{subfigure}
     \begin{subfigure}[b]{0.4\textwidth}
         \centering
         \includegraphics[width=0.9\textwidth]{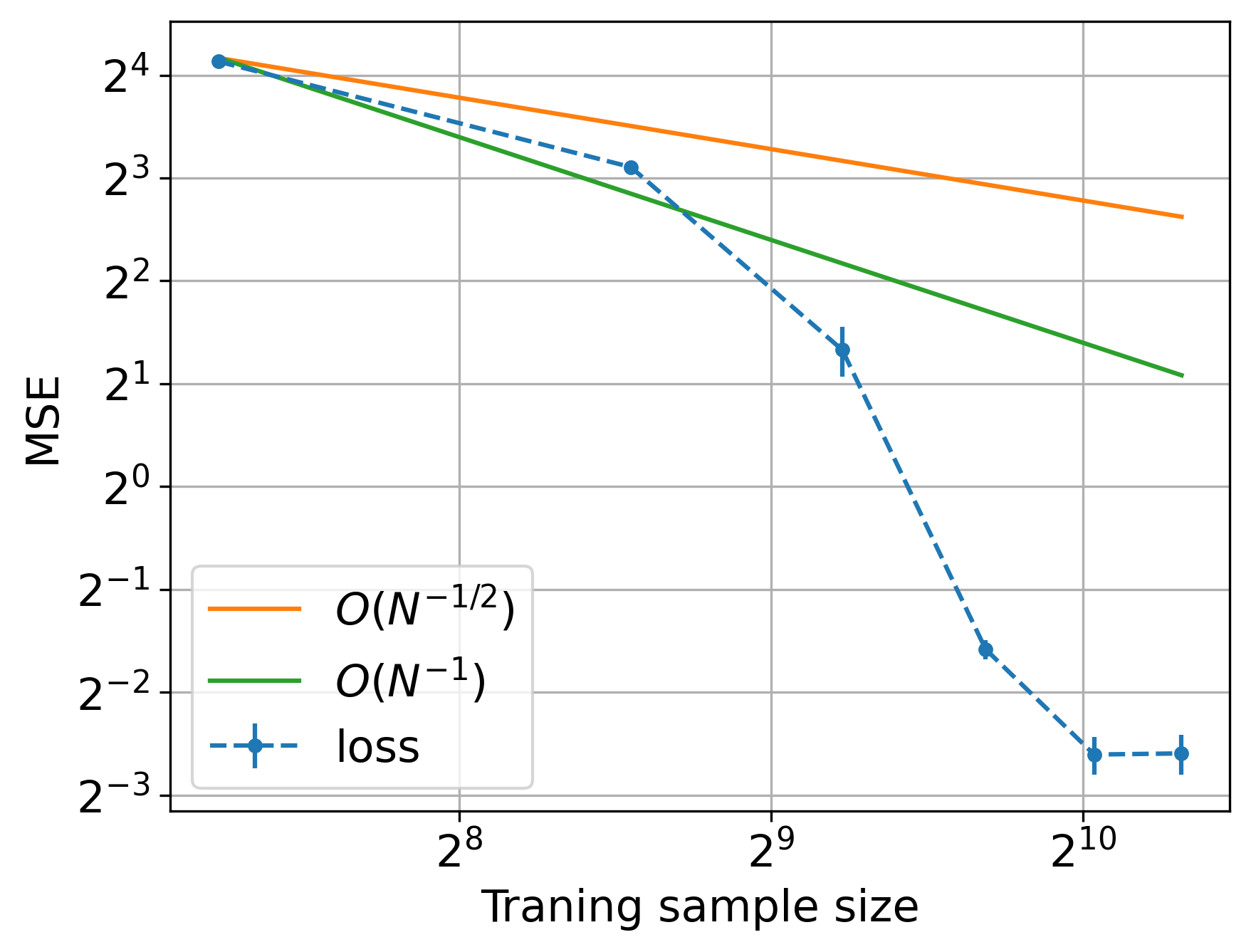}  
         \caption{Neural Networks population error bound}
     \end{subfigure}

\begin{subfigure}[b]{0.4\textwidth}
         \centering
         \includegraphics[width=0.9\textwidth]{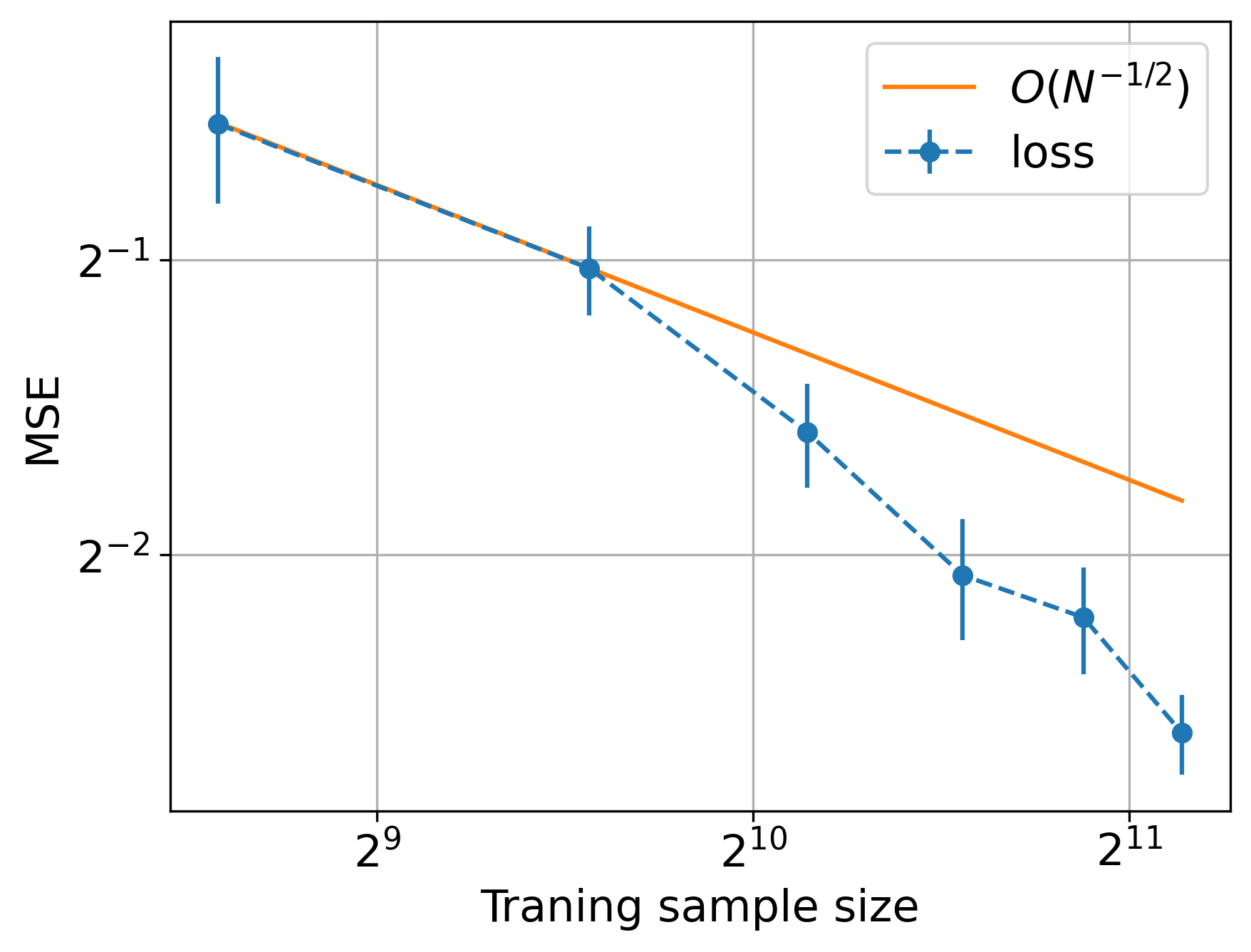}  
         \caption{GBDT empirical bound}
     \end{subfigure}
     \begin{subfigure}[b]{0.4\textwidth}
         \centering
         \includegraphics[width=0.9\textwidth]{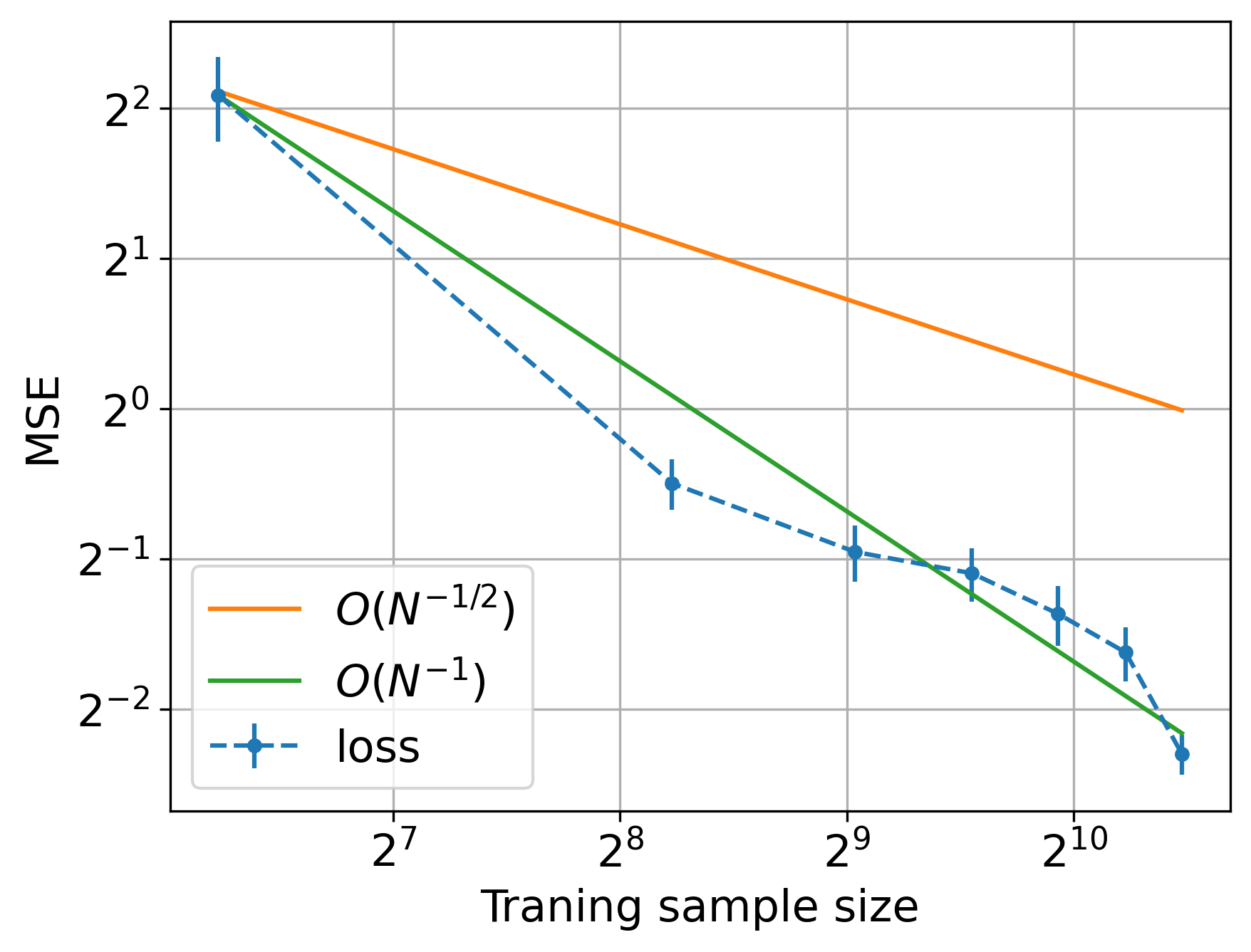}  
         \caption{GBDT population error bound}
     \end{subfigure}
\caption{Log-log plot comparing theoretical  bounds.}
\label{loglog}
\end{figure}

\subsection{Linear models with correlation}
This simulation is based on the following example from \cite{lazyvi}:
\begin{example}
\label{exglazy}
Suppose $Y=\beta_1 X_1+\beta_2 X_2+\epsilon$, where $X_i \sim \mathcal{N}\left(0, \sigma^2\right), i=1,2, \operatorname{Cov}\left(X_1, X_2\right)=\rho$, and $\epsilon$ is a $\mathcal{N}\left(0, \sigma_\epsilon^2\right)$ noise that is independent of the features. The variable importance for the first variable is
\begin{equation}
    \mathrm{VI}_1=\beta_1^2 \cdot \operatorname{Var}\left(X_1 \mid X_2\right)=\beta_1^2\left(1-\rho^2\right) \sigma^2.
\end{equation}
\end{example}
Let $\beta = (1.5, 1.2, 1, 0, 0, 0)^T$, we generate $Y_i = X_i^T \beta + \epsilon$. We drop only the first feature and estimate its VI using our method and dropout. The ground truth VI can be calculated according to the above example. We generate a simulated dataset of sample size 5000 and repeat the experiments 10 times for each $\rho$.
The results for neural networks and GBDT are shown in Figure \ref{corvi}.  
We observe that for both algorithms, as the correlation $\rho$ increases, the estimation of dropout becomes increasingly unreliable and tends to overestimate the variable importance (VI).
However, our early stopping approach maintains its accuracy despite the increase in correlation.


\begin{figure}[htp!]
     \centering
     \begin{subfigure}[b]{0.45\textwidth}
         \centering
\includegraphics[width=1\textwidth]{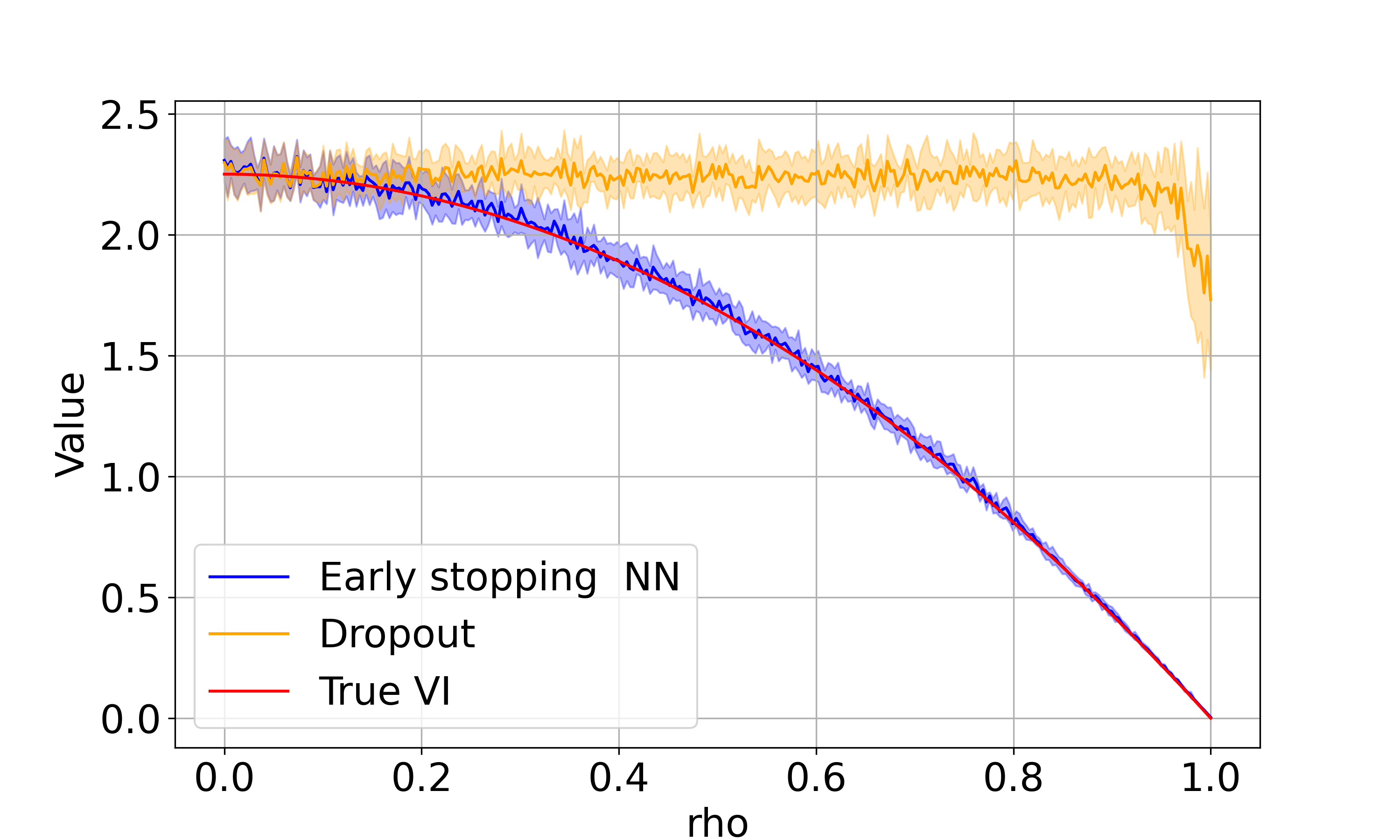}
         \caption{Neural Networks}  
     \end{subfigure}
     \begin{subfigure}[b]{0.45\textwidth}
         \centering
         \includegraphics[width=1\textwidth]{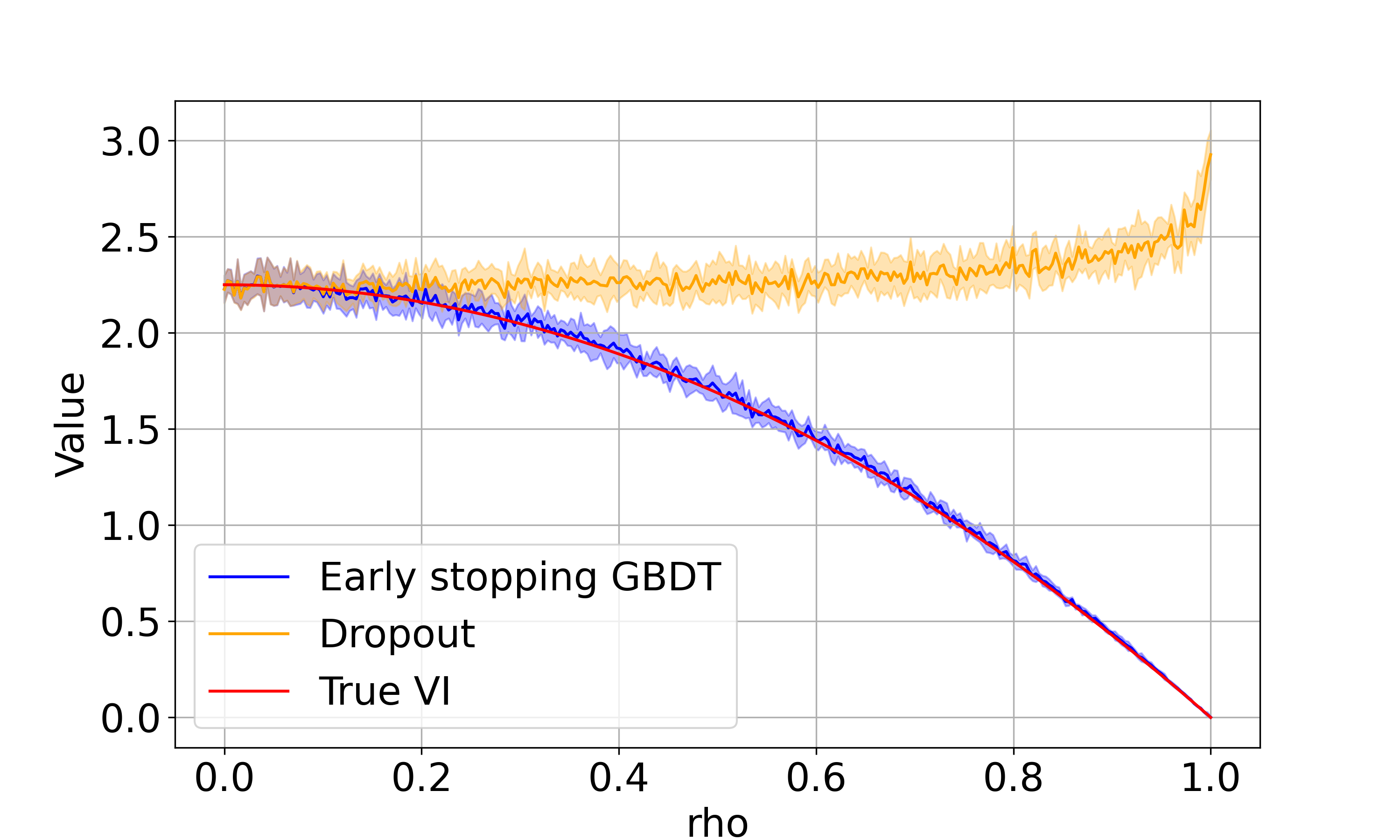}  
         \caption{GBDT}
     \end{subfigure}
\caption{VI estimation comparison for correlated linear model.}
\label{corvi}
\end{figure}

\subsection{High-dimensional regression }


The computational burden of \textit{retrain} is most pronounced in high-dimensional settings, since at least $p$ models are needed to estimate VI \citep{lazyvi}. For this simulation, we generate various high-dimensional datasets for neural networks and GDBT to evaluate their performance in high-dimensional settings. The different simulated datasets are designed to better align the experiments with assumption \ref{inass}.

We generate
$X \sim N\left(0, \Sigma_{100 \times 100}\right)$, where variables are independent except $\operatorname{Corr}\left(X_1, X_2\right)=0.5$.
For neural networks, we let $\boldsymbol{\beta}=(5,4,3,2,1,0, \ldots, 0)^{\top} \in \mathbb{R}^{100}$, we construct a weight matrix $W \in \mathbb{R}^{m \times p}$ such that $W_{:, j} \sim \mathcal{N}\left(\beta_j, \sigma^2\right)$. Let $V \sim \mathcal{N}(0,1)$, we generate the response $Y_i=V \sigma\left(W \mathbf{X}_i\right)+\epsilon_i$ where $\sigma$ is the ReLU function.   For GDBT, we generate the response $Y_i =\mathbf{X}_i \boldsymbol{\beta} + \epsilon_i$.
We take retrain estimation results as ground truth. 
We assessed VI for $X_1$  on 10 simulated datasets with sample size 5000, and our method outperforms \textit{dropout} and is faster than \textit{retrain}, as shown in Figure \ref{hdsim}.

\begin{figure}[htp!]
     \centering
     \begin{subfigure}[b]{0.4\textwidth}
         \centering
         \includegraphics[width=1\textwidth]{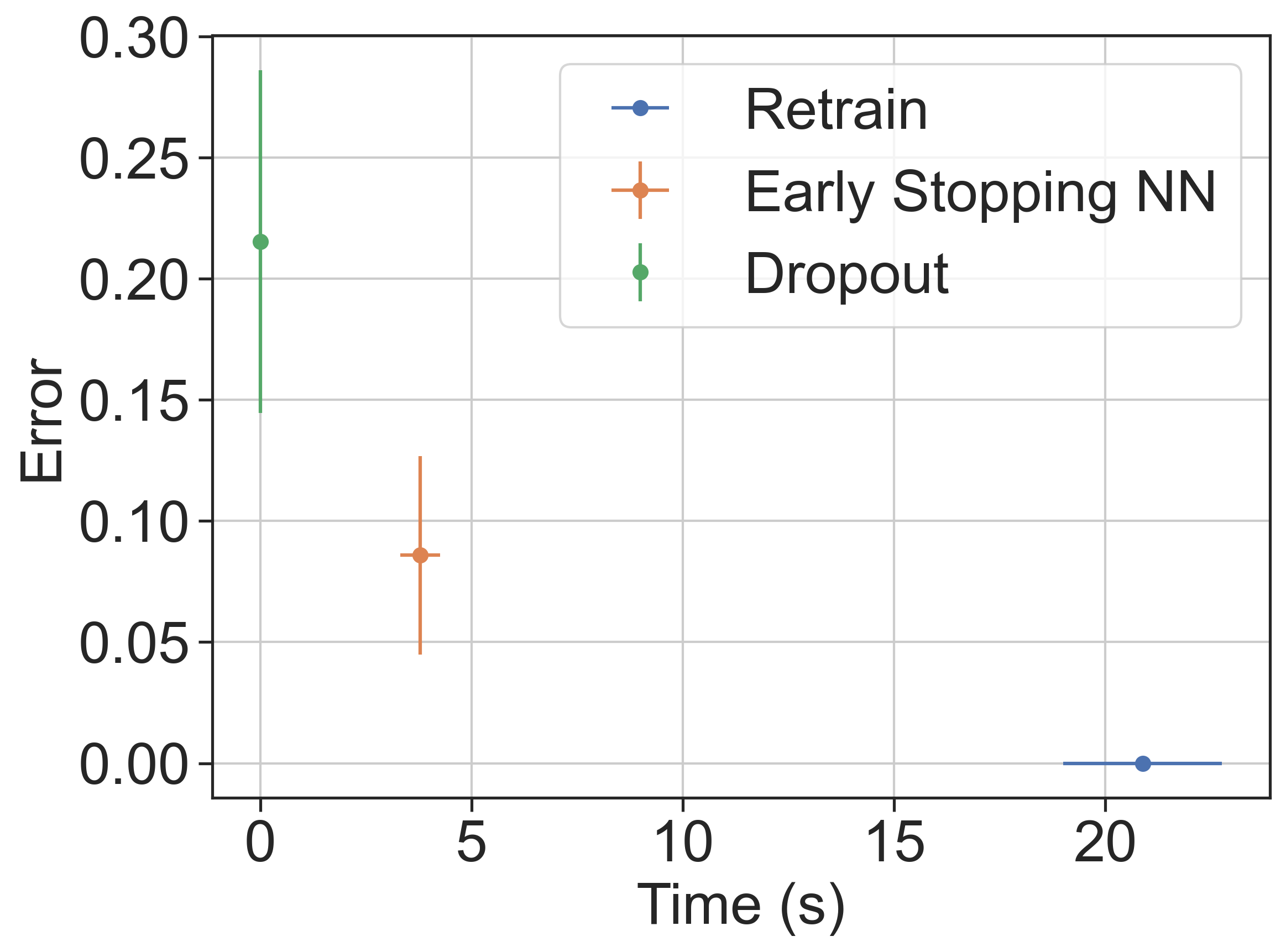}
         \caption{Neural Networks}  
     \end{subfigure}
     \begin{subfigure}[b]{0.4\textwidth}
         \centering
         \includegraphics[width=1.235\textwidth]{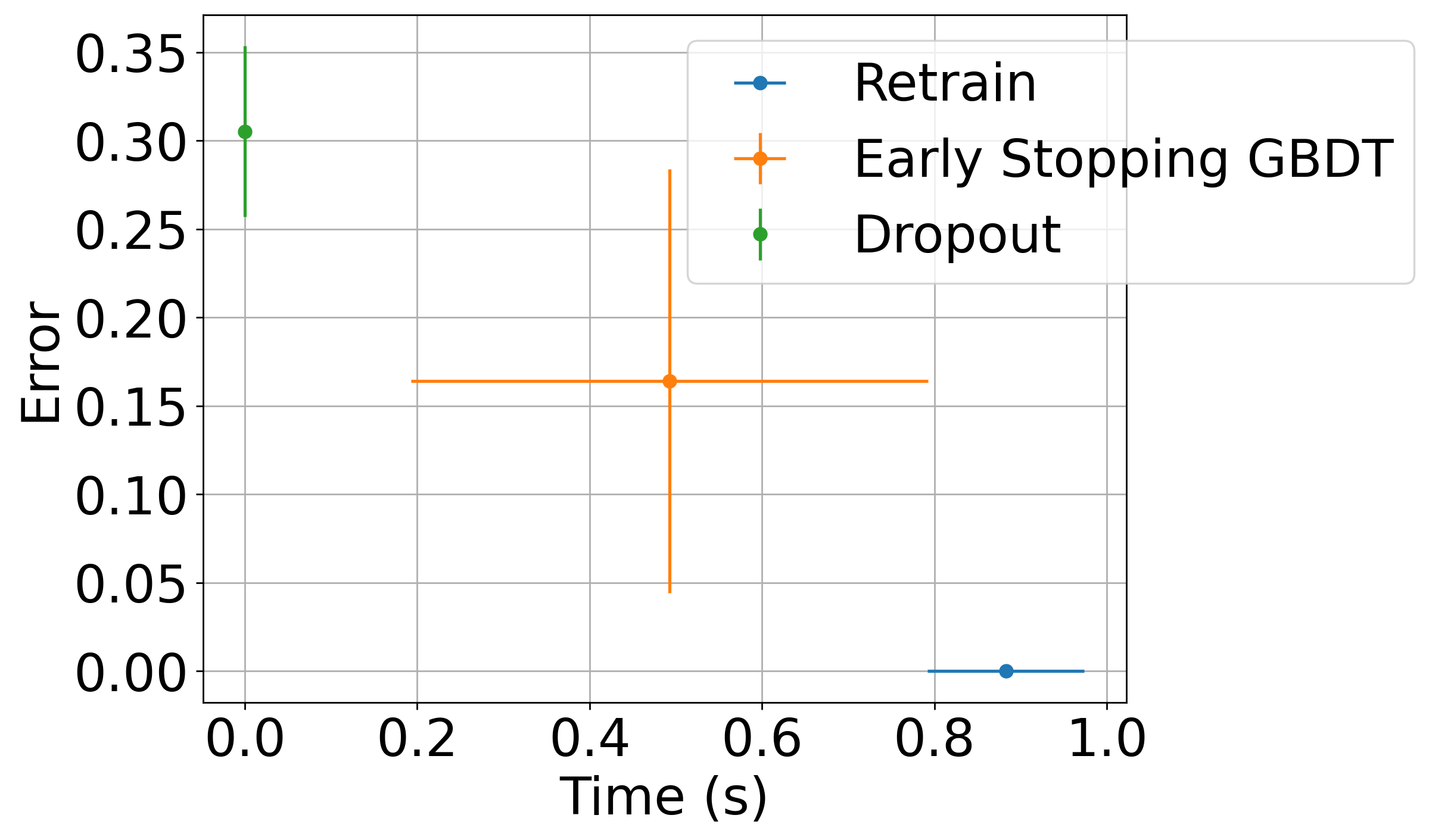}  
         \caption{GBDT}
     \end{subfigure}
\caption{Distribution of computation time vs. normalized estimation error relative to retrain for the VI of $X_1$.}
\label{hdsim}
\end{figure}

\subsection{Comparison with LazyVI}
We compare our method with the LazyVI method \citep{lazyvi} using the same high-dimensional regression simulation described in Section 5.3 of the main text. The results are shown in Figure \ref{compall}. While the LazyVI method for neural networks achieves similar precision to our approach, it requires significantly more computation time. The primary issue with LazyVI is that it relies on k-fold cross-validation to select the ridge penalty parameter $\lambda$, which becomes highly time-consuming when the sample size is large. In our case, with $N = 5000$, the process takes an exceptionally long time to run (note that in \citep{lazyvi} a much smaller sample size was used).
\begin{figure}[htp!]
     \centering
 \begin{subfigure}[b]{0.45\textwidth}
         \centering
         \includegraphics[width=1\textwidth]{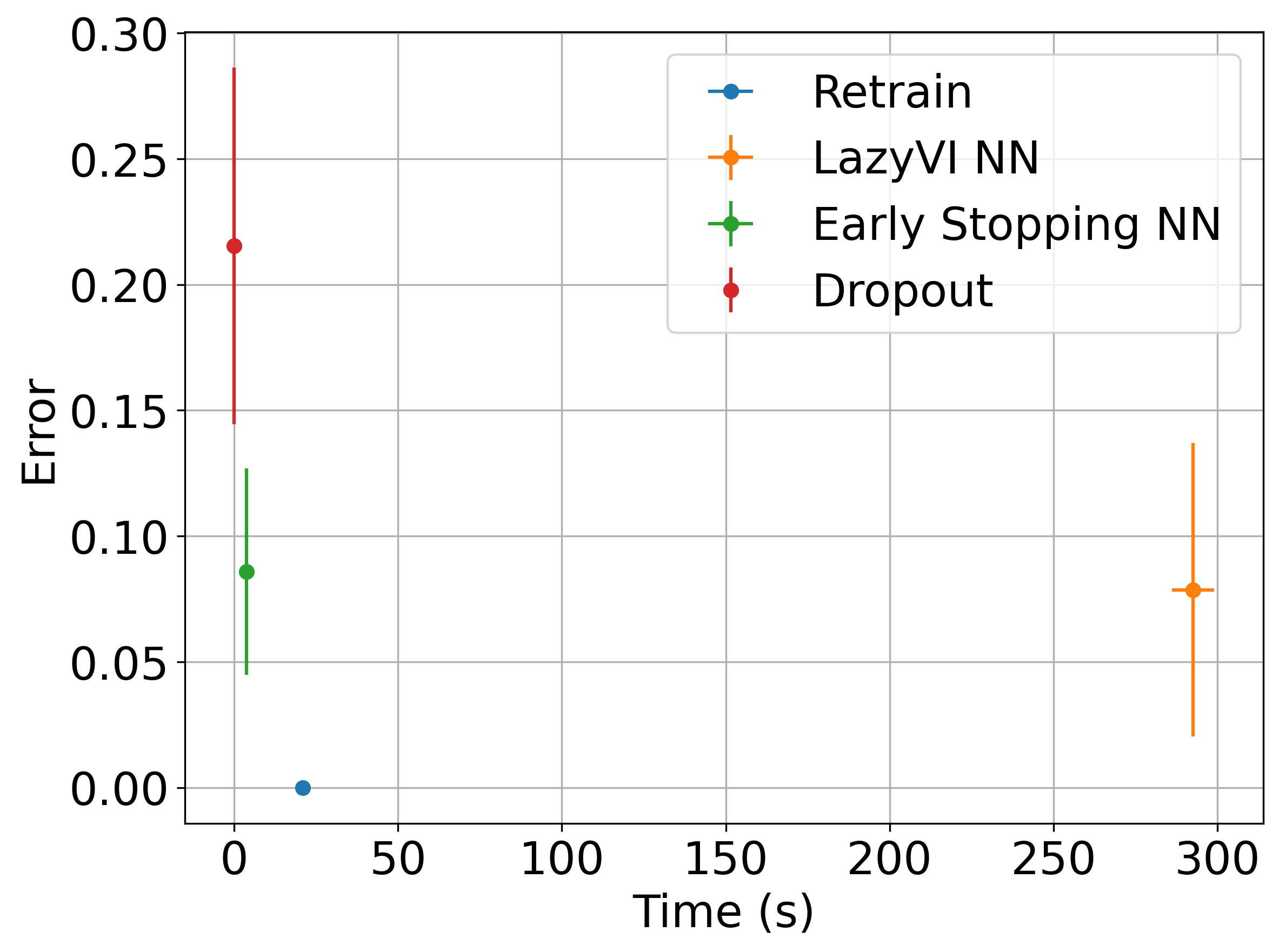}
     \end{subfigure}
\caption{Distribution of computation time vs. normalized estimation error relative to retrain for the VI of $X_1$.}
\label{compall}
\end{figure}

\subsection{Shapley values estimation}
As mentioned in subsection \ref{shapintro}, our method can be used to reduce the high computational cost of Shapley values. 
In our paper, we consider the 
\textit{val}  function in (\ref{shapdef}) to be the negative MSE.  
Consider a logistic regression model, we generate   $X \sim \mathcal{N}\left(0, \Sigma_{10 \times 10}\right)$, where the variables are independent except $\operatorname{Corr}\left(X_1, X_2\right)=0.5$. The responses are  generated from a logistic model: $\log \frac{\mathbb{P}(Y=1)}{1-\mathbb{P}(Y=1)}=X \boldsymbol{\beta}$, where $\boldsymbol{\beta}=10 \times (0, 1, 2,\dots, 9)^{\top} \in \mathbb{R}^{10}$. The dataset we generate is of size $N=5000$.
We use the subset sampling scheme proposed by \citep{sample} when calculating Shapley values. We generate 50 samples to estimate each \(\phi_j(\text{val})\), repeating this process 10 times.
We can see that our proposed approach is closer to the \textit{retrain} compared with \textit{dropout} method, as shown in Figure \ref{shappic}. The run-time for Shapley values estimates is 2-3 times faster using our early stopping approach compared to re-training.

\begin{figure}[htp!]
     \centering
     \begin{subfigure}[b]{0.4\textwidth}
         \centering
         \includegraphics[width=1.\textwidth]{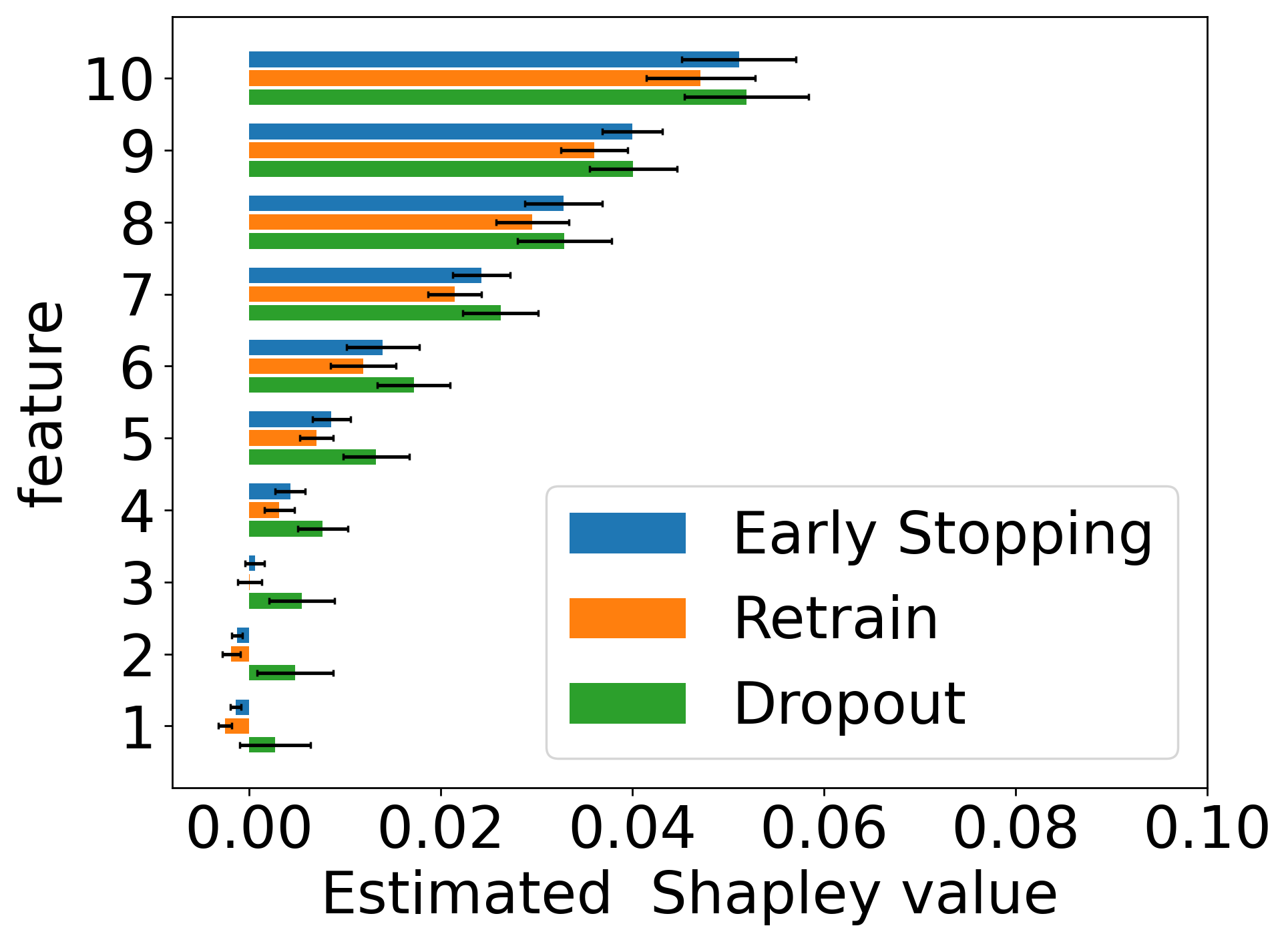}
         \caption{Neural Networks}  
     \end{subfigure}
     \begin{subfigure}[b]{0.4\textwidth}
         \centering
         \includegraphics[width=1.\textwidth]{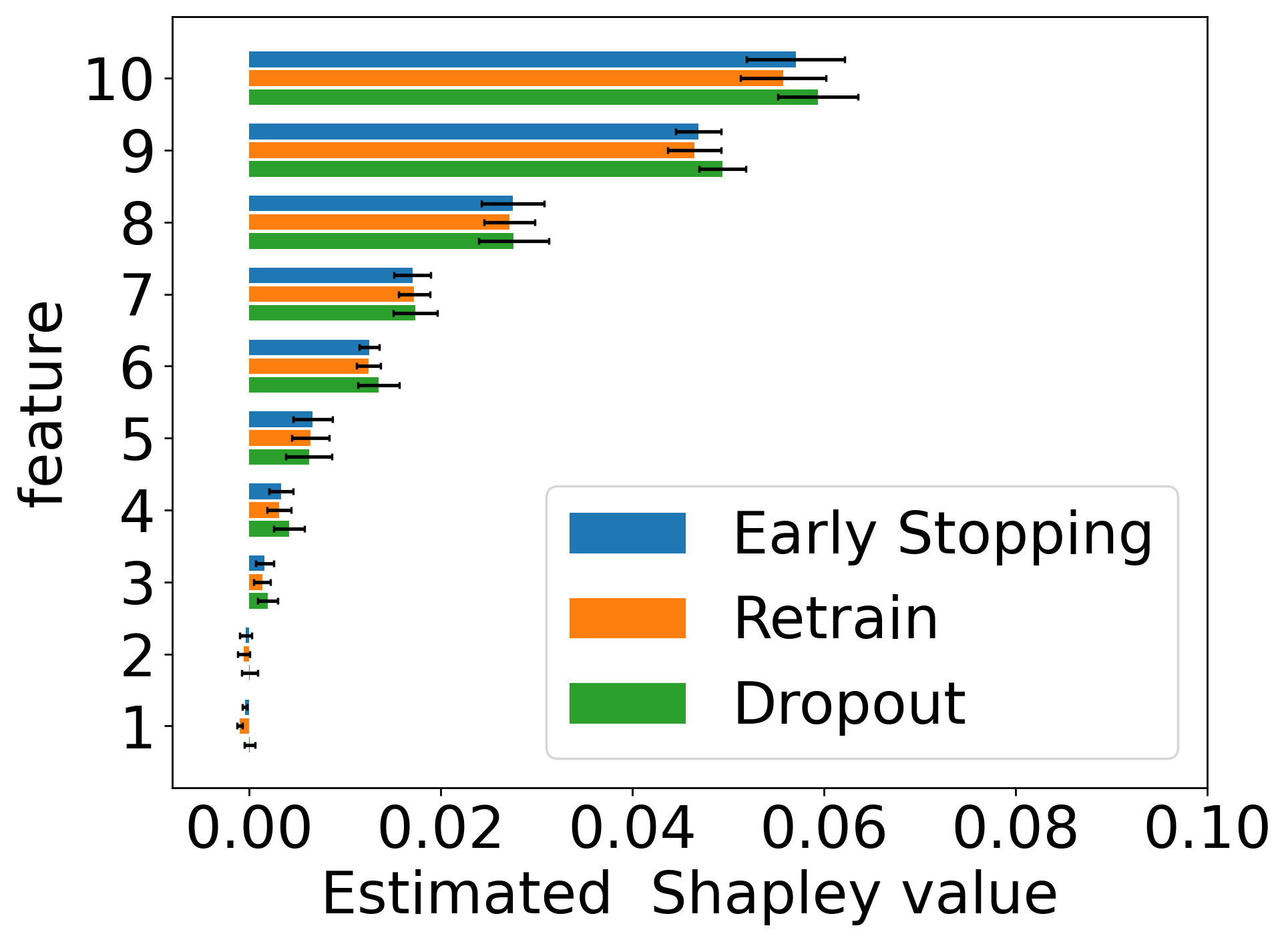}  
         \caption{GBDT}
     \end{subfigure}
\caption{Shapley value estimation for logisitic model.}
\label{shappic}
\end{figure}
\subsection{Wald type confidence interval}
We use example \ref{exglazy} again here to test the Wald-type confidence interval result given in section \ref{waldnn}.  We drop the first variable $X_1$, and the true VI is given by $1.5^2 \cdot (1-\rho^2)$. We test on $\rho = (0, 0.2,0.5,0.8,1.0)$. Specifically, we compute the empirical $95\%$ Wald-type confidence interval across 100 simulated datasets of sample size $N = 5000$ and calculate the coverage probability of the CI. The results are displayed in Figure \ref{waldci}. We can see that when $\rho \neq 1$, i.e., when the true VI is not $0$, the coverage probabilities of neural network are quite close to the desired probability $95\%$. And the coverage drops significantly when the true VI is $0$, which supports our claim that the Wald-type CI result only hold for the case when true VI is not $0$. The results with GBDT are also reasonable for small $\rho$, but when $\rho$ is large, the coverage probabilities drop. 

\begin{figure}[htp!]
     \centering
     \begin{subfigure}[b]{0.4\textwidth}
         \centering
         \includegraphics[width=1.\textwidth]{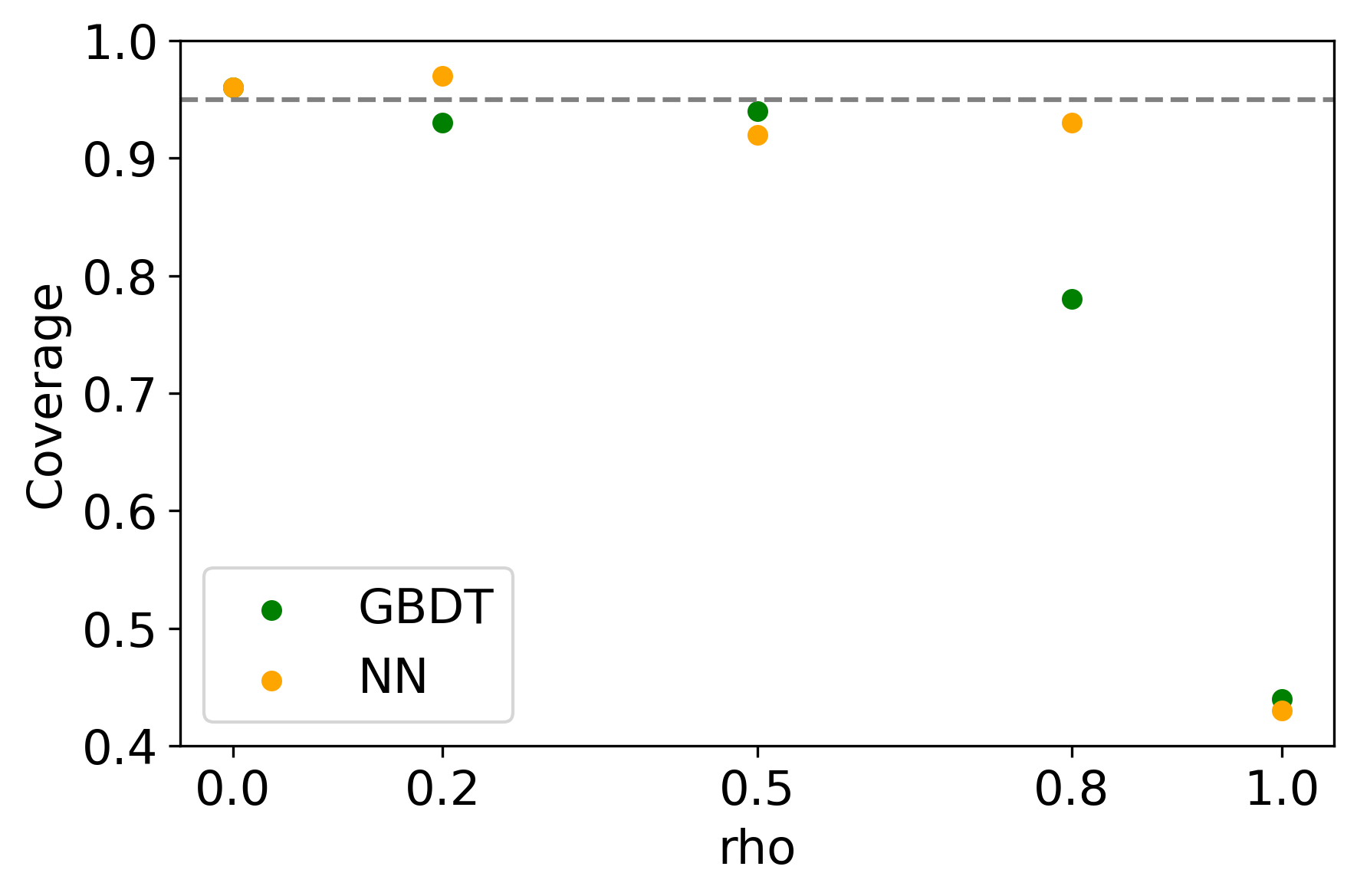}
     \end{subfigure}
 \caption{Wald type CI coverage experiment.}  
\label{waldci}
\end{figure}

\subsection{Predicting flue gas emissions}
Assessing variable importance in predicting CO and NOx emissions from gas turbines is crucial for optimizing efficiency, reducing environmental impact, and ensuring compliance. It helps operators adjust turbine settings and prioritize maintenance on key variables.
To demonstrate the value of our framework, we used data from a gas turbine in northwestern Turkey, collected in 2015, to study NOx emissions \citep{gasdat}. The dataset includes 7,384 instances of 9 sensor measures, such as turbine inlet temperature and compressor discharge pressure, aggregated hourly. Initial variable importance  using negative MSE estimates were low, as shown in 
Figure \ref{gaspic}(a)
due to high feature correlation, prompting us to apply Shapley value estimation using our proposed methods.
According to the Shapley value estimates using neural networks in Figure \ref{gaspic}(b), Ambient Humidity (AH) and Ambient Pressure (AP) are not significant features, while Turbine Inlet Temperature (TIT) is the most important. Other features have similar importance and contribute to predicting NOx emissions. These findings align with previous studies \citep{supp1, supp2}, which highlight the significance of temperature-related variables, such as TIT, in predicting NOx emissions, with ambient conditions like humidity and pressure having a lesser impact. Once again we observe how our early stopping approach closely mimics the estimates for re-training while dropout tends to often over-estimate both VI and Shapley values. The feature correlation heat map and estimation results using GBDT are available in Appendix \ref{exp}.





\begin{figure}[htp!]
     \centering
    \begin{subfigure}[b]{0.4\textwidth}
         \centering
         \includegraphics[width=1\textwidth]{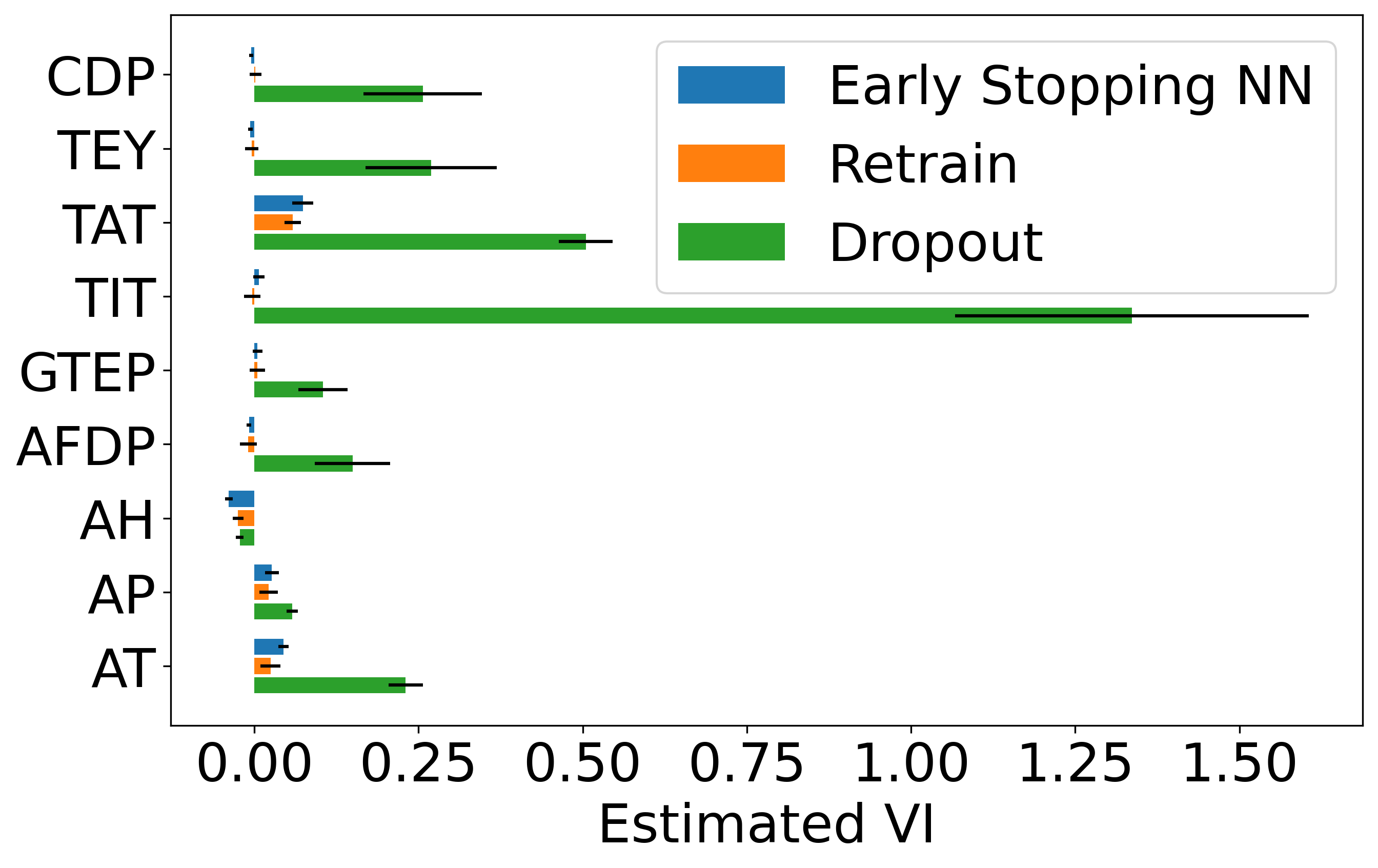}
         \caption{VI estimation}  
     \end{subfigure}
     \begin{subfigure}[b]{0.4\textwidth}
         \centering
         \includegraphics[width=1\textwidth]{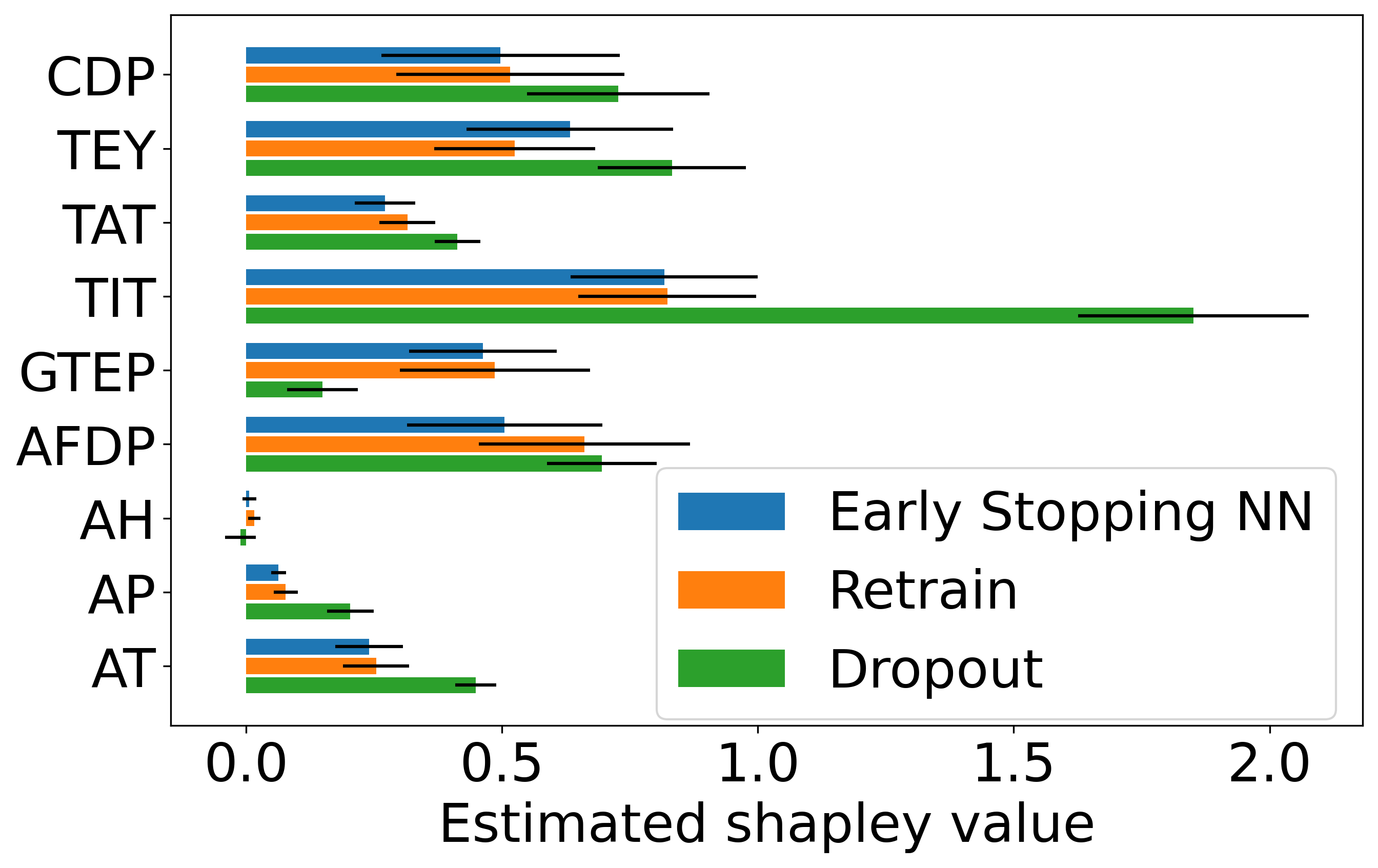}  
         \caption{Shapley value estimation}
     \end{subfigure}
\caption{Flue gas emissions application.}
\label{gaspic}
\end{figure}

\section{Proofs}
We present the proofs of our primary results in this section. The key steps for each proof are outlined in the main text, while the more technical details are included in the appendix.
\subsection{General algorithm}
\subsubsection{Proof of Theorem \ref{genthmfix}}
We first show that  with probability at least $1-c_1 \exp \left(-c_2 N \widehat{\varrho}_N^2\right)$
 for $\tau = 1,2, \dots, \widehat{T}_{max}$, the following holds:
\begin{equation}
   \|f_{\tau} - f_{0,-I}\|_N^2     \leq \frac{C}{\epsilon \tau} + g(\boldsymbol{X}^{(I)},\tau)
       \label{gensumb2}
\end{equation}
where $g(\boldsymbol{X}^{(I)}, \tau)$ is a non-decreasing function that depends on $\tau$, $c_1, c_2$ and $C$ are some  universal positive constants. 
We prove this claim by bounding each term on the RHS of ~\eqref{lemgenb} separately. 

By Lemma 7 in \cite{esnonpara},
for $\tau = 1, \dots, \widehat{T}$, 
the squared bias is upper bounded by
\begin{equation}
    B_{\tau}^2 \leq \frac{C_1}{e \eta_\tau}.
\end{equation}
Moreover, because of Assumption~\ref{rmdassum},
according to the proof of Lemma 7 in \cite{esnonpara}, with probability at least $1-c_1 \exp \left(-c_2 N \widehat{\varrho}_N^2\right)$, 
the variance term is upper bounded as
\begin{equation}
V_{\tau} \leq \frac{C_2}{e \eta_\tau}.
\end{equation}

Then it suffices to show that the difference term is upper bounded by a non-decreasing $g(\tau)$.  Recall that 
\begin{equation}
    S^{\tau} := (I - \epsilon\Lambda)^{\tau}
\end{equation}
so we have for the difference term $D_{\tau}$
\begin{equation}
    \begin{aligned}
    \frac{2\epsilon}{\sqrt{N}}\|\sum_{i=0}^{\tau-1} S^{\tau - 1 -i}\tilde{\delta}_{i}\|_2 \leq \frac{2\epsilon}{\sqrt{N}} 
    \sum_{i=0}^{\tau-1} \| S^{\tau - 1 -i} \|_2  \|\tilde{\delta}_{i}\|     \leq \frac{2\epsilon}{\sqrt{N}} \sum_{i=0}^{\tau-1} \|\tilde{\delta}_{i}\|_2 
     = \frac{2\epsilon}{\sqrt{N}}\sum_{i=0}^{\tau-1} \|{\delta}_{i}\|_2 := \left(g(\boldsymbol{X}^{(I)}, \tau)\right)^{1/2}
\end{aligned} 
\label{bddt}
\end{equation}
where we use the condition that $\epsilon \leq \min \left\{1,1 / \hat{\lambda}_1\right\}$.  Observe that $g(\boldsymbol{X}^{(I)}, \tau)$ is a non-decreasing function.

We next show that for $\tau = \widehat{T}_{\max}$, $f_{\widehat{T}_{\max}}$ satisfies the desired bound. Then since $\widehat{T}_{\text{op}}$ is the optimal stopping time within $[0, \widehat{T}_{\max}]$, it should at least have the same convergence bound. 

The model-specific conditions 
control the term $g(\boldsymbol{X}^{(I)}, \tau)$. Specifically, we require that $g(\boldsymbol{X}^{(I)}, \tau) \in \mathcal{O}(\frac{1}{\sqrt{N}})$. We will see how to control this in the examples of neural network and GBDT in later proofs.

Then we only need to show that 
\begin{equation}
      \frac{C}{ \eta_{ \widehat{T}_{\max}}} \leq \mathcal{O}\left(\frac{1}{\sqrt{N}}\right). 
\end{equation}
According to the proof of Theorem 1 in \cite{esnonpara}
\begin{equation}
    \frac{C}{ \eta_{ \widehat{T}_{\max}}} \leq C^{\prime} \widehat{\varrho}^2_N.  
    \label{eqsum}
\end{equation}
And by the properties of the empirical Rademacher complexity (Appendix D in \cite{esnonpara}), the critical radius satisfies
\begin{equation}\widehat{\mathcal{R}}_K^{(I)}\left(\widehat{\varrho}_N\right)=
\frac{\widehat{\varrho}_N^2 C^2_{\mathcal{H}}}{2 e \sigma}.
\end{equation}
Because we have $\text{tr}(K^{(I)}) = O(1)$ by Assumption~\ref{trass}, then 
\begin{equation}
\frac{\widehat{\varrho}_N^2 C^2_{\mathcal{H}}}{2 e \sigma} = 
\widehat{\mathcal{R}}_K(\widehat{\varrho}_N)=\left[\frac{1}{N} \sum_{i=1}^N \min \left\{\widehat{\lambda}_i, \widehat{\varrho}_N\right\}\right]^{1 / 2} \leq \left[\frac{\text{tr}(K^{(I)})}{N}    \right]^{1 / 2}  = \mathcal{O}\left(\frac{1}{\sqrt{N}}\right). 
\end{equation}
This shows that 
\begin{equation}
    \frac{\widehat{\varrho}_N^2 C^2_{\mathcal{H}}}{2 e \sigma} \leq \mathcal{O}\left(\frac{1}{\sqrt{N}}\right). 
\end{equation}
Combining with ~\eqref{eqsum} completes the proof.

\subsubsection{Proof for Theorem~\ref{genthmpop}}
To prove Theorem \ref{genthmpop}, the key is to write $f_{\tau}$ as $f_{\tau}^{\delta} + f_{\tau} - f_{\tau}^{\delta}$ where $f_{\tau}^{\delta}$ is the part accounting for the evolving kernel and $f_{\tau} - f_{\tau}^{\delta}$ can be interpreted as 
the pure updating function with constant kernel. Then we can use the same proof techniques as Theorem 2  in \cite{esnonpara}  to bound $f_{\tau} - f_{\tau}^{\delta}$ in $L_2$ norm. It remains to bound the $L_2$ norm of $f_{\tau}^{\delta}$. We have the following lemma to decompose $f_\tau$, which may be viewed as a population version  of Lemma~\ref{lemgenb}. 
\begin{lemma}
\label{indupopf}
Let $f_{\tau}^{\delta}$ represent the difference function accounting for the changing kernel,
    we can write $f_\tau$ as 
    \begin{equation}
        f_\tau(\cdot) =  f_{\tau}^{\delta}(\cdot) +  f_\tau(\cdot)-f_{\tau}^{\delta}(\cdot).
    \end{equation}
Define $\delta_{\tau}(\cdot)$ as 
\begin{equation}
  \delta_{\tau}(\cdot) :=\frac{1}{N} \left(\mathbb{K}^{(I)}(\cdot, \boldsymbol{X}^{(I)})   - \mathbb{K}^{(I)}(\cdot, \boldsymbol{X}^{(I)})  \right) [\boldsymbol{Y} - f_{\tau}(\boldsymbol{X}^{(I)})]
\end{equation}
where $\mathbb{K}^{(I)}(\cdot, \boldsymbol{X}^{(I)})$ and $\mathbb{K}_{\tau}^{(I)}(\cdot, \boldsymbol{X}^{(I)})$ represents a function in $\mathbb{R}^N$, with each entry being $\mathbb{K}^{(I)}(\cdot, \mathbf{X}_i^{(I)})$ and $\mathbb{K}_{\tau}^{(I)}(\cdot, \mathbf{X}_i^{(I)})$, respectively. 

Then $f_{\tau}^{\delta}(\cdot)$
 has the following recursion formula:
 \begin{equation}
     \begin{aligned}
         f_1^{\delta}(\cdot) &= \epsilon \delta_0(\cdot) \\
        f_{\tau + 1}^{\delta}(\cdot) &= f_{\tau }^{\delta}(\cdot) + 
       \epsilon  \delta_{\tau}(\cdot) - \frac{\epsilon}{N} \mathbb{K}^{(I)}_\tau (\cdot, \boldsymbol{X}^{(I)}) f_{\tau }^{\delta}(\boldsymbol{X}^{(I)}).
     \end{aligned}
 \end{equation}
\end{lemma}
 Suppose we are able to bound the $L_2$ norm of $\delta_\tau(\cdot)$, then use poof by induction, it is not difficult to show that the 
$f_{\widehat{T}_{\text{op}}}^{\delta}$ can be bounded in $L_2$ norm. Here we assume that $\|f_{\widehat{T}_{\text{op}}}^{\delta}(\cdot)\|_2^2$ can be upper bounded by $O(1 / \sqrt{N})$.  The empirical norm of $f_{\tau}^{\delta}$ is actually upper bounded by the difference term $D^2_\tau$ in Lemma \ref{applem1}, which can also be bounded by $O(1 / \sqrt{N})$. Then it suffices to bound $ \left\|f_{\widehat{T}_{\text{op}}} - f_{\widehat{T}_{\text{op}}}^{\delta}
    -f_{0,-I}\right\|_2^2  $. 

According to the proof of Theorem 2 in \cite{esnonpara}, we have 
\begin{equation}
\begin{aligned}
    \left\|f_{\widehat{T}_{\text{op}}} - f_{\widehat{T}_{\text{op}}}^{\delta}
    -f_{0,-I}\right\|_2^2  
    &\stackrel{(a)}{\leq}  \left\|f_{\widehat{T}_{\text{op}}} - f_{\widehat{T}_{\text{op}}}^{\delta}
    -f_{0,-I}\right\|_N^2 + c_1 \varrho_n^2  \\
    & \stackrel{(b)}{\leq} 2 \left\|f_{\widehat{T}_{\text{op}}}
    -f_{0,-I}\right\|_N^2 + 2 \left\| 
    f_{\widehat{T}_{\text{op}}}^{\delta}
    \right\|_N^2 + c_1 \varrho_n^2 \\
  & \stackrel{(c)}{\leq}  c_4 \widehat{\varrho}_N^2 + 2D_\tau^2  + 2 \left\| 
    f_{\widehat{T}_{\text{op}}}^{\delta}
    \right\|_N^2 + c_1 \varrho_N^2  \\ 
 & \stackrel{(d)}{\leq} 4D_{\widehat{T}_{\text{op}}}^2   + c \varrho_N^2 
    \end{aligned}
\end{equation}
with probability at least $1-c_2 \exp \left(-c_3 n \varrho_n^2\right)$. In equality (a), we apply Lemma 9 and 10 in \cite{esnonpara} as $\widehat{T}_{\text{op}} \leq \widehat{T}_{\max}$ by definition; In equality (c), we use Lemma \ref{applem1} and the same claim as in the proof as Theorem \ref{genthmfix}, i.e. eq. ( \ref{eqsum}); In equality (d), we apply lemma 11 in \cite{esnonpara} and the fact that $\left\| 
    f_{\widehat{T}_{\text{op}}}^{\delta}
    \right\|_N^2 \leq D_{\widehat{T}_{\text{op}}}^2$.

It remains to bound $\varrho_n^2$. By Assumption \ref{trasspop}, apply the same strategy as in the  proof of  Theorem \ref{genthmfix}, we are able to show that  
\begin{equation}
    \varrho_n^2 \leq O\left(\frac{1}{\sqrt{N}}\right).
\end{equation}

\subsection{Neural networks}
\subsubsection{ Proof of Corollary \ref{cor1}}
 First, we  need to show that $\text{tr}(K^{(I)}) = O(1)$ in the case of neural networks. For NTK parameterization, by definition
    \begin{equation}
        \text{tr}(K^{(I)}) =\lim_{m\rightarrow \infty}
        \frac{1}{N} \sum_{i=1}^N \left\langle
     \nabla_{\theta} f\left(\theta(0), \mathbf{X}_i^{(I)}\right),
     \nabla_{\theta} f\left(\theta(0), \mathbf{X}_i^{(I)}\right)\right\rangle.
    \end{equation}
Observe that in our setting, the NTK expression is the same as the one stated in Theorem 1 in \cite{jacot}. So each entry of $K^{(I)}$ is $O(1)$ because of Assumption~\ref{xb}. This implies that $\text{tr}(K^{(I)})$ is also $O(1)$.  We can also interpret this as a result of CLT.
Note that for standard parameterization, there would be a scaling factor $\frac{1}{m}$, and by ~\cite{linearnn}, we verify the same result.   

Next, we need to bound the difference term $D^2_{\tau}$. By ~\eqref{bddt} in the proof of Theorem~\ref{genthmfix},
\begin{equation}
\label{borb}
  D^2_{\tau}  \leq
    \frac{4\eta_0^2}{N} \left(\sum_{i=0}^{\tau-1} \|{\delta}_{i}\|_2\right)^2.
\end{equation}
Apply Lemma~\ref{nnlemma}, there exists $M_1 \in \mathbb{N}$, for any $m\geq M_1$, with probability at least $1-\gamma$
\begin{equation}
    D^2_{\widehat{T}_{\max}} \leq \mathcal{O}\left( 
    \frac{{\widehat{T}_{\max}}^2}{ N m}\right)
\end{equation}
then we can choose a large $M_2 \in \mathbb{R}$, for any $m \geq M_2$, $D^2_{\tau}$ is bounded by $\mathcal{O}(N^{-\frac{1}{2}})$. 
We set $M = \max\{M_1, M_2\}$ to satisfy all requirements. 
\subsubsection{Proof of  Corollary \ref{popnn}}
First we need to prove that $f_{\tau}^{\delta}$ is bounded in $L_2$ norm within the desired rate. 

\begin{lemma}
    We can write the function update of neural networks as
    \begin{equation}
        f_{\tau + 1}(\cdot) = 
        f_{\tau}(\cdot) - \eta_0 \mathbb{K}^{(I)}(\cdot, \boldsymbol{X}^{(I)}) (f_{\tau}(\boldsymbol{X}^{(I)}) - \boldsymbol{Y}) +\eta_0 \delta_{\tau}(\cdot)
    \end{equation}
where $\delta_{\tau}(\cdot)$ is the function form of $\delta_\tau$ in (\ref{nnupdaeq})  containing additional linearization error, see appendix for details. 

And for any $\gamma > 0$, there exists $M \in \mathbb{N}$, for  a network described in Lemma \ref{nnlemma}
with width $m \geq M$, 
the following holds with probability at least $1-\gamma$ 
\begin{equation}
    \|\delta_{\tau}(\cdot)\|_2 \leq \mathcal{O}\left(m^{-\frac{1}{2}}\right).
\end{equation}
\label{nnlemmapop}
\end{lemma}
Then we can use the proof by induction to show that the $L_2$ norm of $f_\tau^{\delta}$ is bounded by the desired rate of convergence by choosing a sufficiently large width $M$.

Next, we bound $\varrho_N$ in the case of neural network. By Proposition 5 in \cite{egdecnnused}, we know that  for a two-layer ReLu network,
$\lambda_1, \lambda_2 >0$, 
 $\lambda_j = 0$ for $j = 2k, k \geq 2$ and $\lambda_j \sim C j^{-p}$ otherwise.  
Then for any $\varrho$, when $j=  2$ or $
j = 2k+1, k \geq 0$, if $j \leq \left(\frac{C}{\varrho^2}\right)^{1/p}$, we have  $\lambda_j \geq \varrho^2$.
Let $M = (\frac{C}{\varrho^2})^{1/p}$.
Then we calculate  $\mathcal{R}_{\mathbb{K}}(\varrho)$ directly by definition as follows:
\begin{equation}
    \begin{aligned}
\mathcal{R}_{\mathbb{K}}(\varrho)=\frac{1}{\sqrt{N}} \sqrt{\sum_{j=1}^{\infty} \min \left\{\lambda_j, \varrho^2\right\}} & \asymp \sqrt{\frac{\lceil M\rceil}{2N}} \varrho+\sqrt{\frac{C}{N}} \sqrt{\sum_{2k+1=\lceil M\rceil}^{\infty} (2k+1)^{-p}} \\
& \asymp \sqrt{\frac{M}{2N}} \varrho+\sqrt{\frac{C^{\prime}}{N}} \sqrt{\int_M^{\infty} (2t+1)^{-p} d t} \\
& \asymp \sqrt{\frac{M}{2N}} \varrho+C^{\prime \prime} \frac{1}{\sqrt{N}}(\frac{1}{2M+1})^{\frac{p - 1}{2}} \\
& \asymp  \frac{\varrho^{1-\frac{1}{p}}}{\sqrt{N}} 
\end{aligned}
\end{equation}
Same as the critical empirical rate, the population rate is obtained when
$\mathcal{R}_{\mathbb{K}}(\varrho) = \frac{ \varrho^2 C_{\mathcal{H}^2}}{40 \sigma}$. So we have $\varrho^2_N \asymp\left(\sigma^2 / N\right)^{\frac{p}{p+1}}$, which gives the claimed rate. 

For a ReLu network with number of layers $L\geq 3$, by Corollary 3 in \cite{egedecay}, we know that $\lambda_i \sim i^{-p}$. By a similar claim as above, we obtain the same rate. 

\subsection{Gradient boosted decision trees}
\subsubsection{Proof of Corollary \ref{treecorfix}}

We first show that $\text{tr}(\mathbb{E}K^{(I)}) \leq 2^d$. The trace is bounded as follows by the definition of kernels:
\begin{equation}
    \sum_i \hat{\lambda}_i=\frac{1}{N} \sum_\nu \sum_{i=1}^{L_\nu} \frac{N}{N_\nu^{(i)}} N_\nu^{(i)} \pi(v)=\sum_\nu L_\nu \pi(v) \leq 2^d.
\end{equation}

Next, we show how to bound $D_{\tau}^2$. It suffices to show that in the limit of $N \rightarrow \infty$, the order of convergence for $D_{\tau}^2$ is at least  $\mathcal{O}\left( \frac{1}{\sqrt{N}}   \right)$. 

\begin{lemma}
\label{treedt}
Same as conditions for Corollary \ref{treecorfix}, we can bound the difference term $D_{\tau}^2$ for GBDT as follows:
\begin{equation}
    D_{\tau}^2 \leq \frac{4\epsilon^2}{N} {C^{\prime2}}(M_{\beta} - 1 )^2  \left(\frac{1-e^{-\frac{\tau \epsilon}{2M_{\beta}N}}}{1-   e^{-\frac{\epsilon}{2M_{\beta}N}}}
\right)^2
\end{equation}
where $C^{\prime}$ is a constant of order $O(\sqrt{N})$, and 
$M_{\beta} =  e^{\frac{ d \cdot \|f_{*}(\boldsymbol{X}^{(I)})\|_2^2}{N \beta}}$ where $f_{*}$ is the empirical minimizer.
\end{lemma}
According to the above lemma, 
\begin{equation}
\begin{aligned}
    D_{\tau}^2  
& \leq \frac{4\epsilon^2}{N} {C^{\prime2}}(M_{\beta} - 1 )^2 
\frac{1}{        \left(  1-   e^{\frac{\epsilon}{-2M_{\beta}N}}\right)^2 }
\end{aligned}
\end{equation}
When $N  \rightarrow \infty $, $\frac{C^{\prime2}}{N}$ tends to a constant, and $M_\beta - 1\sim \mathcal{O}(\frac{1}{\beta})$.
Then apply $(1-e^x)^2 \sim x^2, \text{when } x \rightarrow 0$, we have 
\begin{equation}
\begin{aligned}
     \frac{4\epsilon^2}{N} {C^{\prime2}}(M_{\beta} - 1 )^2 
\frac{1}{        \left(  1-   e^{-\frac{\epsilon}{2M_{\beta}N}}\right)^2 } \sim \mathcal{O}\left(   \frac{N^2}{\beta^2}
\right).
\label{treecalbeta}
\end{aligned}
\end{equation}
To have at least order $\mathcal{O}\left( \frac{1}{\sqrt{N}}   \right)$, we need $\beta \geq N^{5/4}$.

\subsubsection{ Proof of Corollary \ref{treecorpop} }
In this case, since GDBT belongs to the finite kernel function class, so $\sum_i \lambda_i$ is definitely finite, which means that
\begin{equation}
    \varrho_N^2 \leq O\left(\frac{1}{\sqrt{N}}\right).
\end{equation}

Next, we use the following lemma to bound $ \| E_u f_{\tau}^{\delta}(\cdot) \|_2$ in the case of GBDT.
\begin{lemma}
\label{treedtpop}
    The $L_2$ norm of $E_u f_{\tau}^{\delta} $ can be bounded with the same rate as the empirical norm $D_{\tau}^2$:
 \begin{equation}
        \| E_u f_{\tau}^{\delta}(\cdot) \|_2^2 \leq \epsilon^2N  {C^{\prime2}}(M_{\beta} - 1 )^2 
\frac{\tau^2}{        \left(  1-   e^{\frac{\epsilon}{-2M_{\beta}N}}\right)^2 }.
    \end{equation}
\end{lemma}
Note that we need to take expectation of $f_\tau^{\delta}$ w.r.t. the randomness of the algorithm.   

Note that we have
\begin{equation}
    \frac{1}{\eta_{\widehat{T}_{\max}+1}} \leq \widehat{\varrho}_n^2 \leq \frac{1}{\eta_{\widehat{T}_{\max}}}
\end{equation}
therefore
\begin{equation}
    \widehat{T}_{\max} \asymp  \frac{\sqrt{N}}{\epsilon}.
\end{equation}
So we can apply the same as in the proof of empirical bound to obtain
\begin{equation}
      \| E_u f_{\widehat{T}_{\text{op}}}^{\delta}(\cdot) \|_2^2 \leq \epsilon^2 N  {C^{\prime2}}(M_{\beta} - 1 )^2 
\frac{{\widehat{T}_{\text{max}}}^2}{     \left(  1-   e^{\frac{\epsilon}{-2M_{\beta}N}}\right)^2 } \sim
O\left( \frac{N^7}{\beta^2} \right),
\end{equation}
where we also use the fact that $\epsilon = O\left(\frac{1}{N}\right)$. 
To get the desired convergence result, we need $\beta \geq N^{15/4}$.

\section{Conclusion and Discussion}


Evaluating the importance of variables in machine learning is essential as these technologies are increasingly used in impactful areas such as autonomous driving, financial and healthcare decisions, and social and criminal justice systems. 
In this paper, we introduce a general algorithm called warm-starting early-stopping, designed for any gradient-based iterative process to efficiently estimate variable importance (VI). We demonstrate that our method accurately estimates VI and matches the precision of more computationally demanding retrain methods. We showcase the general warm-start early-stopping idea works and give a guide of how to use this algorithm in practice.

The theory presented in this paper is a significant advancement towards enhancing the computational efficiency of interpretability in machine learning models. We believe that this theoretical framework could be extended in various other directions as well.  An interesting extension of this is in finite sample case, could we derive the exact or approximate distribution of the VI estimator and then perform a hypothesis testing to test if our estimates are reliable. Second, for GBDT our theory is based on a simple decision tree algorithm, extending the analysis to regular decision tree structure is of great interest. Third, the GBDT is a classic boosting algorithm, extending the lazy training idea to bagging algorithms, like random forest, is also an appealing direction to explore. Fourth, we can expect that
the stability results of NTK to be extended to more general initialization settings and more general network architecture.

\section{Acknowledgment}
Support for this research was provided by American Family Insurance through a research partnership with the University of Wisconsin–Madison’s American Family Insurance Data Science Institute.

\bibliographystyle{unsrtnat}
\bibliography{references}

\appendix

\section{General algorithm}
\subsection{Proof of Lemma \ref{applem1}}
\begin{proof}
Starting from equation (16) in the main text, we have
\begin{equation}
   \begin{aligned}f_{\tau+1}(\boldsymbol{X}^{(I)}) 
&= (I-\epsilon K^{(I)})f_{\tau}(\boldsymbol{X}^{(I)}) + \epsilon K^{(I)}  \boldsymbol{Y} + \epsilon\delta_\tau
\label{deq}
\end{aligned}
\end{equation}
where $\delta_\tau=\left(K^{(I)}-K^{(I)}_\tau\right)\left[f_\tau\left(\boldsymbol{X}^{(I)}\right)-\boldsymbol{Y}\right]$. 

As discussed above, we should analyze $\boldsymbol{e}^{(I)}$. Let $f_{\tau}^{\prime} = f_{\tau}- f_N^c$. Then we can write the above equation as:
\begin{equation}
   \begin{aligned}f_{\tau+1}^{\prime}(\boldsymbol{X}^{(I)}) 
&= (I-\epsilon K^{(I)})f^{\prime}_{\tau}(\boldsymbol{X}^{(I)}) + \epsilon K^{(I)}  \boldsymbol{e}^{(I)} + \epsilon\delta^{\prime}_\tau
\label{deq2}
\end{aligned}
\end{equation}
where $\delta^{\prime}_\tau=\left(K^{(I)}-K^{(I)}_\tau\right)\left[f^{\prime}_{\tau}\left(\boldsymbol{X}^{(I)}\right)-\boldsymbol{e}^{(I)}\right]$. 

Let $\zeta^{\tau} = \frac{1}{\sqrt{N}}U^T f_{\tau}^{\prime}(\boldsymbol{X}^{(I)}) $ and 
$\zeta^* = \frac{1}{\sqrt{N}}U^T \left[f_{0,-I}(\boldsymbol{X}^{(I)}) - f_N^c (\boldsymbol{X}^{(I)}) \right]$, we can write 
\begin{equation}
    \zeta^{\tau + 1} = \zeta^{\tau} + \epsilon \Lambda\frac{\tilde{w}}{\sqrt{N}} - \epsilon \Lambda(\zeta^{\tau} - \zeta^*) + \frac{\epsilon U^T\delta_{\tau}}{\sqrt{N}}
\end{equation}
where $\tilde{w} = U^T[\boldsymbol{e}^{(I)}- f_{0,-I}(\boldsymbol{X}^{(I)}) + 
f_{N}^c(\boldsymbol{X}^{(I)}) 
] = U^T w^{(I)}$.

Rearranging, we have
\begin{equation}
    \zeta^{\tau + 1} - \zeta^* =(I-\epsilon\Lambda)( \zeta^{\tau} - \zeta^*)+ \epsilon \Lambda\frac{\tilde{w}}{\sqrt{N}}  + \frac{\epsilon \tilde{\delta}_{\tau}}{\sqrt{N}}
\end{equation}
where $\tilde{\delta}_{\tau} = U^T \delta_{\tau}$.

Unwrapping, we have
\begin{equation}
    \zeta^{\tau}- \zeta^* = (I-S^{\tau})\frac{\tilde{w}}{\sqrt{N}} - S^{\tau} \zeta^* +  
\sum_{i=0}^{\tau-1} S^{\tau - i -1} 
\frac{\epsilon\tilde{\delta}_{i}}{\sqrt{N}}.
\label{toreferlater}
\end{equation}
The zero initialization condition is important, otherwise we cannot derive the above equation. 
Then we have
\begin{equation}
    \|\zeta^{\tau}-\zeta^*\|_2^2 \leq \frac{4}{N}\|(I-S^{\tau}) \tilde{w} \|_2^2 + \frac{4\epsilon^2}{N}\|\sum_{i=0}^{\tau-1} S^{\tau - i -1} \tilde{\delta}_{i}\|_2^2 + 2\|S^{\tau} \zeta^*\|_2^2
\end{equation}
where we use the inequality $\|a+b\|_2^2 \leq 2\left(\|a\|_2^2+\|b\|_2^2\right)$ twice. Plug in $\tilde{w}, \tilde{\delta}_{i}$ and $\zeta^*$, we can the desired results.  Note that we subtract $f_N^c$ from both $f_\tau$ and $f_{0,-I}$, so it simply cancels out. 
\end{proof}



\subsection{Proof of Lemma \ref{indupopf}}
\begin{proof}
In a high level,
the lemma is to derive the population version of the last term in \ref{toreferlater}. So we can 
separate 
the function $f_\tau$ into  two parts, one are purely functions in $\text{span}\{ \mathbb{K}^{(I)}(\cdot,X_{-I}) \}$ and the other one is some difference function resulted from the changing kernels during training. 

In general, the function form of \textit{iterative kernel update equation} can be written as
\begin{equation}
\begin{aligned}
      f_{\tau + 1}(\cdot) & = f_{\tau}(\cdot) + 
    \frac{\epsilon}{N} \mathbb{K}_{\tau}^{(I)}(\cdot, \boldsymbol{X}^{(I)})[\boldsymbol{Y} - f_{\tau}(\boldsymbol{X}^{(I)})] \\
    & = f_{\tau}(\cdot) + \frac{\epsilon}{N} \mathbb{K}^{(I)}(\cdot, \boldsymbol{X}^{(I)})[\boldsymbol{Y} - f_{\tau}(\boldsymbol{X}^{(I)})] +
    \frac{\epsilon}{N} \left(
    \mathbb{K}_{\tau}^{(I)}(\cdot, \boldsymbol{X}^{(I)}) - \mathbb{K}^{(I)}(\cdot, \boldsymbol{X}^{(I)})
    \right) [\boldsymbol{Y} - f_{\tau}(\boldsymbol{X}^{(I)})]
\end{aligned}
\end{equation}
For $\tau = 1$, it is easy to see that 
\begin{equation}
    f_1^{\delta}(\cdot) = \epsilon \delta_0(\cdot).
\end{equation}
We then use proof by induction to prove our claim. 
Suppose the formula holds for $f_{\tau}$, then 
denote the pure part from $\text{span}\{ \mathbb{K}^{(I)}(\cdot,X_{-I}) \}$  for $f_\tau$ as $f_{\tau}^p$. We can write
\begin{equation}
    \begin{aligned}
        f_{\tau+1}(\cdot) &= f_\tau^p(\cdot) +  f_\tau^{\delta}(\cdot) + 
       \frac{\epsilon}{N} \mathbb{K}_{\tau}^{(I)}(\cdot, \boldsymbol{X}^{(I)})[\boldsymbol{Y} - f_{\tau}^p(\boldsymbol{X}^{(I)})
        - f_{\tau}^{\delta}(\boldsymbol{X}^{(I)})] 
+ \epsilon \delta_{\tau}(\cdot) \\
&= \underbrace{f_\tau^p(\cdot) + \frac{\epsilon}{N} \mathbb{K}_{\tau}^{(I)}(\cdot, \boldsymbol{X}^{(I)})[\boldsymbol{Y} - f_{\tau}^p(\boldsymbol{X}^{(I)})]}_{
f_{\tau+1}^p(\cdot)
} + 
\underbrace{
f_\tau^{\delta}(\cdot) - \frac{\epsilon}{N}\mathbb{K}^{(I)}(\cdot, \boldsymbol{X}^{(I)})f_\tau^{\delta}(\boldsymbol{X}^{(I)}) + \epsilon \delta_{\tau}(\cdot)}_{
f_{\tau+1}^{\delta}(\cdot)
}.
    \end{aligned}
\end{equation}
And this concludes our proof. 
Observe that $f_{\tau}^p$ is the function we have if the kernel does not change during training. 

\end{proof}
\section{Neural network} \label{nn}
We aim to give all the theoretical results for neural networks in this section. As stated in the main text, to fit neural networks, the key is to extend Theorem 2.1 in \cite{linearnn} to our warm-start initialization settings. We adopt the same notations and assumptions as those \cite{linearnn} to facilitate the proof in this section. Note that this set of notations is a little bit different from that used in the main text.


\subsection{Notation and preliminaries}

 Define $\theta^l \equiv \operatorname{vec}\left(\left\{W^l, b^l\right\}\right)$, the $\left(\left(n_{l-1}+1\right) n_l\right) \times 1$ vector of all parameters for layer $l . \theta=$ $\operatorname{vec}\left(\cup_{l=1}^{L+1} \theta^l\right)$ then indicates the vector of all network parameters, with similar definitions for $\theta \leq l$ and $\theta^{>l}$. Denote by $\theta_t$ the time-dependence of the parameters and by $\theta_0$ their initial values. We use $f_\tau(x) \equiv h^{L+1}(x) \in \mathbb{R}$ to denote the output of the neural network at iteration $\tau$.  Since the standard  parameterization does not have the scaling factor $\frac{1}{\sqrt{n_L}}$ in the output layer, when we define the NTK for standard  parameterization, we should normalize it  with the width $m$ as follows:
\begin{equation}
   \frac{1}{m} \left\langle
     \nabla_{\theta} f\left(\theta(t), x\right),
     \nabla_{\theta} f\left(\theta(t), x^{\prime}\right)\right\rangle.
\end{equation}
The corresponding  \textit{empirical kernel matrix} $K^{(I)}$ under standard parameterization and warm-start initialization has entries:
\begin{equation}
     K^{(I)}(i,j) = \lim_{m\rightarrow \infty}\frac{1}{N} \cdot
      \frac{1}{m} \left\langle
     \nabla_{\theta} f\left(\theta(0), \mathbf{X}_i^{(I)}\right),
     \nabla_{\theta} f\left(\theta(0), \mathbf{X}_j^{(I)}\right)\right\rangle.
\end{equation}


We define the below short-hand:
\begin{equation}
    \begin{aligned}
& f\left(\theta_\tau \right)=f\left(\boldsymbol{X}^{(I)}, \theta_\tau\right) \in \mathbb{R}^{N} \\
& g\left(\theta_\tau\right)=f\left(\boldsymbol{X}^{(I)} , \theta_\tau\right) - 
\boldsymbol{Y}  \in \mathbb{R}^{N} \\
& J\left(\theta_\tau \right)=\nabla_\theta f\left(\boldsymbol{X}^{(I)} , \theta_\tau \right) \in \mathbb{R}^{N \times|\theta|}
\end{aligned}
\end{equation}
where $N$ is the number of training samples.  Then the least-square loss we consider becomes 
\begin{equation}
    \mathcal{L}(t)=\frac{1}{2N}\left\|g\left(\theta_\tau\right)\right\|_2^2.
\end{equation}

 The linearized network is defined as:
 \begin{equation}
     f_\tau^{\operatorname{lin}}(x) \equiv f_0(x)+\left.\nabla_\theta f_0(x)\right|_{\theta=\theta_0} \omega_\tau
 \end{equation}
where $\omega_\tau \equiv \theta_\tau -\theta_0$ is the change in the parameters from their initial values. The linearized network is simply first order Taylor expansion.  
And the change of $w_\tau$ for $f_\tau^{\operatorname{lin}}(x)$ is given by
\begin{equation}
    w_{\tau + 1} -  w_{\tau} = - \frac{\epsilon }{N}J(\theta_0)^T \left( 
    f_\tau^{\operatorname{lin}}(\boldsymbol{X}^{(I)}) - \boldsymbol{Y} 
    \right).
\end{equation}

\subsection{Neural network linearization}
\label{nnlinour}
 We denote the parameter of $f_N^c$ as $\theta_0$, and let $\theta_0^{\prime}$ be the parameter of the initial network used to train $f_N^c$.
Follow the same proof strategy of \cite{linearnn}, 
we prove convergence of neural network training and the stability of NTK for discrete gradient
descent under our special warm-start initialization. 
We first prove the local Lipschitzness of the Jacobian, and then prove the global convergence of the NTK. Finally, with both of these two results we are able to show that the linearization of neural network under our setup is valid. 

\begin{lemma}[Local Lipschitzness of the Jacobian under warm-start initialization]
Under assumptions \ref{linea1}-\ref{linea4},
 there is a $L>0$ such that for every $C>0$, 
 for $\gamma > 0$ and $\eta_0<\eta_{\text {critical }}$, 
there exists $M \in \mathbb{N}$, 
for $m \geq M$,
the following holds with probability at least $\left(1-\gamma\right)$ over warm-start initialization
\begin{equation}
\left\{\begin{array}{ll}
\frac{1}{\sqrt{m}}\|J(\theta)-J(\tilde{\theta})\|_F & \leq L\|\theta-\tilde{\theta}\|_2 \\
\frac{1}{\sqrt{m}}\|J(\theta)\|_F & \leq L
\end{array}, \quad \forall \theta, \tilde{\theta} \in B\left(\theta_0, C m^{-\frac{1}{2}}\right)\right.    
\end{equation}
where
\begin{equation}
    B\left(\theta_0, R\right):=\left\{\theta:\left\|\theta-\theta_0\right\|_2<R\right\}
\end{equation}
and 
$\theta_0$ is the parameter of $f_N^c$.
\label{liplemma}
\end{lemma}
\begin{proof}
Since the full model $f_N^c$  is training from random normal initialization, we can use the theoretical results in \cite{linearnn} directly to prove this result. 
By theorem G.1 in \cite{linearnn}, $\exists M \in \mathbb{N}$, for $m \geq M$, the following holds with probability at least $1-\frac{\gamma}{2}$:
\begin{equation}
    \|\theta_0 - \theta_0^{\prime}\|_2 \leq C_1 m^{-\frac{1}{2}}
\end{equation}
where $\theta_0^{\prime}$ is the parameter of a network with initialization in (\ref{stdnorm}).
Then consider any $\theta, \tilde{\theta} \in B\left(\theta_0, C m^{-\frac{1}{2}}\right)$
\begin{equation}
    \|\theta - \theta_0^{\prime} \|_2 \leq  \|\theta - \theta_0^{\prime} \|_2 +  \|\theta_0 - \theta_0^{\prime} \|_2 \leq (C + C_1)  m^{-\frac{1}{2}}
\end{equation}
same argument holds for $\tilde{\theta}$. Thus we have 
$\theta, \tilde{\theta} \in B\left(\theta^{\prime}_0, (C+C_1) m^{-\frac{1}{2}}\right)$. Apply lemma 1 in \cite{linearnn} with probability 
$1- \frac{\gamma}{2}$ we can  get  the desired result. 
\end{proof}
\begin{remark}
     Note that the width $m$ requirement here is for the full model $f_N^c$, because we used Theorem G.1 in \cite{linearnn} w.r.t. $f_N^c$. \end{remark}

\begin{theorem}
\label{nnconthm}
    Under assumptions \ref{linea1}-\ref{linea4}, for $\gamma>0$ and   $\eta_0<\eta_{\text {critical }}$, there exist $R_0>0, M\in \mathbb{N}$ and $L>1$, such that for every $m \geq M$, the following holds with probability at least $\left(1-\gamma\right)$ over warm-start initialization when applying gradient descent with learning rate $\epsilon=\frac{\eta_0}{m}$,
\begin{equation}
    \left\{\begin{array}{l}
\left\|g\left(\theta_\tau\right)\right\|_2 \leq\left(1-\frac{\eta_0 N \lambda_{\min }}{3}\right)^\tau R_0 \\
\sum_{j=1}^\tau\left\|\theta_j-\theta_{j-1}\right\|_2 \leq \frac{\eta_0 K^{(I)} R_0}{\sqrt{n}} \sum_{j=1}^\tau\left(1-\frac{\eta_0 N\lambda_{\min }}{3}\right)^{j-1} \leq \frac{3 L R_0}{N\lambda_{\min }} m^{-\frac{1}{2}}
\end{array}\right.
\end{equation}
and
\begin{equation}
    \sup_{\tau}\left\|K^{(I)}_0-K^{(I)}_{\tau}\right\|_F \leq \frac{6 L^3 R_0}{N^2 \lambda_{\min }} m^{-\frac{1}{2}}
\end{equation}
where $K^{(I)}_0$ is the empirical kernel matrix induced by $f_N^c$ and $K^{(I)}_\tau$ is the one induced by $f_{\tau}$ using dropping features $\boldsymbol{X}^{(I)}$. 
\end{theorem}
\begin{proof}
Follow the same proof techniques as theorem G.1 in \cite{linearnn}, we first need to have 
\begin{equation}
    \left\|g\left(\theta_0\right)\right\|_2<R_0.
\end{equation}
This  is ensured by assumption 4.4 in the main text. Note that in \cite{linearnn}, the input data is considered to be fixed, so the sample size is just a constant. But in our setup, we need to be more careful with the constants that are dependent with the training size $N$ as this is crucial when we obtain the prediction error bound in terms of $N$. Here the $R_0$ is a constant of order $O(\sqrt{N})$. We do not write the dependence of $R_0$ on $N$ here, but we will be careful about the 
effects of $R_0$ in the proof of prediction error bound.

The next condition we need to satisfy is something similar to equation (S61) in \cite{linearnn}. Specifically, we should have with high probability 
\begin{equation}
    \| K^{(I)}- K^{(I)}_0 \|_F \leq \frac{\eta_0 \lambda_{min}}{3}.
\end{equation}
Note that the sample size $N$ on both sides cancel out, and here the empirical kernel matrix is calculated using dropped data  $\boldsymbol{X}^{(I)}$.
The key difference in our setup is that we drop features in set $I$, the data becomes different. 
A fact about NTK is that, the convergence of NTK when width goes to infinity is independent of data. So even though we change the input data, this result still holds. 

The full model is trained using complete data $\boldsymbol{X}$ and random normal
initialization, so theorem G.1 in \cite{linearnn}, the difference between $\theta_0$ and some random initialization $\theta_0^{\prime}$ is still within $O(m^{-\frac{1}{2}})$. Then by Lemma 1 in \cite{linearnn}, we can show that  the convergence of NTK still holds. This argument is essentially the same as the reasoning as equation (S217) in \cite{linearnn}. The reason why we can do this is because in the proof of Lemma 1 in \cite{linearnn}, the constant depends on the training sample size $N$ rather than specific data we use. Notice that in the detailed proof of 
$lemma$ 1 in \cite{linearnn},   in equation (S85) and (S86), 
we can set the constant large enough so that for any input data, the local Lipschitz condition holds. This is ensured by the assumption that the input data lies in a closed set.  

Finally, combined with lemma \ref{liplemma}, we are able to prove this result using the same proof strategy of theorem G.1 in \cite{linearnn}.
\end{proof}
\begin{remark}
    The assumption 4.4 seems too strong in the case of neural network. But we can formally show that the dropout error is at least $O_p(\sqrt{N})$. First when we train the full model $\|g(\theta^\prime_0,\boldsymbol{X})\|_2$ is $O_p(\sqrt{N})$, where $\theta^\prime_0$ is the parameter with random normal initialization (see proof in \cite{linearnn} (S44)). Then since the loss is decreasing over training $\|g(\theta_0,\boldsymbol{X})\|_2$, where $\theta_0$ is the parameter of $f_N^c$. Remember we dropped data, so we need to show that $\|g(\theta_0,\boldsymbol{X}^{(I)})\|_2$ is $O_p(\sqrt{N})$.  This is sufficient to show that $f(\theta)$ is lipschitz continuous w.r.t. input data with high probability. This is not hard to show as $\theta_0$ and $\theta_0^{\prime}$ are close and $\|W_0\|_{op}$ is bounded w.h.p. similar to arguments as (S84) in \cite{linearnn} . 
 \end{remark}

\begin{theorem}
 Under assumptions \ref{linea1}-\ref{linea4}, for $\gamma>0$ arbitrarily small and $\eta_0<\eta_{\text {critical }}$, there exist $R_0>0$ and $M \in \mathbb{N}$ such that for every $m \geq M$, with probability at least $\left(1-\gamma\right)$ over warm-start initialization when applying gradient descent with learning rate $\epsilon=\frac{\eta_0}{m}$
\begin{equation}
    \sup _\tau\left\|f_{\tau}^{\text {lin }}(\boldsymbol{X}^{(I)})-f_{\tau}(  \boldsymbol{X}^{(I)}  )\right\|_2  ,
    \sup _{\tau}\left\|f_{\tau}^{\text {lin }}( x)-f_{\tau}( x)\right\|_2 
    \lesssim 
    m^{-\frac{1}{2}} R_0^2
\end{equation}
where $\lesssim $ is to hide the dependence on some uninteresting constants.
\label{linthm}
\end{theorem}
\begin{proof}
    With lemma \ref{liplemma} and theorem \ref{nnconthm}, follow the same proof techniques as Theorem H.1 in \cite{linearnn}
we can prove this result. Note that 
Theorem H.1 in \cite{linearnn} is  for  the case of gradient flow.  For the gradient descent case, the update of linear model is given by:
\begin{equation}
    \begin{aligned}
& \omega_{t+1} - \omega_{t}= -\epsilon \nabla_\theta f_0(\mathcal{X})^T \nabla_{f_t^{\operatorname{lin}}(\mathcal{X})} \mathcal{L} \\
& {f}_{\tau + 1}^{\operatorname{lin}}(x) -  {f}_{\tau}^{\operatorname{lin}}(x)=-\epsilon J(\theta_0, x) J(\theta_0, \boldsymbol{X}^{(I)} )^T\nabla_{f_t^{\operatorname{lin}}(\mathcal{X})} \mathcal{L} .
\end{aligned}
\end{equation}
Note that the update for function value is exact, because the linear model only contains linear terms. Then by induction, it is not hard to show that 
\begin{equation}
    f_{\tau}^{\operatorname{lin}}(\boldsymbol{X}^{(I)}) = \left( 
I - \left(I-\epsilon K_0^{(I)}\right)^{\tau}\right) \boldsymbol{Y} + \left(I-\eta K_0^{(I)}\right)^{\tau}f_0(\boldsymbol{X}^{(I)})
\end{equation}
For the nonlinear counterpart, we need to use mean value theorem, at each update we have
\begin{equation}
    f_{\tau+1}(x) - f_{\tau+1}(x) = J(\tilde{\theta}_{\tau}) \Delta \theta
\end{equation}
where $\tilde{\theta}_{\tau}$ is some some linear interpolation between $\theta_{\tau}$ and $\theta_{\tau + 1}$. If we use $J({\theta}_{\tau})$, the expression is not exact and there will be an second order Taylor expansion remainder term. Then use  lemma \ref{liplemma} and theorem \ref{nnconthm}, it is not hard to prove the results (apply similar arguments as in the proof of Lemma \ref{nnlemma} ) . 
\end{proof}
\begin{remark}
    For NTK parameterization, analogs of these theoretical results still hold. For the local Lipschitzness, we just drop the scaling factor $\frac{1}{\sqrt{m}}$. And for the stability  and linearization results, we replace learning rate with $\epsilon = \eta_0$. And our requirement on $f_N^c$ becomes initializing all the parameters with $\mathcal{N}(0,1)$. The proof are almost the same with some minor changes.

The stability of NTK under warm-start initialization is of independent interest. In this paper, we extend the results in \cite{linearnn}  to our setup. We make use of the least square loss and the proof is not that long. Actually, for a general setting, where the training loss and direction are more general like results in \cite{jacot}, we can 
also prove that the stability of NTK under warm-start initialization still holds in the gradient flow case. 
To achieve that, we also need some boundness condition. 
We can  show that the output of each layer of $f_N^c$ 
is bounded by some non-centered normal distributions 
and then follow a similar proof strategy as \cite{jacot} to prove the counterpart of theorem 2 in
\cite{jacot} in the warm-start setup. Same as here, we require the parameters of $f_N^c$ to be initialized with $\mathcal{N}(0,1)$. The details are not presented in this paper as we do not use that result directly. 
\end{remark}

\subsection{Proof of Lemma \ref{nnlemma}}

\begin{proof}
    By theorem \ref{linthm}, we have for any $\gamma > 0$, there exists $M_1 \in \mathbb{N}$, a full connected neural network with width $m \geq M_1$, with probability at least $1-\frac{\gamma}{3}$
    \begin{equation}
        \begin{aligned}
 f_{\tau}({\boldsymbol{X}^{(I)}}) &= f_0(\boldsymbol{X}^{(I)}) +  \left.\nabla_\theta f_0(\boldsymbol{X}^{(I)})\right|_{\theta=\theta_0} \omega_\tau +  e_\tau \\
 f_{\tau+1}({\boldsymbol{X}^{(I)}} )&= f_0(\boldsymbol{X}^{(I)}) +  \left.\nabla_\theta f_0(\boldsymbol{X}^{(I)})\right|_{\theta=\theta_0} \omega_{\tau+1} + e_{\tau+1}
        \end{aligned}
    \end{equation}
where for $e_{\tau}$ and $e_{\tau+1}$ we have
\begin{equation}
   \|e_{\tau}\|_2,  \|e_{\tau+1}\|_2 \lesssim 
    m^{-\frac{1}{2}} R_0^2.
\end{equation}
Subtract these two equations, we have
  \begin{equation}
        \begin{aligned}
 f_{\tau+1}({\boldsymbol{X}^{(I)}} )&= f_\tau({\boldsymbol{X}^{(I)}}) +  \left.\nabla_\theta f_0({\boldsymbol{X}^{(I)}})\right|_{\theta=\theta_0}( \omega_{\tau+1} -\omega_{\tau})+ 
e_{\tau+1} - e_{\tau}
 \\ 
 &= f_\tau({\boldsymbol{X}^{(I)}})  -   \frac{\eta_0}{m N} J(\theta_0) J(\theta_0)^{T} g(\theta_\tau) + 
 \frac{\eta_0}{m N} J(\theta_0) J(\theta_0)^{T}e_{\tau}
 +
 e_{\tau+1} - e_{\tau} \\
 &=  f_\tau({\boldsymbol{X}^{(I)}})  - \eta_0
 K^{(I)} g(\theta_\tau) +  \eta_0
 (K^{(I)}-K^{(I)}_0)g(\theta_\tau) + 
\eta_0
 K^{(I)} e_{\tau}
 +
 e_{\tau+1} - e_{\tau} \\
 &= 
f_\tau({\boldsymbol{X}^{(I)}}) - \eta_0
 K^{(I)}( f_{\tau} (\boldsymbol{X}^{(I)}) - \boldsymbol{Y}) + \eta_0 \delta_{\tau}
        \end{aligned}
    \end{equation}
where $\delta_{\tau}$ is 
\begin{equation}
(K^{(I)}-K^{(I)}_0)g(\theta_\tau) +  K^{(I)} e_{\tau} + \frac{e_{\tau+1} - e_{\tau}}{\eta_0}.
\end{equation}
Then apply lemma \ref{liplemma} and theorem \ref{nnconthm} with probability $  1-\frac{\gamma}{3}$, there exists $M_2$, s.t. for any $m \geq M_2$,
we can bound the $L_2$ norm of $\eta_0 \delta_{\tau}$  with probability at least $1-\gamma$
\begin{equation}
    \|  \delta_{\tau}\|_2 \leq  \mathcal{O}\left(
    \frac{R_0^2}{ m^{\frac{1}{2}}}\right).
\end{equation}
Choose $M = \max\{M_1, M_2\}$, then we finish our proof.
\end{proof}
    We can see that here since neural network has an additional linearlization error, the term $\delta_\tau$ is more complicated than the noraml one in the general theretical framework. 

\subsection{Proof of Lemma \ref{nnlemmapop}}
\setcounter{lemma}{4}
\begin{lemma}
    We the function update of neural network, we can write
    \begin{equation}
        f_{\tau + 1}(\cdot) = 
        f_{\tau}(\cdot) - \eta_0 \mathbb{K}^{(I)}(\cdot, \boldsymbol{X}^{(I)}) (f_{\tau}(\boldsymbol{X}^{(I)}) - \boldsymbol{Y}) +\eta_0 \delta_{\tau}(\cdot)
    \end{equation}
where $\delta_{\tau}(\cdot)$ is the function form of $\delta_\tau$ in (\ref{nnupdaeq})  containing additional linearlization error, see appendix for details. 

And for any $\gamma > 0$, there exists $M \in \mathbb{N}$, for  a network described in Lemma \ref{nnlemma}
with width $m \geq M$, 
the following holds with probability at least $1-\gamma$ 
\begin{equation}
    \|\delta_{\tau}(\cdot)\|_2 \leq \mathcal{O}\left(m^{-\frac{1}{2}}\right).
\end{equation}
\end{lemma}
\begin{proof}
The proof are quite similar to  the proof of Lemma \ref{nnlemma}, just change $\boldsymbol{X}^{(I)}$ to general  $\mathbf{X}$ 
and apply same startegy. 
\end{proof}

\section{Gradient
boosting decision trees}
\subsection{Notations}
We list the used notations in this section for GBDT below, some of them are the same as the main text:
\begin{itemize}
\item $\| \cdot \|_{\mathbb{R}^N}$ --- $\| \cdot \|_2$ of $f(\boldsymbol{X}^{(I)})$ ; 
    \item $\rho$ --- distribution of features;
    \item $N$ --- number of samples;
    \item $\mathcal{V}$ --- set of all possible tree structures;
    \item $L_\nu: \mathcal{V} \rightarrow \mathbb{N}$ --- number of leaves for $\nu \in \mathcal{V}$
    \item $D(\nu, z)$ --- score used to choose a split (3);
    \item $\mathcal{S}$ --- indices of all possible splits;
    \item $n$ --- number of borders in our implementation of SampleTree;
    \item $d$ --- depth of the tree in our implementation of SampleTree;
    \item $\beta$ --- random strength;
    \item $\epsilon$ --- learning rate;
    \item $\mathcal{F}$ --- space of all possible ensembles of trees from $\mathcal{V}$;
    \item $\phi_\nu^{(j)}$ --- indicator of $j$-th leaf;
\item $k_\nu(\cdot, \cdot)$ --- single tree kernel;
\item $L(f)$ --- $\frac{1}{2 N}\left\|\boldsymbol{Y}-f\left(\boldsymbol{X}^{(I)}\right)\right\|_{\mathbb{R}^N}^2$, empirical error of a model $f$;

\item  $f_*$ --- $ {\arg \min}_{f \in \mathcal{F}} L(f)$, empirical minimizer ;
\item $V(f)$ --- $L(f) - L(f_*)$, excess risk;
\item $R$ --- $\|f_*\|_{\mathbb{R}^N}  $;
\item $\mathbb{K}(\cdot, \cdot)$ --- stationary kernel of the GBDT;
\item $p(\cdot \mid f, \beta)$ --- distribution of trees, $f \in \mathcal{F}$;

\item  $\pi(\cdot)=\lim _{\beta \rightarrow \infty} p(\cdot \mid f, \beta)=p\left(\cdot \mid f_*, \beta\right)$ --- stationary distribution of trees;
\item $\boldsymbol{e}^{(I)}$ --- dropout error $\boldsymbol{Y} - f_N^c(\boldsymbol{X}^{(I)})$;
\item $M_{\beta}$ --- $e^{\frac{m R^2}{N \beta}}$.
\end{itemize}

\subsection{Tree Structure}
\label{sampletreealgo}

The weak learner \textit{SampleTree} here is oblivious decision tree, all nodes at a given level share the same splitting criterion (feature and threshold).  The tree is built in a top-down greedy manner. 
 To reduce the possible number of splits for each feature, we discretize each feature into $n+1$ bins. This means that for each feature, there are $n$ potential thresholds that can be selected freely.
A common method involves quantizing the feature so that each of the $n+1$ buckets contains roughly an equal number of training samples.
 The depth of the tree is constrained to a maximum of $d$. Remember, we represent the collection of all potential tree structures as $\mathcal{V}$.
At each step, the algorithm chooses one split among all the remaining candidates based on the following \textit{score} defined for $\nu \in \mathcal{V}$ and residuals $z$ :
\begin{equation}
    D(\nu, z):=\frac{1}{N} \sum_{j=1}^{L_v} \frac{\left(\sum_{i=1}^N \phi_\nu^{(j)}\left(x_i\right) z_i\right)^2}{\sum_{i=1}^N \phi_\nu^{(j)}\left(x_i\right)}.
\end{equation}
\begin{minipage}{0.46\textwidth}
\begin{algorithm}[H]
    \centering
    \caption{SampleTree$(z;d,n, \beta)$}\label{algorithm}
    \footnotesize
    \begin{algorithmic}
        \State \textbf{input:}  \text{
        residuals
        } $z = (z_i)_{i=1}^N$
        \State \textbf{output:}  oblivious tree structures $\nu \in \mathcal{V}$
        \State \textbf{hyper-parameters:} number of feature splits $n$, \\
        max tree depth $d$, random strength $\beta \in [0,\infty)$
        \State \textbf{definitions:} \\
        $S =\{ (j,k) | j\in \{1,\dots,d\}, k \in \{1,\dots, n\}  \}$ --- 
        indices of all possible splits
        \State \textbf{instructions:}\\
        initialize $i =0, \nu_0 = \emptyset, S^{(0)} =S$
        \Repeat
        \State sample $(u_i(s))_{s\in S^{i}} \sim U\left( 
        [0,1]^{nd-i}
        \right)$
        \State   choose the next split as $\{s_{i+1}\} = $
        \State $\argmax_{s\in S^{(i)}}\left( D\left((\nu_i,s),z\right)  -\beta \log(-\log u_i(s))\right) $
         \State  update tree:  $\nu_{i+1} = (\nu_i,s_{i+1})$
      \State   update candidate splits: 
      $S^{(i+1)} = S^{(i)} \setminus \{s_{i+1}\}$

        \State $i = i + 1$
        \Until $i \geq d$ \textbf{or} $S^{(i)} = \emptyset$
        \State \textbf{return:} $\nu_i$
       
    \end{algorithmic}
    \label{stalgo}
\end{algorithm}
\end{minipage}
\hfill
\begin{minipage}{0.46\textwidth}
\begin{algorithm}[H]
    \centering
    \caption{TrainGBDT$(\mathcal{D};\epsilon, T, d, n ,\beta)$}\label{algorithm1}
    \footnotesize
    \begin{algorithmic}
        \State \textbf{input:} dataset $\mathcal{D} = (\boldsymbol{X}^{(I)}, \boldsymbol{Y})$
        \State \textbf{hyper-parameters:} learning rate $\epsilon > 0$, \\
        regularization $\lambda >0$, iteration of boosting $T$, \\
        parameters of SampleTree $d,n,\beta$ \\
        \State \textbf{instructions:}\\
        initialize $\tau = 0, f_0(\cdot) = 0$ \\
        \Repeat
        
        \State $z_{\tau} = \boldsymbol{Y} - f_{\tau} (\boldsymbol{X}^{(I)})$
        \State $\nu_{\tau} = \text{SampleTree}(z_{\tau}; d,n,\beta)$
         \State $\theta_\tau=\left(\frac{\sum_{i=1}^N \phi_{\nu_\tau}^{(j)}\left(x_i\right) z_\tau^{(i)}}{\sum_{i=1}^N \phi_{\nu_\tau}^{(j)}\left(x_i\right)}\right)_{j=1}^{L_{\nu_\tau}}$
        \State $f_{\tau + 1}(\cdot) = (1- \frac{\lambda\epsilon}{N}) f_{\tau }(\cdot) + \epsilon\left\langle\phi_{\nu_\tau}(\cdot), \theta_\tau\right\rangle_{\mathbb{R}^{L_{\nu_\tau}}}
        $
       \State $\tau = \tau + 1$
        \Until  $\tau \geq T$ 
        \State \textbf{return:}  $f_T(\cdot)$
    \end{algorithmic}
\end{algorithm}
\end{minipage}

The \textit{SampleTree} algorithm induces a local family of distribution $p(\nu|f,\beta)$ for each $f\in \mathcal{F}$:
\begin{equation}
    p(\nu \mid f, \beta)=\sum_{\varsigma \in \mathcal{P}_d} \prod_{i=1}^d \frac{e^{\frac{D\left(\nu_{\varsigma, i}, z\right)}{\beta}}}{\sum_{s \in \mathcal{S} \backslash \nu_{\varsigma, i-1}} e^{\frac{D\left(\left(\nu_{\varsigma, i-1}, s\right), z\right)}{\beta}}}
\end{equation}
where the sum is over all permutations $\varsigma \in \mathcal{P}_d, \nu_{{\varsigma},i} = (s_{\varsigma(1)}, \dots, s_{\varsigma(i)}), \nu = (s_1,\dots, s_d)$.

\subsection{Proof of Lemma \ref{treekernel}}
As mentioned in the main text, because of the randomness of the \textit{SampleTree} algorithm, we need to analyze $\mathbb{E} f_{\tau}$ to be consistent with our general framework. In this subsection, we show the necessary modifications of the general framework to fit the GBDT algorithm.

Because of the additive structure of the boosting algorithm, we can subtract $f_{N}^c(\boldsymbol{X}^{(I)})$ from $\boldsymbol{Y}$. We can then analyze from the dropout error $\boldsymbol{e}^{(I)}$ and regard this as a GBDT with initialization 0. Because $f_{N}^c(\boldsymbol{X}^{(I)})$
is also generated by \textit{SampleTree}, it is also random, so $\boldsymbol{e}^{(I)}$  is random. 
\setcounter{lemma}{2}
\begin{lemma}
The iterative kernel update equation for  the particular GBDT we consider is:
    \begin{equation}
    \begin{aligned}\mathbb{E}f_{\tau+1}(\boldsymbol{X}^{(I)}) 
&= (I-\epsilon \mathbb{E}K^{(I)})\mathbb{E}f_{\tau}(\boldsymbol{X}^{(I)}) + \epsilon \mathbb{E}K^{(I)} \mathbb{E} \boldsymbol{e}^{(I)} + \epsilon\delta_\tau
\end{aligned}
\end{equation}
where $\delta_{\tau} = \mathbb{E}(K^{(I)}-K^{(I)}_{\tau})[f_{\tau}(\boldsymbol{X}^{(I)}) - \boldsymbol{e}^{(I)}]$ and the expectation is taken w.r.t. the randomness of the \textit{SampleTree} algorithm. 

The corresponding error decomposition then becomes: 
 \begin{equation}
    \begin{aligned}
        \left\|\mathbb{E}_u f_\tau-f_{0,-I}\right\|_N^2 \leq & \underbrace{2 \sum_{j=1}^r\left[S^\tau\right]_{j j}^2\left(\zeta_{j j}^*\right)^2+2 \sum_{j=r+1}^n\left(\zeta_{j j}^*\right)^2}_{\text {Bias } B_\tau^2} 
        +
        \underbrace{\frac{4}{N} \sum_{j=1}^r\left(1-S_{j j}^\tau\right)^2\left[U^T w^{(I)}\right]_j^2}_{\text {Variance } V_\tau} \\
        &  + \underbrace{\frac{4 \epsilon^2}{N}\|\sum_{i=0}^{\tau-1} S^{\tau-1-i} \tilde{\delta}_i\|_2^2}_{\text {Difference } D_\tau^2} 
    \end{aligned}
    \end{equation}
where 
$\zeta_{jj}^* = \left[\frac{1}{\sqrt{N}}U^T \left(
f_{0,-I}(\boldsymbol{X}^{(I)}) - \mathbb{E}_u f_N^c(\boldsymbol{X}^{(I)})\right)\right]_j$,
$f_\tau$  here has absorbed $f_N^c$ and $\mathbb{E}_u$ here means taking expectation w.r.t. the 
 randomness of the \textit{SampleTree} algorithm. 
\end{lemma}
\begin{proof}
By lemma 3.7 in \cite{gdbt},
the iteration of GBDT  can be written as
\begin{equation}
    f_{\tau+1} = f_{\tau} + \frac{\epsilon}{N} k_{\nu_{\tau}}(\cdot,\boldsymbol{X}^{(I)})[\boldsymbol{e}^{(I)}-f_{\tau}(\boldsymbol{X}^{(I)})], \ \nu_{\tau} \sim p(\nu | f_{\tau},\beta)
\end{equation}
The empirical kernel matrix at each iteration is 
\begin{equation}
    K^{(I)}_{\tau} = \frac{1}{N}k_\nu\left(\boldsymbol{X}^{(I)}, \boldsymbol{X}^{(I)}\right)= \frac{1}{N}\oplus_{i=1}^{L_\nu} w_{\nu_{\tau}}^{(i)} \mathbf{1}_{N_\nu^i \times N_\nu^i}
\end{equation}
 The eigenvalues of 
 $K^{(I)}_{\tau}$
 are $(1,\dots, 1, 0, \dots,0)$, with first $L_{\nu_{\tau}}$ eigenvalues being $1$. 
Take the expectations w.r.t. the randomness of \textit{SampleTree} algorithm,  we have 
\begin{equation}
    \mathbb{E} K^{(I)}_{\tau} = \frac{1}{N} \mathbb{K}^{(I)}_f\left(\boldsymbol{X}^{(I)}, \boldsymbol{X}^{(I)}\right)=\sum_{\nu \in \mathcal{V}} \frac{1}{N} k_\nu\left(\boldsymbol{X}^{(I)}, \boldsymbol{X}^{(I)}\right) p(\nu | f_{\tau},\beta)
\end{equation}
By the iteration update rule, we can write 
\begin{equation}
    \begin{aligned}f_{\tau+1}(\boldsymbol{X}^{(I)}) &= f_{\tau}(\boldsymbol{X}^{(I)}) + \frac{\epsilon}{N} k_{\nu_{\tau}}(\boldsymbol{X}^{(I)},\boldsymbol{X}^{(I)})[e^{(I)}-f_{\tau}(\boldsymbol{X}^{(I)})] \\ &=f_{\tau}(\boldsymbol{X}^{(I)}) - \epsilon K^{(I)}_{\tau} [f_{\tau}(\boldsymbol{X}^{(I)}) - \boldsymbol{e}^{(I)}] \\ &= (I-\epsilon K^{(I)}_{\tau})f_{\tau}(\boldsymbol{X}^{(I)}) + \epsilon K^{(I)}_{\tau} \boldsymbol{e}^{(I)}\end{aligned} 
\end{equation}
where $\nu_{\tau} \sim p(\nu | f_{\tau},\beta)$ and  $K^{(I)}_{\tau} = \frac{1}{N}k_{\nu_{\tau}}(\boldsymbol{X}^{(I)},\boldsymbol{X}^{(I)})$.


For the stationary kernel $\mathbb{K}^{(I)}$, we have
\begin{equation}
    \mathbb{E}K^{(I)} = \frac{1}{N}\sum_{\nu} k_{\nu}(\boldsymbol{X}^{(I)},\boldsymbol{X}^{(I)})\pi(\nu)
\end{equation}
Apply SVD, we have $\mathbb{E}K^{(I)} = U \Lambda U^T$.

Since $K^{(I)}$ is independent of $f_{\tau}$ and $\boldsymbol{e}^{(I)}$, we can write
\begin{equation}
    \begin{aligned}\mathbb{E}f_{\tau+1}(\boldsymbol{X}^{(I)}) 
&= (I-\epsilon \mathbb{E}K^{(I)})\mathbb{E}f_{\tau}(\boldsymbol{X}^{(I)}) + \epsilon \mathbb{E}K^{(I)} \mathbb{E} \boldsymbol{e}^{(I)} + \epsilon\delta_\tau
\end{aligned}
\end{equation}
where $\delta_{\tau} = \mathbb{E}(K^{(I)}-K^{(I)}_{\tau})[f_{\tau}(\boldsymbol{X}^{(I)}) - \boldsymbol{e}^{(I)}]$.

Let $\zeta^{\tau} = \frac{1}{\sqrt{N}}U^T \mathbb{E}f_{\tau}(\boldsymbol{X}^{(I)}) $ and 
$\zeta^* = \frac{1}{\sqrt{N}}U^T \mathbb{E}[ 
f_{0,-I}(\boldsymbol{X}^{(I)}) - f_N^c(\boldsymbol{X}^{(I)})]  $, we can write 
\begin{equation}
    \zeta^{\tau + 1} = \zeta^{\tau} + \epsilon \Lambda\frac{\tilde{w}}{\sqrt{N}} - \epsilon \Lambda(\zeta^{\tau} - \zeta^*) + \frac{\epsilon U^T\delta_{\tau}}{\sqrt{N}}
\end{equation}
where $\tilde{w} = U^T[\mathbb{E} e^{(I)}- f_{0,-I}(\boldsymbol{X}^{(I)}) +  \mathbb{E}  f_N^c(\boldsymbol{X}^{(I)})  ]$.
Here we can regard the response as $\mathbb{E} e^{(I)}$ and its true generating function is 
$f_{0,-I}(\boldsymbol{X}^{(I)}) -  \mathbb{E}  f_N^c(\boldsymbol{X}^{(I)}) $ under reduced data. Since $\mathbb{E} e^{(I)}  = f_{0,-I}(\boldsymbol{X}^{(I)}) -   \mathbb{E}  f_N^c(\boldsymbol{X}^{(I)})  ] 
+ w^{(I)}$, $\tilde{w} = U^Tw^{(I)}$.

Rearranging, we have
\begin{equation}
    \zeta^{\tau + 1} - \zeta^* =(I-\epsilon\Lambda)( \zeta^{\tau} - \zeta^*)+ \epsilon \Lambda\frac{\tilde{w}}{\sqrt{N}}  + \frac{\epsilon \tilde{\delta}_{\tau}}{\sqrt{N}}
\end{equation}
where $\tilde{\delta}_{\tau} = U^T \delta_{\tau}$. 

Unwrapping, we have
\begin{equation}
    \zeta^{\tau}- \zeta^* = (I-S^{\tau})\frac{\tilde{w}}{\sqrt{N}} - S^{\tau} \zeta^* +  
\sum_{i=0}^{\tau-1} S^{\tau - i -1} 
\frac{\epsilon\tilde{\delta}_{i}}{\sqrt{N}}
\end{equation}
Then we have
\begin{equation}
    \|\zeta^{\tau}-\zeta^*\|_2^2 \leq \frac{4}{N}\|(I-S^{\tau}) \tilde{w} \|_2^2 + \frac{4\epsilon^2}{N}\|\sum_{i=0}^{\tau-1} S^{\tau - i -1} \tilde{\delta}_{i}\|_2^2 + 2\|S^{\tau} \zeta^*\|_2^2.
\end{equation}
Plug in $\tilde{w}, \tilde{\delta}_{i}$ and $\zeta^*$, we then get the error decomposition result for GBDT as follows:
 \begin{equation}
    \begin{aligned}
        \left\|\mathbb{E}_u f_\tau-f_{0,-I}\right\|_N^2 \leq & \underbrace{2 \sum_{j=1}^r\left[S^\tau\right]_{j j}^2\left(\zeta_{j j}^*\right)^2+2 \sum_{j=r+1}^n\left(\zeta_{j j}^*\right)^2}_{\text {Bias } B_\tau^2} 
        +
        \underbrace{\frac{4}{N} \sum_{j=1}^r\left(1-S_{j j}^\tau\right)^2\left[U^T w^{(I)}\right]_j^2}_{\text {Variance } V_\tau} \\
        &  + \underbrace{\frac{4 \epsilon^2}{N}\|\sum_{i=0}^{\tau-1} S^{\tau-1-i} \tilde{\delta}_i\|_2^2}_{\text {Difference } D_\tau^2} 
    \end{aligned}
    \end{equation}
where $f_\tau$  here has absorbed $f_N^c$.
\end{proof}

\subsection{Proof of Lemma \ref{treedt}}
To prove corollary 
in the main text, we need to bound the term $\delta_{\tau}$ first to get a final bound for the difference term. We need to use the properties of the specifically defined GBDT algorithm. The key lemma we use is shown below. 
\setcounter{lemma}{7}
\begin{lemma}
Let the step size satisies  $\epsilon, 0<\epsilon<1$ and $\frac{1}{4 M_\beta N} \geq \epsilon$. Then  the following inequality holds:
    \begin{equation}
        \mathbb{E} V\left(f_T\right) \leq \frac{R^2}{2 N} e^{-\frac{T \epsilon}{2 M_\beta N} }.
    \end{equation}
    where $M_{\beta} = e^{\frac{mR^2}{N\beta}}$.
\end{lemma}
\begin{proof}
    By theorem G.22 in \cite{gdbt}, let  $\lambda \rightarrow 0$, we get the desired result. 
\end{proof}

By assumption 4.4 in the mian text, the dropout error $e^{(I)}$ is bounded.  
We denote the bound for $e^{(I)}$ as $C_0$, i.e., $\| \boldsymbol{e}^{(I)}\|_2 \leq C_0$, and $C_0$ is of order $O(\sqrt{N})$. We hide the dependence of $N$ here, but it will be considered when we prove the error bound. 
 Define $f_{*}$ as the empirical minimizer, which is different to our target function $f_{0,-I}$. We further denote $\|f_{*}\|_{\mathbb{R}^N} = R$.

\setcounter{lemma}{5}
\begin{lemma}
Same as conditions of corollary \ref{treecorfix}, we can bound the difference term $D_{\tau}^2$ for GBDT as follows:
\begin{equation}
    D_{\tau}^2 \leq \frac{4\epsilon^2}{N} {C^{\prime2}}(M_{\beta} - 1 )^2  \left(\frac{1-e^{-\frac{\tau \epsilon}{2M_{\beta}N}}}{1-   e^{-\frac{\epsilon}{2M_{\beta}N}}}
\right)^2
\end{equation}
where $C^{\prime} = 2R + C_0$. 
\end{lemma}
\begin{proof}

Note that 
\begin{equation}
    \begin{aligned}
    \|\delta_{\tau}\|_2 &= \|\mathbb{E}[ (K^{(I)} - K^{(I)}_{\tau}) f_{\tau}(\boldsymbol{X}^{(I)})    ]
- \mathbb{E}[ (K^{(I)} - K^{(I)}_{\tau})] \boldsymbol{e}^{(I)}\|_2 \\
&\leq \|\mathbb{E}[ (K^{(I)} - K^{(I)}_{\tau}) f_{\tau}(\boldsymbol{X}^{(I)})    ]\|_2 + \| \mathbb{E}[ (K^{(I)} - K^{(I)}_{\tau})] \boldsymbol{e}^{(I)}\|_2
\end{aligned}
\label{treq1}
\end{equation}



We leave the second term right now and start from the first term of the above equation. 

For the first term, we have
\begin{equation}
    \begin{aligned}
    \|\mathbb{E}[ (K^{(I)} - K^{(I)}_{\tau}) f_{\tau}(\boldsymbol{X}^{(I)})    ]\|_2 &=    
    \| \mathbb{E} K^{(I)} \mathbb{E} f_{\tau}(\boldsymbol{X}^{(I)}) - \mathbb{E}[  
    \mathbb{E} K^{(I)}_{\tau} f_{\tau}(\boldsymbol{X}^{(I)})  | f_{\tau}
    ] \|_2
    \\
    &= \left\|     \frac{1}{N}\sum_{\nu} k_{\nu}(\boldsymbol{X}^{(I)},\boldsymbol{X}^{(I)})\pi(\nu) \mathbb{E} f_{\tau}(\boldsymbol{X}^{(I)}) - \mathbb{E}\left[    \frac{1}{N}\sum_{\nu} k_{\nu}(\boldsymbol{X}^{(I)},\boldsymbol{X}^{(I)}) p(\nu | f_{\tau},\beta) f_{\tau}(\boldsymbol{X}^{(I)})  \right]
    \right\|_2 \\
    &= \left\| \frac{1}{N}\sum_{\nu} k_{\nu}(\boldsymbol{X}^{(I)},\boldsymbol{X}^{(I)}) \pi(\nu) \mathbb{E}\left[\left(1- \frac{p(\nu | f_{\tau},\beta) }{\pi(\nu)}\right) f_{\tau}(\boldsymbol{X}^{(I)})\right] \right\|_2 \\
    &\leq \mathbb{E} \left[ \max_{\nu \in \mathcal{V}}  \left| 1- \frac{p(\nu | f_{\tau},\beta)}{\pi(\nu)}  \right|
    \left\| f_{\tau}(\boldsymbol{X}^{(I)}) \right\|_2 \right]    \quad \quad  \quad   (\star)
\end{aligned}
\end{equation}
In the last inequality, we use the fact that  the max eigenvalues of 
$\frac{1}{N}k_\nu\left(\boldsymbol{X}^{(I)}, \boldsymbol{X}^{(I)}\right)$  is 1. 
Then by lemma F.2 in \cite{gdbt}, we can bound $  \max_{\nu \in \mathcal{V}} 
 \left| 1- \frac{p(\nu | f_{\tau},\beta)}{\pi(\nu)}  \right|$
as follows
\begin{equation}
    \begin{aligned}  \max_{\nu \in \mathcal{V}} 
    \left| 1- \frac{p(\nu | f_{\tau},\beta)}{\pi(\nu)}  \right| &\leq 
    \max \left\{ 1-e^{-\frac{2mV(f_\tau)}{\beta}}, 
    e^{\frac{2mV(f_\tau)}{\beta}} -1 \right\} \\
    &\leq  C V(f_\tau) 
\end{aligned}
\end{equation}
the second inequality is because $V(f_\tau) \leq \frac{R^2}{2N}$ 
according to lemma G.20 in \cite{gdbt},
we have $\|f_{\tau} - f_*\|_{\mathbb{R}^N} \leq \| f_*\|_{\mathbb{R}^N}$.
We then can find  $C V(f_\tau)$ to bound 
both the exponential terms. It is easy to see this from graphs. The constant $C$ here can be
\begin{equation}
    \frac{2N}{R^2} \left( 
e^{\frac{mR^2}{N\beta}} - 1
\right) = \frac{2N}{R^2} \left( 
M_{\beta} - 1
\right).
\end{equation}

Then by Corollary G.21 in \cite{gdbt}, we have 
$\|f_{\tau} \|_{\mathbb{R}^N} \leq 2\| f_*\|_{\mathbb{R}^N}$, so 
we can bound $(\star)$ as 
\begin{equation}
    \begin{aligned}
   \mathbb{E}\left[  \max_{\nu \in \mathcal{V}}  \left| 1- \frac{p(\nu | f_{\tau},\beta)}{\pi(\nu)}  \right|
    \left\| f_{\tau}(\boldsymbol{X}^{(I)}) \right\|_2 \right] & \leq 2 C R \mathbb{E} V(f_{\tau}) \leq
    2R(M_{\beta} - 1) e^{-\frac{\tau\epsilon}{2M_{\beta}N}}.
\end{aligned}
\end{equation}


For the second term in (\ref{treq1}) involving  $e^{(I)}$, by assumption 4.4, we can bound it similarly as:
\begin{equation}
    \| \mathbb{E}[ (K^{(I)} - K^{(I)}_{\tau})] e^{(I)}\|_2 \leq C_0 (M_{\beta} - 1) e^{-\frac{\tau\epsilon}{2M_{\beta}N}}.
\end{equation}

Combining all of this, we can bound the difference term 
$ D_{\tau }$
as
\begin{equation}
    D_{\tau } \leq \frac{2\epsilon}{\sqrt{N}} \sum_{i=0}^{\tau -1} (C_0 +2R)(M_{\beta} - 1 ) e^{-\frac{i\epsilon}{2M_{\beta}N}} = \frac{2\epsilon}{\sqrt{N}} (C_0 +2R)(M_{\beta} - 1 )  \frac{1-e^{-\frac{\tau \epsilon}{2M_{\beta}N}}}{1-   e^{-\frac{\epsilon}{2M_{\beta}N}}}.
\end{equation}
Let $C^{\prime} = 2R + C_0$, and square the RHS of the above equation, we get the same difference term as stated in the corollary.
Note that
for the stopping threshold for GBDT, we need to use the spectrum of $\mathbb{E}K^{(I)}$ instead of $K^{(I)}$ as illustrated in the proof. 

\end{proof}

\subsection{Proof of Lemma \ref{treedtpop}}
\begin{proof}
    For the GBDT variation of Lemma \ref{indupopf}, we also need to take expectation w.r.t. the randomness of the algorithm. So we need to add $\mathbb{E}_u$ for the $f_{\tau}^{\delta}$. 

The next thing to do is prove the analog bound of $\|\delta(\cdot)\|_2$. For this, we just use the similar proof strategy as in the proof of Lemma \ref{treedt}. We have the following bound 
\begin{equation}
    \begin{aligned}
    \|\delta_{\tau}(\cdot)\|_2 
    &= \left\| \frac{1}{N}\sum_{\nu} k_{\nu}(\cdot,\boldsymbol{X}^{(I)}) \pi(\nu) \mathbb{E}\left[\left(1- \frac{p(\nu | f_{\tau},\beta) }{\pi(\nu)}\right) f_{\tau}(\boldsymbol{X}^{(I)})\right] \right\|_2 \\
    &\leq \sum_{\nu} \pi(\nu)
    \mathbb{E}
    \left\| \frac{1}{N} k_{\nu}(\cdot,\boldsymbol{X}^{(I)}) \left[\left(1- \frac{p(\nu | f_{\tau},\beta) }{\pi(\nu)}\right) f_{\tau}(\boldsymbol{X}^{(I)})\right] \right\|_2 \\
    &=  \sum_{\nu} \pi(\nu)
    \mathbb{E}
    \left\|  
    \left(1- \frac{p(\nu | f_{\tau},\beta) }{\pi(\nu)}\right)
    \sum_i \frac{1}{N}
    k_{\nu}(\cdot,\mathbf{X}_i^{(I)}) f_{\tau}(\mathbf{X}_i^{(I)})\right\|_2
    \\
    &\leq \sqrt{N} \mathbb{E} \left[ \max_{\nu \in \mathcal{V}}  \left| 1- \frac{p(\nu | f_{\tau},\beta)}{\pi(\nu)}  \right|
    \left\| f_{\tau}(\boldsymbol{X}^{(I)}) \right\|_2 \right]    \quad \quad  \quad   (\star)
\end{aligned}
\end{equation}
where in the last inequality we use the fact that $0 \leq \frac{1}{N}
    k_{\nu}(\cdot,\mathbf{X}_i^{(I)}) \leq 1$. 
Then similar to the argument in the proof of Lemma \ref{treedt}, we have
\begin{equation}
    \| \delta_{\tau}(\cdot)\|_2 \leq 
     (C_0 + 2R) (M_{\beta} - 1) e^{-\frac{\tau\epsilon}{2M_{\beta}N}}.
     \label{treeeqida}
\end{equation}
For notation simplicity, we denote the RHS of (\ref{treeeqida}) as $err_\tau$. 
Then according to the recursion updating, we have 
\begin{equation}
     \mathbb{E}_uf_{\tau + 1}^{\delta}(\cdot) = \mathbb{E}_uf_{\tau }^{\delta}(\cdot) + \epsilon \mathbb{E}_u\delta_{\tau}(\cdot) - \frac{\epsilon}{N}\mathbb{E}_u \mathbb{K}^{(I)} (\cdot, \boldsymbol{X}^{(I)}) f_{\tau }^{\delta}(\boldsymbol{X}^{(I)}).
\end{equation}
We just need to bound each term separately. 
For the last term, note that 
$f_{\tau }^{\delta}(\boldsymbol{X}^{(I)})$ is simply the different term $D_{\tau}^2$, so we have the following bound according to the proof of Lemma \ref{treedt}
\begin{equation}
    \|\frac{1}{N}\mathbb{E}_u \mathbb{K}^{(I)} (\cdot, \boldsymbol{X}^{(I)}) f_{\tau }^{\delta}(\boldsymbol{X}^{(I)})\|_2 \leq 
\sqrt{N} \sum_{i = 0}^{\tau - 1} err_i
\end{equation}
where we also apply the same claim above.

Then use proof by induction, it is not hard to show that
\begin{equation}
   \| \mathbb{E}_uf_{\tau }^{\delta}(\cdot)\|_2 \leq 
   \epsilon \sqrt{N}
   \sum_{i=0}^{\tau - 1}(\tau - i ) err_i \leq 
   \epsilon\sqrt{N}
(C_0 +2R)(M_{\beta} - 1 )  \frac{\tau }{1-   e^{-\frac{\epsilon}{2M_{\beta}N}}}.
\end{equation}
\end{proof}


\section{Implementation and experiments details} \label{exp}

\subsection{Practical guidance} \label{dicprac}
As stated in Section 5 of the main text, the proposed optimal stopping rule is not feasible in practice. In addition to the method described in the main text, we could also use cross-validation. However, this would increase the computational cost of our algorithm without necessarily ensuring better fitting results. It is important to note that the primary computational expense comes from the gradient calculations.

An alternative approach is to perform a single update with a larger step size. Cross-validation could also be used to select an appropriate step size for this one-iteration update. In this scenario, the method becomes similar to kernel ridge regression. According to \cite{esnonpara}, kernel ridge regression and early stopping are equivalent, with the running sum $\eta_{\tau}$ functioning similarly to the regularization parameter $\lambda$ in kernel ridge regression. In the kernel ridge regression framework of lazy training, \cite{lazyvi} uses cross-validation to select an optimal $\lambda$. Hence, we can view early stopping as a comparable regularization technique.

\subsection{Implementation} 
We implemented the neural networks and related experiments using PyTorch Lightning \citep{pytorchlightning}. Our neural network structure follows the design described in \cite{jacot}, which includes a scaling factor of $\frac{1}{\sqrt{m}}$. For the GBDT experiments, we utilized the standard CatBoost library \citep{catboost}. The experiments were conducted on a high-performance computing (HPC) system equipped with an Intel Xeon E5-2680 v3 @ 2.50GHz CPU and an NVIDIA RTX 2080Ti GPU. The running time for the neural networks was measured using the NVIDIA RTX 2080Ti, while the GBDT experiments were timed using the CPU.


\subsection{Verifying theoretical bounds}
The features for the neural networks are generated as 
\(\boldsymbol{X} \sim \mathcal{N}(0, I_{3 \times 3})\),
while the features for the GBDT are drawn from a discrete uniform distribution,
\(X_i \sim \text{Uniform}(i-1, i+2)\).
We use a discrete uniform distribution for GBDT because, if we were to use a continuous distribution, 
the CatBoost library would partition the feature space into too many bins.
This would significantly increase the computational cost when calculating the exact empirical kernel matrix.
By using a discrete uniform distribution, the number of bins is limited to a manageable size. And we use $\beta_1 = 3$ in the construction of  $f_0 := f(\boldsymbol{X}^{(1)}) + \beta_1 X_1$ defined in Section 5.1 of the main text.

The reason why we use independent features in this simulation is because if $X_1$ is correlated with some other features, 
when we drop $X_1$, $f_{0,-1}$ would not simply be $f$ as constructed in Section 5.1 in the main text, 
and it is complex to derive the true $f_{0,-1}$ in this setting.

As for the use of shallow networks and GBDT, in neural networks, the calculation of the Neural Tangent Kernel (NTK) is implemented recursively. 
For deeper networks, this significantly increases computational cost. 
In the case of GBDT, by the definition of the weak learner's kernel in equation (\ref{wekernel}), deeper tree structures would lead to an increased number of trees. 
To compute the empirical kernel matrix, we would need to perform calculations for each possible tree, making the implementation for a tree of depth 2 the simplest.

To compute the constant $C_{\mathcal{H}}$ in the stopping rule, 
we assume that $f_N^c, f_{0,-I} \in \text{span} \left\{ 
\mathbb{K}^{(I)}(\cdot, \mathbf{X}_i^{(I)}), i = 1,\dots, N
\right\}$. When samples are large enough, this is fine. Note that this is simply an approximation. Then we can write 
\begin{equation}
   f_N^c(\cdot) - f_{0,-I}(\cdot)  = \sum_{i=1}^N \alpha_i \mathbb{K}^{(I)}(\cdot, \mathbf{X}_i^{(I)}).
\end{equation}
Plug into the data, we can solve $\boldsymbol{\alpha}$ as
\begin{equation}
    \boldsymbol{\alpha} = \frac{1}{N}{K^{(I)}}^{-1} \left[f_N^c(\boldsymbol{X}^{(I)} )-f_{0,-I}(\boldsymbol{X}^{(I)}) \right]
\end{equation}
where we use pesudo-inverse if $K^{(I)}$ is not ivnertible. 
By property of the RKHS, the Hilbert norm of $f_N^c(\cdot) - f_{0,-I}(\cdot)$ can be computed using the emprical kernel matrix as: 
\begin{equation}
    C_{\mathcal{H}}^2 = N \boldsymbol{\alpha}^T K^{(I)}\boldsymbol{\alpha} = \frac{1}{N}
    \boldsymbol{c}^T {K^{(I)}}^{-1} 
     \boldsymbol{c}
\end{equation}
where $ \boldsymbol{c} := f_N^c(\boldsymbol{X}^{(I)}) - f_{0,-I}(\boldsymbol{X}^{(I)})$.

Even though we provide the optimal stopping time \(\widehat{T}_{\text{op}}\) in the main text, 
the upper bound for the difference term \(g(\tau)\) is still quite complex.
As a result, computing the exact \(\widehat{T}_{\text{op}}\) is infeasible.
This is why we use \(\widehat{T}_{\max}\) in this simulation.
As shown in the proof, \(\widehat{T}_{\max}\) also achieves the desired convergence rate of \(\mathcal{O} \left( \frac{1}{\sqrt{N}} \right)\).
The advantage of \(\widehat{T}_{\max}\) is that it can be calculated exactly using the  straightforward formula:
\begin{equation}
    \widehat{T}_{\max }:=\arg \min \left\{\tau \in \mathbb{N} \left\lvert\, \widehat{\mathcal{R}}_K\left(1 / \sqrt{\eta_\tau}\right)>\frac{C^2_{\mathcal{H}}}{2 e \sigma \eta_\tau}\right.\right\}-1.
\end{equation}

In this simulation, we made several simplifications to ease the implementation and employed approximations at various steps.
Even under these simplified conditions, and with knowledge of the true data-generating functions, 
computing the theoretical optimal stopping time \(\widehat{T}_{\text{op}}\) remains challenging.
The main goal of this simulation is to empirically validate our theoretical results and provide readers with a clearer understanding of the theoretical framework proposed in this paper.

\subsection{Shapley value estimation} 
 \label{shapsample} 
The running times for estimating all 10 Shapley values using both the neural network and GBDT models are listed in Table \ref{shaptime}.
Our proposed approach is approximately twice as fast as the retrain method in both models.
The running time for the neural network is quite long, largely due to our unoptimized implementation and limited computing infrastructure.
Parallelization could certainly be applied to estimate Shapley values, as the calculation for each value is independent, though we did not implement it here.
The running times should be interpreted in the context of consistent implementation and infrastructure, our proposed method demonstrates significant computational savings compared to the retrain approach.
Given that GBDT runs much faster than neural networks while providing similar estimation results, it is recommended to use GBDT in practice when advanced computing infrastructure and optimization techniques are not available.
\begin{table}[htp!]
    \centering
 \begin{tabular}{l| l l }
\toprule
     Model &   Method  & Time (Mean $\pm$ Std)  \\
     \midrule
     Neural network &  Early stopping  &  $2.98 \pm 0.24$ hours      \\
       &  Retrain &        $5.44 \pm 0.37$ hours       \\
     GBDT &    Early stopping & $56.94 \pm 1.59$ seconds     \\
         &  Retrain &   $154.12 \pm 4.27$ seconds \\
\bottomrule
\end{tabular}
    \caption{Shapley value etimation running time comparison.}
    \label{shaptime}
\end{table}

\subsection{Predicting flue gas emissions}
The feature correlation heat map of the NOx emissions data, used in Section 5.4 of the main text, is depicted in Figure \ref{heatmap}.
The variable importance (VI) estimation results using GBDT are shown in Figure \ref{treevi}. Note that we use decision trees of depth 8 here to enable a better fitting of the data. 
It is evident that the features are highly correlated.
The GBDT estimation results are similar to those from the neural network.
For VI, defined using negative MSE, all values are relatively small.
Regarding Shapley values, the neural network estimates tend to be higher than those from GBDT.
Similar to the neural network, we observe that our early stopping approach closely matches the estimates obtained through re-training, 
while the dropout method tends to overestimate both VI and Shapley values.

\begin{figure}[htp!]
     \centering
 \begin{subfigure}[b]{0.45\textwidth}
         \centering
         \includegraphics[width=1\textwidth]{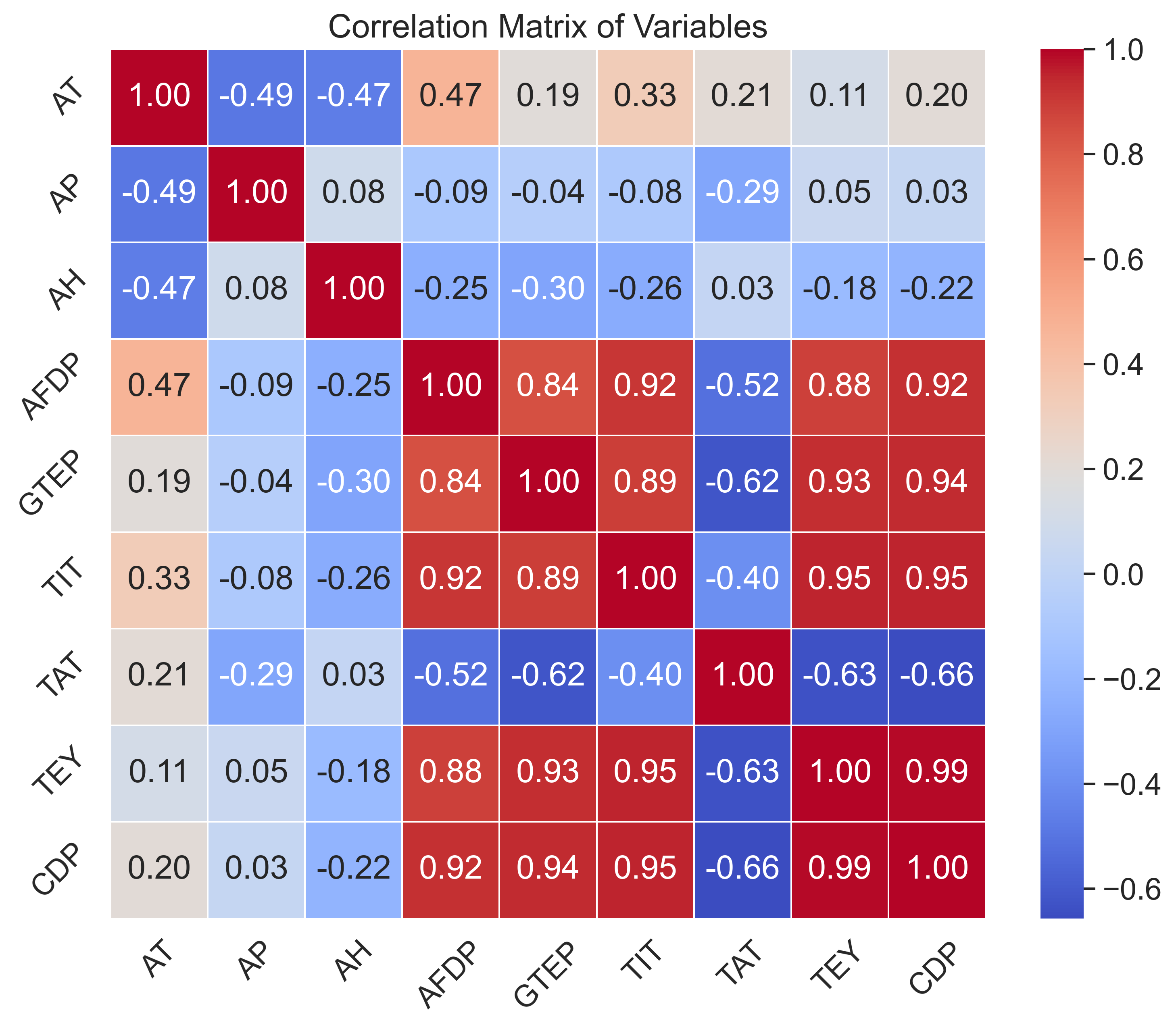}
     \end{subfigure}
\caption{Correlation heat map.}
\label{heatmap}
\end{figure}

\begin{figure}[htp!]
     \centering
    \begin{subfigure}[b]{0.45\textwidth}
         \centering
         \includegraphics[width=1\textwidth]{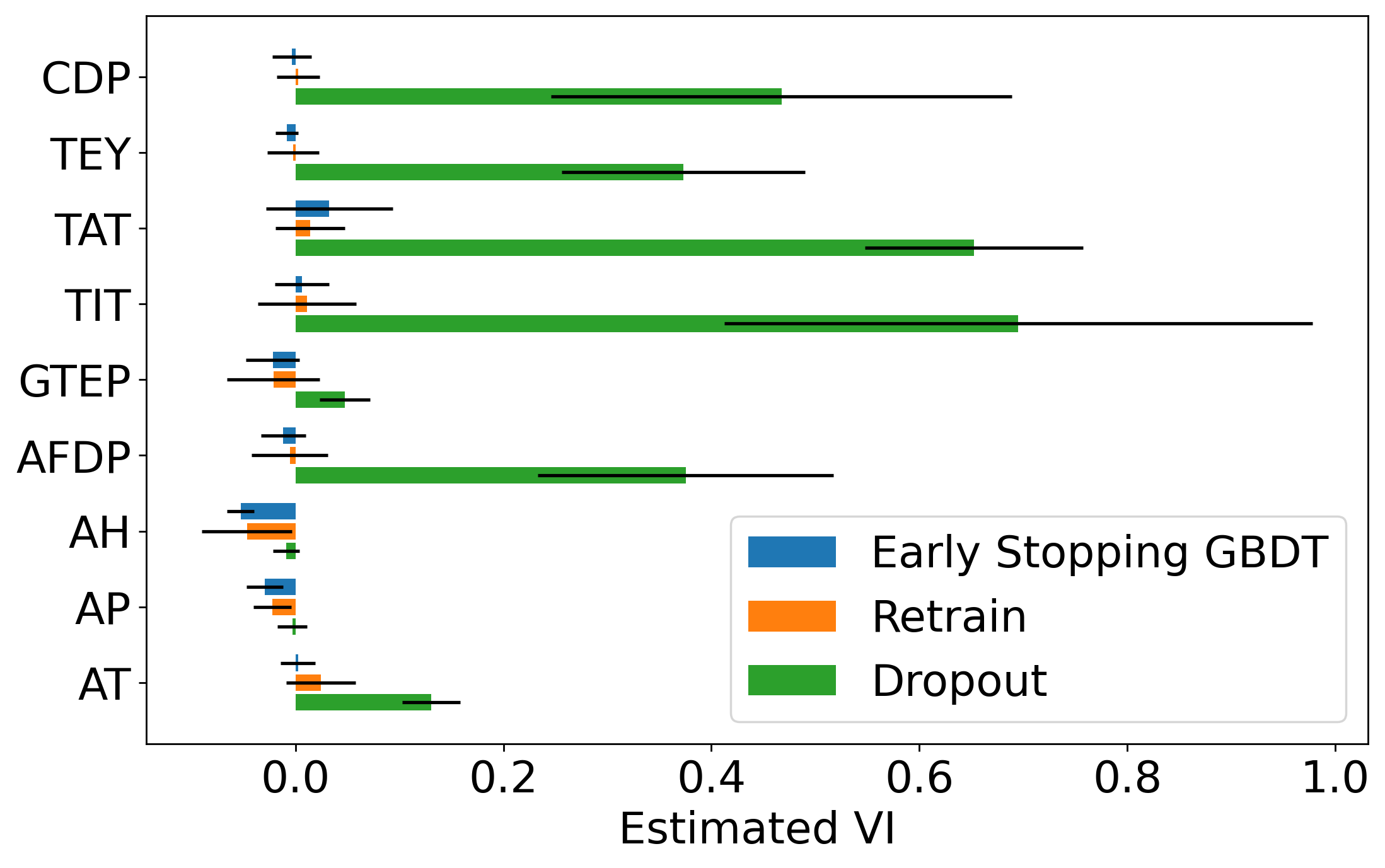}
         \caption{VI estimation}  
     \end{subfigure}
     \begin{subfigure}[b]{0.45\textwidth}
         \centering
         \includegraphics[width=1\textwidth]{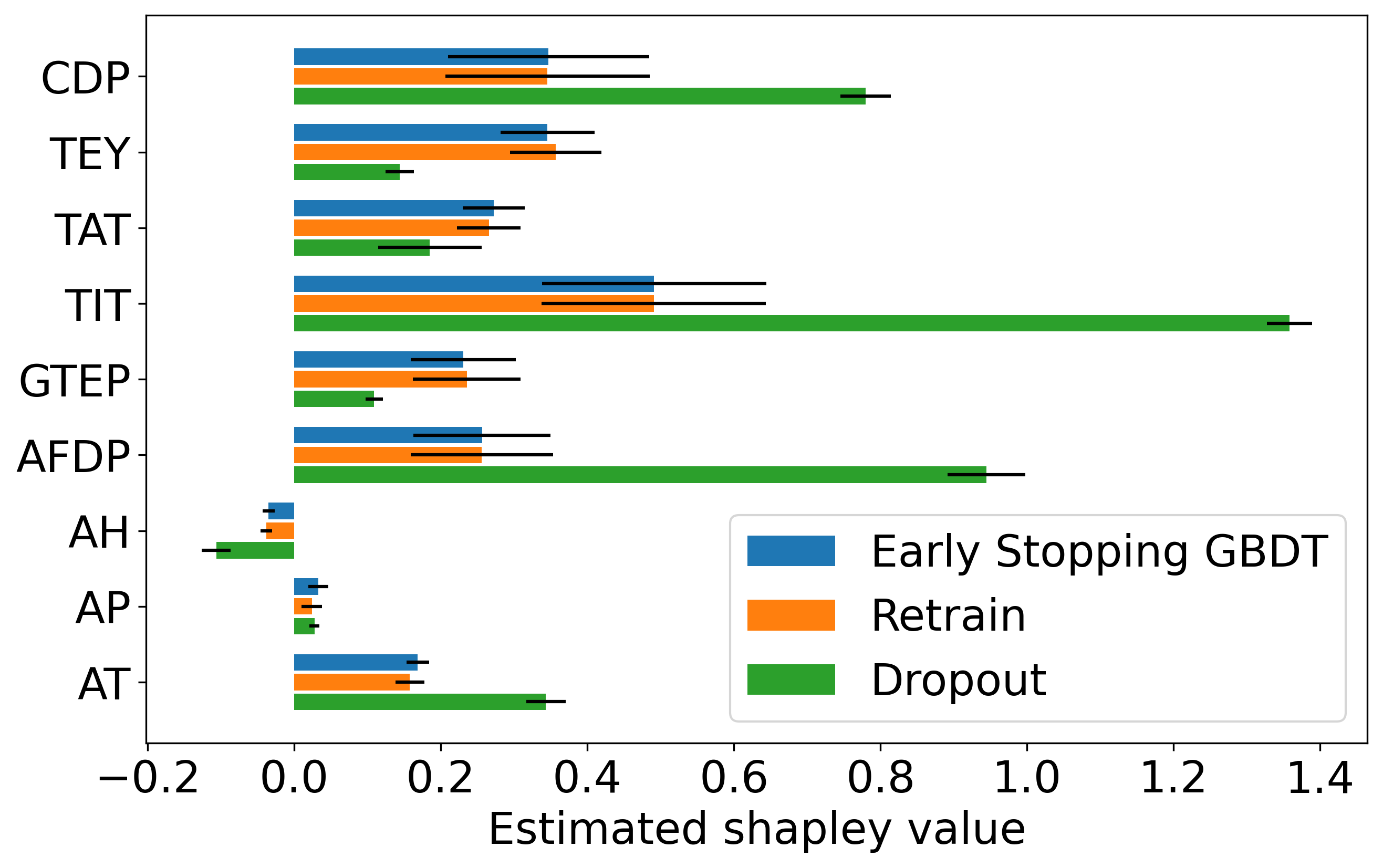}  
         \caption{Shapley value estimation}
     \end{subfigure}
\caption{Flue gas emissions variable importance estimation results using GBDT.}
\label{treevi}
\end{figure}

\end{document}